\documentclass{article}

\PassOptionsToPackage{numbers}{natbib}

    \usepackage[final]{neurips_2022}

\usepackage[utf8]{inputenc} 
\usepackage[T1]{fontenc}    
\usepackage{hyperref}       
\usepackage{url}            
\usepackage{booktabs}       
\usepackage{amsfonts}       
\usepackage{nicefrac}       
\usepackage{microtype}      
\usepackage{xcolor}         

\usepackage{bm}
\usepackage{subfigure}
\usepackage{graphicx}
\usepackage{algorithm}
\usepackage{algorithmic}

\usepackage{amsmath}
\usepackage{amssymb}
\usepackage{mathtools}
\usepackage{amsthm}

\newcommand{\songl}[1]{#1}
\newcommand{\rbt}[1]{#1}
\usepackage{multirow}
\theoremstyle{plain}
\newtheorem{theorem}{Theorem}[section]

\newtheorem{lemma}[theorem]{Lemma}

\theoremstyle{definition}

\newtheorem{assumption}[theorem]{Assumption}
\theoremstyle{remark}

\title{Monte Carlo Tree Search based Variable Selection\\ for High Dimensional Bayesian Optimization}

\author{%
 Lei Song\thanks{Equal Contribution}, \ Ke Xue\footnotemark[1], \ Xiaobin Huang, Chao Qian\thanks{Corresponding Author} \\
State Key Laboratory for Novel Software Technology,\\
Nanjing University, Nanjing 210023, China\\
  \texttt{\{songl, xuek, huangxb, qianc\}@lamda.nju.edu.cn} \\
}

\begin{document}

\maketitle

\begin{abstract}
Bayesian optimization (BO) is a class of popular methods for expensive black-box optimization, and has been widely applied to many scenarios. However, BO suffers from the curse of dimensionality, and scaling it to high-dimensional problems is still a challenge. 
In this paper, we propose a variable selection method MCTS-VS based on Monte Carlo tree search (MCTS), to iteratively select and optimize a subset of variables. That is, MCTS-VS constructs a low-dimensional subspace via MCTS and optimizes in the subspace with any BO algorithm. We give a theoretical analysis of the general variable selection method to reveal how it can work. Experiments on high-dimensional synthetic functions and real-world problems (i.e., NAS-bench problems and MuJoCo locomotion tasks) show that MCTS-VS equipped with a proper BO optimizer can achieve state-of-the-art performance.
\end{abstract}

\section{Introduction}

In many real-world tasks such as neural architecture search (NAS)~\citep{lanas} and policy search in reinforcement learning (RL)~\citep{borl}, one often needs to solve the expensive black-box optimization problems. Bayesian optimization (BO)~\citep{telosurvey,bosurvey2,bosurvey3,bosurvey1} is a sample-efficient algorithm for solving such problems. It iteratively fits a surrogate model, typically Gaussian process (GP), and maximizes an acquisition function to obtain the next point to evaluate. While BO has been employed in a wide variety of settings, successful applications are often limited to low-dimensional problems.

Recently, scaling BO to high-dimensional problems has received a lot of interest. Decomposition-based methods~\citep{addtreestructure,factorgraph,addgpucb,addoverlapping2,addoverlapping1} assume that the high-dimensional function to be optimized has a certain structure, typically the additive structure. By decomposing the original high-dimensional function into the sum of several low-dimensional functions, they optimize each low-dimensional function to obtain the point in the high-dimensional space. However, it is not easy to decide whether a decomposition exists as well as to learn the decomposition. 

Other methods often assume that the original high-dimensional function with dimension $D$ has a low-dimensional subspace with dimension $d\ll D$, and then perform the optimization in the low-dimensional subspace and project the low-dimensional point back for evaluation. For example, embedding-based methods~\citep{alebo,hesbo,rembo} use a random matrix to embed the original space into the low-dimensional subspace. Another way is to select a subset of variables directly, which can even avoid the time-consuming matrix operations of embedding-based methods. For example, Dropout~\citep{dropout} selects $d$ variables randomly in each iteration. Note that for both embedding and variable selection methods, the parameter $d$ can have a large influence on the performance, which is, however, difficult to set in real-world problems.


In this paper, we propose a new Variable Selection method using Monte Carlo Tree Search (MCTS), called MCTS-VS. MCTS is employed to partition the variables into important and unimportant ones, and only those selected important variables are optimized via any black-box optimization algorithm, e.g., vanilla BO~\citep{bosurvey1} or TuRBO~\citep{turbo}. The values of unimportant variables are sampled using historical information. Compared with Dropout-BO, MCTS-VS can select important variables automatically.

We also provide regret and computational complexity analyses of general variable selection methods, showing that variable selection can reduce the computational complexity while increasing the cumulative regret. Our regret bound generalizes that of GP-UCB~\cite{gpucb} which always selects all variables, as well as that of Dropout~\cite{dropout} which selects $d$ variables randomly in each iteration. The results suggest that a good variable selection method should select as important variables as possible.

Experiments on high-dimensional synthetic functions and real-world problems (i.e., NAS and RL problems) show that MCTS-VS is better than the previous variable selection method Dropout~\cite{dropout}, and can also achieve the competitive performance to state-of-the-art BO algorithms. Furthermore, its running time is small due to the advantage of variable selection. We also observe that MCTS-VS can select important variables, explaining its good performance based on our theoretical analysis.  

\section{Background}

\subsection{Bayesian Optimization}


We consider the problem $\max_{\bm x \in \mathcal X} f(\bm x)$, where $f$ is a black-box function and $\mathcal X \subseteq \mathbb R^D$ is the domain. The basic framework of BO contains two critical components: a surrogate model and an acquisition function. GP is the most popular surrogate model. Given the sampled data points $\{(\bm x^i, y^i)\}_{i=1}^{t-1}$, where $y^i=f(\bm x^i) + \epsilon^i$ and $\epsilon^i\sim \mathcal{N}(0, \eta^2)$ is the observation noise, GP at iteration $t$ seeks to infer $f\sim \mathcal{GP}(\mu(\cdot), k(\cdot, \cdot)+\eta^2 \mathbf{I})$, specified by the mean $\mu(\cdot)$ and covariance kernel $k(\cdot, \cdot)$, where $\mathbf I$ is the identity matrix of size $D$. After that, an acquisition function, e.g., \rbt{Probability of Improvement (PI)~\cite{Krushner64PI}, Expected Improvement (EI)~\citep{ei2} or Upper Confidence Bound (UCB)~\citep{gpucb}}, is used to determine the next query point $\bm x^t$ while balancing exploitation and exploration.

\subsection{High-dimensional Bayesian Optimization}

Scaling BO to high-dimensional problems is a challenge due to the curse of dimensionality and the computation cost. As the dimension increases, the search space increases exponentially, requiring more samples, and thus more expensive evaluations, to find a good solution. Furthermore, the \rbt{computation} cost of updating the GP model \rbt{and optimizing the acquisition function} will be very time-consuming~\cite{gp}. There have been a few common approaches to tackle high-dimensional BO with different assumptions.

\textbf{Decomposition.} Assuming that the function can be decomposed into the sum of low-dimensional functions with disjoint subspaces,~\citet{addgpucb} proposed the Add-GP-UCB algorithm to optimize those low-dimensional functions separately, which was further generalized to overlapping subspaces~\cite{addoverlapping2,addoverlapping1}.~\citet{ebo} proposed ensemble BO that uses an ensemble of additive GP models for scalability.~\citet{addtreestructure} constrained the dependency graphs of decomposition to tree structures to facilitate the decomposition learning and optimization. For most problems, however, the decomposition is unknown, and also difficult to learn.



\textbf{Embedding.} Assuming that only a few dimensions affect the high-dimensional function significantly, embedding-based methods embed the high-dimensional space into a low-dimensional subspace, and optimize in the subspace while projecting the point back for evaluation. REMBO and its variants use a random matrix to embed the search space into a low-dimensional subspace~\cite{rembophi,rembogamma,rembo}.~\citet{hesbo} used a hash-based method for embedding.~\citet{alebo} proposed ALEBO, focusing on several misconceptions in REMBO to improve the performance. The VAE-based approaches were also employed to project a structured input space (e.g., graphs and images) to a low-dimensional subspace~\cite{vae1,vae2}. 


\textbf{Variable Selection.} Based on the same assumption as embedding, variable selection methods iteratively select a subset of variables to build a low-dimensional subspace and optimize through BO. The selected variables can be viewed as important variables that are valuable for exploitation, or having high uncertainty that are valuable for exploration. A classical method is Dropout~\cite{dropout}, which randomly chooses $d$ variables in each iteration.~\citet{hsicbo} uses Hilbert Schmidt Independence criterion to guide variable selection. When evaluating the sampled point, the values of those unselected variables are obtained by random sampling or using historical information. VS-BO~\cite{vsbo} selects variables with larger estimated gradients and uses CMA-ES~\cite{cmaes} to obtain the values of unselected variables. Note that variable selection can be faster than embedding, because the embedding cost (e.g., matrix inversion) is time-consuming for high-dimensional optimization.


Both embedding and variable selection methods need to specify the parameter $d$, i.e., the dimension of low-dimensional subspace, which will affect the performance significantly, but is not easy to set. There are also some methods to improve the basic components of BO directly for high-dimensional problems. For example, DNGO~\cite{dngo} uses the neural network as an alternative of GP to speed up inference; \songl{BO-PP~\cite{bopp} generates pseudo-points (i.e., data points whose objective values are not evaluated) to improve the GP model; SAASBO~\cite{saas} uses sparsity-inducing prior to perform variable selection implicitly, making the coefficients of unimportant variables near to zero and thus restraining over-exploration on these variables. Note that different from Dropout and our proposed MCTS-VS, SAASBO still optimizes all variables, and also due to its high computational cost of inference, it is very time-consuming as reported in~\cite{saas}.} These methods can be combined with the above-mentioned dimensionality reduction methods, which may bring further improvement.

\subsection{Monte Carlo Tree Search}

MCTS~\citep{mctssurvey} is a tree search algorithm based on random sampling, and has shown great success in high-dimensional tasks, such as Go~\citep{mctsgo,mctsgo1}. A tree node represents a state, describing the current situation, e.g., the position in path planning. Each tree node $X$ stores a value $v_X$ representing its goodness, and the number $n_X$ that it has been visited. They are used to calculate UCB~\citep{ucb1}, i.e.,\vspace{-0.1em}
\begin{equation}
\label{eq-MCTS-UCB}v_X + 2C_p\sqrt{2(\log n_p)/n_X},
\vspace{-0.1em}
\end{equation}
where $C_p$ is a hyper-parameter, and $n_p$ is the number of visits of the parent of $X$. UCB considers both exploitation and exploration, and will be used for node selection.

MCTS iteratively selects a leaf node of the tree for expansion. Each iteration can be divided into four steps: \emph{selection}, \emph{expansion}, \emph{simulation} and \emph{back-propagation}. Starting from the root node, selection is to recursively select a node with larger UCB until a leaf node, denoted as $X$. Expansion is to execute a certain action in the state represented by $X$ and transfer to the next state, e.g., move forward and arrive at a new position in path planning. We use the child node $Y$ of $X$ to represent the next state. Simulation is to obtain the value $v_Y$ via random sampling. Back-propagation is to update the value and the number of visits of $Y$'s ancestors. 

To tackle high-dimensional optimization,~\citet{lamcts} proposed LA-MCTS, which applies MCTS to iteratively partition the search space into small sub-regions, and optimizes only in the good sub-regions. That is, the root of the tree represents the entire search space $\Omega$, and each tree node $X$ represents a sub-region $\Omega_X$. The value $v_X$ is measured by the average objective value of the sampled points in the sub-region $\Omega_X$. In each iteration, after selecting a leaf node $X$, LA-MCTS performs the optimization in $\Omega_X$ by vanilla BO~\citep{bosurvey1} or TuRBO~\citep{turbo}, and the sampled points are used for clustering and classification to bifurcate $\Omega_X$ into two disjoint sub-regions, which are ``good'' and ``bad'', respectively. Note that the sub-regions are generated by dividing the range of variables, and their dimensionality does not decrease, which is still the number of all variables.~\citet{lamcts} have empirically shown the good performance of LA-MCTS. However, as the dimension increases, the search space increases exponentially, and more partitions and evaluations are required to find a good solution, making the application of LA-MCTS to high-dimensional optimization still limited.

\section{MCTS-VS Method}

In this section, we propose a Variable Selection method based on MCTS for high-dimensional BO, briefly called MCTS-VS. The main idea is to apply MCTS to iteratively partition all variables into important and unimportant ones, and perform BO only for those important variables. Let $[D]=\{1,2,\ldots,D\}$ denote the indexes of all variables $\bm{x}$, and $\bm{x}_{\mathbb{M}}$ denote the subset of variables indexed by $\mathbb{M} \subseteq [D]$.

We first introduce a $D$-dimensional vector named \emph{variable score}, which is a key component of MCTS-VS. Its $i$-th element represents the importance of the $i$-th variable $x_i$. During the running process of MCTS-VS, after optimizing a subset $\bm{x}_{\mathbb{M}}$ of variables where $\mathbb{M} \subseteq [D]$ denotes the indexes of the variables, a set $\mathcal D$ of sampled points will be generated, and the pair $(\mathbb{M},\mathcal{D})$ will be recorded into a set $\mathbb{D}$, called \emph{information set}. The variable score vector is based on $\mathbb{D}$, and calculated as
\begin{equation}
\label{eq-score}
\bm{s}=\left(\sum_{(\mathbb M, \mathcal D)\in \mathbb D} \sum_{(\bm x^i,y^i) \in \mathcal D} y^i \cdot g(\mathbb M) \right) \rbt{ \big / \left( \sum_{(\mathbb M, \mathcal D)\in \mathbb D} |\mathcal D| \cdot g(\mathbb M) \right) },
\end{equation}
where the function $g:2^{[D]} \rightarrow \mathbb \{0,1\}^D$ gives the Boolean vector representation of a variable index subset $\mathbb{M} \subseteq [D]$ (i.e., the $i$-th element of $g(\mathbb M)$ is $1$ if $i\in \mathbb M$, and 0 otherwise), and $/$ is the element-wise division. Each dimension of $\sum_{(\mathbb M, \mathcal D)\in \mathbb D} \sum_{(\bm x^i,y^i) \in \mathcal D} y^i \cdot g(\mathbb M)$ is the sum of query evaluations using each variable, and each dimension of $\sum_{(\mathbb M, \mathcal D)\in \mathbb D} |\mathcal D| \cdot g(\mathbb M)$ is the number of queries using each variable.
Thus, the $i$-th element of variable score $\bm{s}$, representing the importance of the $i$-th variable $x_i$, is actually measured by the \rbt{average} goodness of all the sampled points that are generated by optimizing a subset of variables containing $x_i$. The variable score $\bm{s}$ will be used to define the value of each tree node of MCTS as well as for node expansion.

In MCTS-VS, the root of the tree represents all variables. A tree node $X$ represents a subset of variables, whose index set is denoted by $\mathbb A_X \subseteq \mathbb [D]$, and it stores the value $v_X$ and the number $n_X$ of visits, which are used to calculate the value of UCB as in Eq.~(\refeq{eq-MCTS-UCB}). The value $v_X$ is defined as the average score (i.e., importance) of the variables contained by $X$, which can be calculated by $\bm s\cdot g(\mathbb A_X)/|\mathbb A_X|$, where $g(\mathbb A_X)$ is the Boolean vector representation of $\mathbb{A}_X$ and $|\mathbb A_X|$ is the size of $\mathbb A_X$, i.e., the number of variables in node $X$.

At each iteration, MCTS-VS first recursively selects a node with larger UCB until a leaf node (denoted as $X$), which is regarded as containing important variables. Note that if we optimize the subset $\bm{x}_{\mathbb A_X}$ of variables represented by the leaf $X$ directly, the variables in $\bm{x}_{\mathbb A_X}$ will have the same score (because they are optimized together), and their relative importance cannot be further distinguished. Thus, MCTS-VS uniformly selects a variable index subset $\mathbb M$ from $\mathbb A_X$ at random, and employs BO to optimize $\bm{x}_{\mathbb M}$ as well as $\bm{x}_{\mathbb{A}_X \setminus \mathbb{M}}$; this process is repeated for several times. After that, the information set $\mathbb{D}$ will be augmented by the pairs of the selected variable index subset $\mathbb M$ (or $\mathbb{A}_X \setminus \mathbb{M}$) and the corresponding sampled points generated by BO. The variable score vector $\bm{s}$ will be updated using this new $\mathbb{D}$. Based on $\bm{s}$, the variable index set $\mathbb{A}_X$ represented by the leaf $X$ will be divided into two disjoint subsets, containing variables with larger and smaller scores (i.e., important and unimportant variables), respectively, and the leaf $X$ will be bifurcated into two child nodes accordingly. Finally, the $v$ values of these two children will be calculated using the variable score vector $\bm{s}$, and back-propagation will be performed to update the $v$ value and the number of visits of the nodes along the current path of the tree. 

MCTS-VS can be equipped with any specific BO optimizer, resulting in the concrete algorithm MCTS-VS-BO, where BO is used to optimize the selected subsets of variables during the running of MCTS-VS. Compared with LA-MCTS~\citep{lamcts}, MCTS-VS applies MCTS to partition the variables instead of the search space, and thus can be more scalable. Compared with the previous variable selection method Dropout~\citep{dropout}, MCTS-VS can select important variables automatically instead of randomly selecting a fixed number of variables in each iteration. Next we introduce it in detail.

\begin{algorithm}[t!]
\caption{MCTS-VS}
\label{lamcts_vs}
{\textbf{Parameters}:} batch size $N_v$ of variable index subset, sample batch size $N_{s}$, total number $N_e$ of evaluations, threshold $N_{bad}$ for re-initializing a tree and $N_{split}$ for splitting a node, \rbt{hyper-parameter $k$ for the best-$k$ strategy}\\
{\textbf{Process}:}
\begin{algorithmic}[1]
\STATE Initialize the information set $\mathbb{D} = \{(\mathbb M_i, \mathcal D_{i}), ( \mathbb{\bar M}_i, \mathcal D_{\bar i})\}_{i=1}^{N_v}$;
\STATE Store the best $k$ sampled points in $\mathbb{D}$;
\STATE Calculate the variable score $\bm s$ using $\mathbb{D}$ as in Eq.~(\refeq{eq-score});
\STATE Initialize the Monte Carlo tree;
\STATE Set $t = 2\times N_v\times N_{s}$ and $n_{bad} = 0$;
\WHILE{$t < N_e$}
\IF{$n_{bad} > N_{bad}$}
\STATE Initialize the Monte Carlo tree and set $n_{bad} = 0$
\ENDIF
\STATE $X \leftarrow$ the leaf node selected by UCB;
\STATE Let $\mathbb A_X$ denote the indexes of the subset of variables represented by $X$;
\STATE \rbt{Increase $n_{bad}$ by 1 once visiting a right child node on the path from the root node to $X$;}
\FOR {$j=1 : N_v$}
\STATE Sample a variable index subset $\mathbb M$ from $\mathbb A_X$ uniformly at random;
\STATE Fit a GP model using the points $\{(\bm x_\mathbb M^i, y^i)\}_{i=1}^t$ sampled-so-far, where only the variables indexed by $\mathbb M$ are used; 
\STATE Generate $\{\bm x_{\mathbb {M}}^{t+i}\}_{i=1}^{N_s}$ by maximizing an acquisition function;
\STATE Determine $\{\bm x_{[D]\setminus \mathbb M}^{t+i}\}_{i=1}^{N_s}$ by the ``fill-in'' strategy;
\STATE Evaluate $\bm x^{t+i} = [\bm x_{\mathbb M}^{t+i}, \bm x_{[D]\setminus \mathbb M}^{t+i}]$ to obtain $y^{t+i}$ for $i=1,2,\ldots,N_s$; 
\STATE $\mathbb D= \mathbb D\cup \{(\mathbb M, \{(\bm{x}^{t+i}, y^{t+i})\}_{i=1}^{N_s})\}$;
\STATE \rbt{Store} the best $k$ points sampled-so-far;
\STATE $t = t +  N_s$;
\STATE Repeat lines~15--21 for $\bar{\mathbb M} = \mathbb A_X \setminus \mathbb M$
\ENDFOR
\STATE Calculate the variable score $\bm s$ using $\mathbb{D}$ as in Eq.~(\refeq{eq-score});
\IF{$|\mathbb A_{X}| > N_{split}$}
\STATE Bifurcate the leaf node $X$ into two child nodes, whose $v$ value and number of visits are calculated by $\bm s$ and set to 0, respectively 
\ENDIF
\STATE Back-propagate to update the $v$ value and number of visits of the nodes on the path from the root to $X$ 
\ENDWHILE
\end{algorithmic}
\end{algorithm}

\subsection{Details of MCTS-VS}
\label{sec:details}

The procedure of MCTS-VS is described in Algorithm~\ref{lamcts_vs}. In line~1, it first initializes the information set $\mathbb D$. In particular, a variable index subset $\mathbb M_i$ is randomly sampled from $[D]$, and the Latin hypercube sampling~\citep{lhs} is used to generate two sets (denoted as $\mathcal D_{i}$ and $\mathcal D_{\bar{i}}$) of $N_s$ points to form the two pairs of $(\mathbb M_i, \mathcal D_{i})$ and $(\bar{\mathbb M}_i, \mathcal D_{\bar{i}})$, where $\bar {\mathbb M}_i = [D] \setminus \mathbb M_i$. This process will be repeated for $N_v$ times, resulting in the initial $\mathbb D=\{(\mathbb M_i, \mathcal D_{i}), ( \mathbb{\bar M}_i, \mathcal D_{\bar i})\}_{i=1}^{N_v}$. The variable score vector $\bm{s}$ is calculated using this initial $\mathbb D$ in line~3, and the Monte Carlo tree is initialized in line~4 by adding only a root node, whose $v$ value is calculated according to $\bm{s}$ and number of visits is 0. MCTS-VS uses the variable $t$ to record the number of evaluations it has performed, and thus $t$ is set to $2\times N_v \times N_s$ in line~5 as the initial $\mathbb D$ contains $2\times N_v \times N_s$ sampled points in total.

In each iteration (i.e., lines~7--28) of MCTS-VS, it selects a leaf node $X$ by UCB in line~10, and optimizes the variables (i.e., $\bm{x}_{\mathbb{A}_X}$) represented by $X$ in lines~13--23. Note that to measure the relative importance of variables in $\bm{x}_{\mathbb{A}_X}$, MCTS-VS optimizes different subsets of variables of $\bm{x}_{\mathbb{A}_X}$ instead of $\bm{x}_{\mathbb{A}_X}$ directly. That is, a variable index subset $\mathbb{M}$ is randomly sampled from $\mathbb{A}_X$ in line~14, and the corresponding subset $\bm{x}_\mathbb{M}$ of variables is optimized by BO in lines~15--16. The data points $\{(\bm x_\mathbb M^i, y^i)\}_{i=1}^t$ sampled-so-far is used to fit a GP model, and $N_s$ (called \emph{sample batch size}) new points $\{\bm x_{\mathbb {M}}^{t+i}\}_{i=1}^{N_s}$ are generated by maximizing an acquisition function. Note that this is a standard BO procedure, which can be replaced by any other variant. To evaluate $\bm x_{\mathbb {M}}^{t+i}$, we need to fill in the values of the other variables $\bm x_{[D]\setminus \mathbb M}^{t+i}$, which will be explained later. After evaluating $\bm{x}^{t+i}=[\bm x_{\mathbb M}^{t+i}, \bm x_{[D]\setminus \mathbb M}^{t+i}]$ in line~18, the information set $\mathbb{D}$ is augmented with the new pair of $(\mathbb{M},\{(\bm{x}^{t+i}, y^{t+i})\}_{i=1}^{N_s})$ in line~19, and $t$ is increased by $N_s$ accordingly in line~21. For fairness, the complement subset $\bm{x}_{\bar{\mathbb M}}$ of variables, where $\bar{\mathbb M} = \mathbb{A}_X \setminus \mathbb M$, is also optimized by the same way, i.e., lines~15--21 of Algorithm~\ref{lamcts_vs} is repeated for $\bar{\mathbb M}$. The whole process of optimizing $\bm{x}_{\mathbb M}$ and $\bm{x}_{\bar{\mathbb M}}$ in lines~14--22 will be repeated for $N_v$ times, which is called \emph{batch size of variable index subset}.

To fill in the values of the un-optimized variables in line~17, we employ the \emph{best-$k$} strategy, which utilizes the best $k$ data points sampled-so-far, denoted as $\{(\bm{x}^{*j}, y^{*j})\}_{j=1}^k$. That is, $\{y^{*j}\}_{j=1}^k$ are the $k$ largest objective values observed-so-far. If the $i$-th variable is un-optimized, its value will be uniformly selected from $\{x_i^{*j}\}_{j=1}^k$ at random. Thus, MCTS-VS needs to store the best $k$ data points in line~2 after initializing the information set $\mathbb{D}$, and update them in line~20 after augmenting $\mathbb{D}$. Other direct ``fill-in'' strategies include sampling the value randomly, or using the average variable value of the best $k$ data points. The superiority of the employed best-$k$ strategy will be shown in the experiments in Appendix~\ref{appendix:ablation}.

After optimizing the variables $\bm{x}_{\mathbb{A}_X}$ represented by the selected leaf $X$, the variable score vector $\bm s$ measuring the importance of each variable will be updated using the augmented $\mathbb{D}$ in line~24. If the number $|\mathbb A_{X}|$ of variables in the leaf $X$ is larger than a threshold $N_{split}$ (i.e., line~25), $\mathbb{A}_X$ will be divided into two subsets. One contains those ``important'' variable indexes with score larger than the average score of $\bm{x}_{\mathbb{A}_X}$, and the other contains the remaining ``unimportant'' ones. The leaf $X$ will be bifurcated into a left child $Y$ and a right child $Z$ in line~26, containing those important and unimportant variables, respectively. Meanwhile, $v_Y$ and $v_Z$ will be calculated according to $\bm{s}$, and the number of visits is 0, i.e., $n_Y=n_Z=0$. Finally, MCTS-VS performs back-propagation in line~28 to re-calculate the $v$ value and increase the number of visits by 1 for each ancestor of $Y$ and $Z$. 

MCTS-VS will run until the number $t$ of performed evaluations reaches the budget $N_e$. Note that as the Monte Carlo tree may be built improperly, we use a variable $n_{bad}$ to record the number of visiting a right child node (regarded as containing unimportant variables), measuring the goodness of the tree. In line~5 of Algorithm~\ref{lamcts_vs}, $n_{bad}$ is initialized as 0. During the procedure of selecting a leaf node by UCB in line~10, $n_{bad}$ will be increased by 1 once visiting a right child node, which is updated in line~12. If $n_{bad}$ is larger than a threshold $N_{bad}$ (i.e., line~7), the current tree is regarded as bad, and will be re-initialized in line~8. \rbt{Furthermore, the frequency of re-initialization can be used to indicate whether MCTS-VS can do a good variable selection for the current problem.} For ease of understanding, we also provide an example illustration of MCTS-VS in Appendix~\ref{sec:example illustration}.

\section{Theoretical Analysis}
\label{sec:theory_analysis}

Although it is difficult to analyze the regret of MCTS-VS directly, we can theoretically analyze the influence of general variable selection by adopting the acquisition function GP-UCB. The considered general variable selection framework is as follows: after selecting a subset of variables at each iteration, the corresponding observation data (i.e., the data points sampled-so-far where only the selected variables are used) is used to build a GP model, and the next data point is sampled by maximizing GP-UCB. We use $\mathbb M_t$ to denote the sampled variable index subset at iteration $t$, and let $|\mathbb M_t| = d_t$.

\textbf{Regret Analysis.} \rbt{Let $\bm x^*$ denote an optimal solution. We analyze the cumulative regret $R_T=\sum_{t=1}^T(f(\bm x^*) - f(\bm x^t))$,} i.e., the sum of the gap between the optimum and the function values of the selected points by iteration $T$. To derive an upper bound on $R_T$, we pessimistically assume that the worst function value, i.e., $\min_{\bm x_{[D] \setminus \mathbb M_t}} f([\bm x_{\mathbb M_t}, \bm x_{[D] \setminus\mathbb M_t}])$, given $\bm x_{\mathbb M_t}$ is returned in evaluation. As in~\cite{dropout,gpucb}, we assume that $\mathcal X\subset [0, r]^D$ is convex and compact, and $f$ satisfies the following Lipschitz assumption. 
  
\begin{assumption} \label{ass:1}
The function $f$ is a GP sample path. For some $a, b>0$, given $L>0$, the partial derivatives of $f$ satisfy that $\forall i \in [D]$, $\exists \alpha_i \geq 0$, 
\vspace{-0.5em}
\begin{equation}
\label{eq-assump1}
P\left(\sup\nolimits_{\bm{x}\in \mathcal X} \left|\partial f/\partial x_i\right|<\alpha_i L \right)\ge 1 - ae^{-(L/b)^2}.
\end{equation}\vspace{-2em}
\end{assumption}
Based on Assumption~\ref{ass:1}, we define $\alpha^*_{i}$ to be the minimum value of $\alpha_{i}$ such that Eq.~(\refeq{eq-assump1}) holds, which characterizes the importance of the $i$-th variable $x_i$. The larger $\alpha^*_{i}$, the greater influence of $x_i$ on the function $f$. Let $\alpha_{\max} = \max\nolimits_{i \in [D]} \alpha^*_{i}$. 

Theorem~\ref{the:regret} gives an upper bound on the cumulative regret $R_T$ with high probability for general variable selection methods. The proof is inspired by that of GP-UCB without variable selection~\citep{gpucb} and provided in Appendix~\ref{sec:theory:theorem}. If we select all variables each time (i.e., $\forall t: \mathbb{M}_t=[D]$) and assume $\forall i: \alpha^*_i \leq 1$, the regret bound Eq.~(\refeq{eq-main-theo}) becomes $R_T \le \sqrt{C_1T\beta^*_{T}\gamma_T}+2$, which is consistent with~\citep{gpucb}. Note that $\forall t: |\mathbb M_t|=d_t=D$ in this case, which implies that $\beta_t$ increases with $t$, leading to $\beta^*_T=\beta_T$. We can see that using variable selection will increase $R_T$ by $2\sum_{t=1}^{T}\sum_{i\in [D]\setminus \mathbb M_t}\alpha^*_{i}Lr$, related to the importance (i.e., $\alpha^*_i$) of unselected variables at each iteration. The more important variables unselected, the larger $R_T$. Meanwhile, the term $\sqrt{C_1T\beta_{T}^*\gamma_T}$ will decrease as $\beta_{T}^*$ relies on the number $d_t$ of selected variables positively. Ideally, if the unselected variables at each iteration are always unrelated (i.e., $\alpha^*_{i}\!=\!0$), the regret bound will be better than that of using all variables~\citep{gpucb}.

\begin{theorem}
$\forall \delta \in (0,1)$, let $\beta_t = 2\log(4\pi_{t}/\delta)+2d_t\log(d_tt^2br\sqrt{\log(4Da/\delta)})$ and $L=b\sqrt{\log(4Da/\delta)}$, where \rbt{$r$ is the upper bound on each variable,} and $\{\pi_t\}_{t\ge 1}$ satisfies $\sum_{t\geq 1} \pi^{-1}_t=1$ and $\pi_t >0$. Let $\beta_{T}^* = \max\nolimits_{1\le i\le T} \beta_{t}$. At iteration $T$, the cumulative regret\vspace{-0.3em}
\begin{equation}
\begin{aligned}
\label{eq-main-theo}
R_T \le \sqrt{C_1T\beta_{T}^*\gamma_T}+2\alpha_{\max}+2\sum\nolimits_{t=1}^{T}\sum\nolimits_{i\in [D]\setminus \mathbb M_t}\alpha^*_{i}Lr
\end{aligned}\vspace{-0.5em}
\end{equation}
holds with probability at least $1\!-\!\delta$, where $C_1$ is a constant, $\gamma_{T}\! =\! \max\nolimits_{|\mathcal D| = T} I(\bm{y}_{\mathcal D},\bm{f}_{\mathcal D})$, $I(\cdot,\cdot)$ is the information gain, and $\bm{y}_{\mathcal D}$ and $\bm{f}_{\mathcal D}$ are the noisy and true observations of a set $\mathcal{D}$ of points, respectively.
\label{the:regret}\vspace{-1em}
\end{theorem}
By selecting $d$ variables randomly at each iteration and assuming that $r=1$ and $\forall i: \alpha^*_i \leq 1$, it has been proved~\citep{dropout} that the cumulative regret of Dropout satisfies
\begin{equation}
\label{eq-dropout}
R_T \leq \sqrt{C_1T\beta_{T}\gamma_T}+2+2TL(D-d).
\vspace{-0.2em}
\end{equation}
In this case, we have $d_t=|\mathbb{M}_t|=d$, $r=1$ and $\forall i: \alpha^*_i \leq 1$. Thus, Eq.~(\refeq{eq-main-theo}) becomes 
\begin{equation}
  R_T \leq \sqrt{C_1T\beta_{T}^*\gamma_T}+2+ 2TL(D-d).
\vspace{-0.2em}
\end{equation}
Note that $\beta_{T}^*=\beta_T$ here, as $\beta_t$ increases with $t$ given $d_t=d$. This implies that our bound Eq.~(\refeq{eq-main-theo}) for general variable selection is a generalization of Eq.~(\refeq{eq-dropout}) for Dropout~\citep{dropout}. In~\cite{vsbo}, a regret bound analysis has also been performed for variable selection, by optimizing over $d$ fixed important variables and using a common parameter $\alpha$ to characterize the importance of all the other $D-d$ variables.


\textbf{Computational Complexity Analysis.} The computational complexity of one iteration of BO depends on three critical components: fitting a GP surrogate model, maximizing an acquisition function and evaluating a sampled point. If using the squared exponential kernel, the computational complexity of fitting a GP model at iteration $t$ is $\mathcal O(t^3+t^2d_t)$. Maximizing an acquisition function is related to the optimization algorithm. If we use the Quasi-Newton method to optimize GP-UCB, the computational complexity is $\mathcal O(m(t^2 + td_t + d^2_t))$~\citep{NoceWrig06}, where $m$ denotes the Quasi-Newton's running rounds. The cost of evaluating a sampled point is fixed. Thus, by selecting only a subset of variables, instead of all variables, to optimize, the computational complexity can be decreased significantly. The detailed analysis is provided in Appendix~\ref{sec:complexity}. 

\textbf{Insight.} The above regret and computational complexity analyses have shown that variable selection can reduce the computational complexity while increasing the regret. Given the number $d_t$ of variables to be selected, a good variable selection method should select as important variables as possible, i.e., variables with as large $\alpha^*_i$ as possible, which may help to design and evaluate variable selection methods. The experiments in Section~\ref{sec:synthetic} will show that MCTS-VS can select a good subset of variables while maintaining a small computational complexity.


\section{Experiment}

To examine the performance of MCTS-VS, we conduct experiments on different tasks, including synthetic functions, NAS-bench problems and MuJoCo locomotion tasks, to compare MCTS-VS with other black-box optimization methods. For MCTS-VS, we use the same hyper-parameters except $C_p$, which is used for calculating UCB in Eq.~(\ref{eq-MCTS-UCB}). For Dropout and embedding-based methods, we set the parameter $d$ to the number of valid dimensions for synthetic functions, and a reasonable value for real-world problems. The hyper-parameters of the same components of different methods are set to the same. We use five identical random seeds (2021--2025) for all problems and methods. More details about the settings can be found in Appendix~\ref{appendix:hp}. \songl{Our code is available at \url{https://github.com/lamda-bbo/MCTS-VS}}.

\subsection{Synthetic Functions}
\label{sec:synthetic}

We use Hartmann ($d=6$) and Levy ($d=10$) as the synthetic benchmark functions, and extend them to high dimensions by adding unrelated variables as~\cite{alebo,hesbo,rembo}. For example, Hartmann$6$\_$300$ has the dimension $D=300$, and is generated by appending $294$ unrelated dimensions to Hartmann. The variables affecting the value of $f$ are called \emph{valid variables}. 

\textbf{Effectiveness of Variable Selection.} Dropout~\citep{dropout} is the previous variable selection method which randomly selects $d$ variables in each iteration, while our proposed MCTS-VS applies MCTS to automatically select important variables. We compare them against vanilla BO \citep{bosurvey1} without variable selection. The first two subfigures in Figure~\ref{fig:synthetic1} show that Dropout-BO and MCTS-VS-BO are better than vanilla BO, implying the effectiveness of variable selection. We can also see that MCTS-VS-BO performs the best, implying the superiority of MCTS-based variable selection over random selection.

We also equip MCTS-VS and Dropout with the advanced BO algorithm TuRBO~\citep{turbo}, resulting in MCTS-VS-TuRBO and Dropout-TuRBO. The last two subfigures in Figure~\ref{fig:synthetic1} show the similar results except that MCTS-VS-TuRBO needs more evaluations to be better than Dropout-TuRBO. This is because TuRBO costs more evaluations than BO on the same selected variables, and thus needs more evaluations to generate sufficient samples for an accurate estimation of the variable score in Eq.~(\refeq{eq-score}).\vspace{-0.7em}

\begin{figure*}[htbp!]
    \centering
    \hspace{-1em}\subfigure{\includegraphics[width=0.45\textwidth]{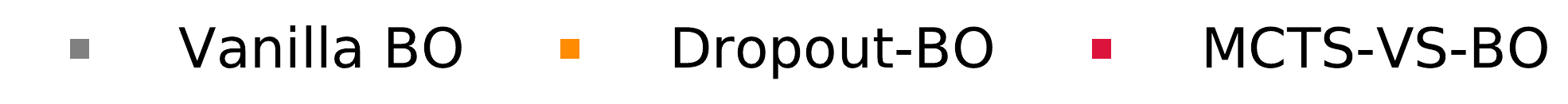}}\hspace{-0.5em}
    \subfigure{\includegraphics[width=0.45\textwidth]{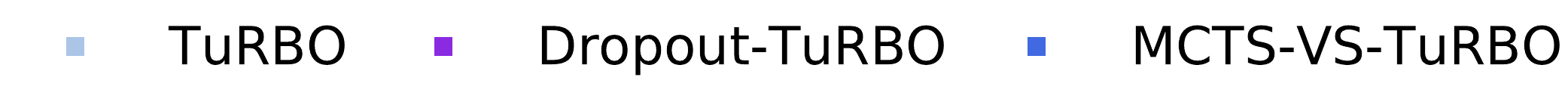}}\\\vspace{-1em}
    \centering
    \subfigure{\includegraphics[width=0.24\textwidth]{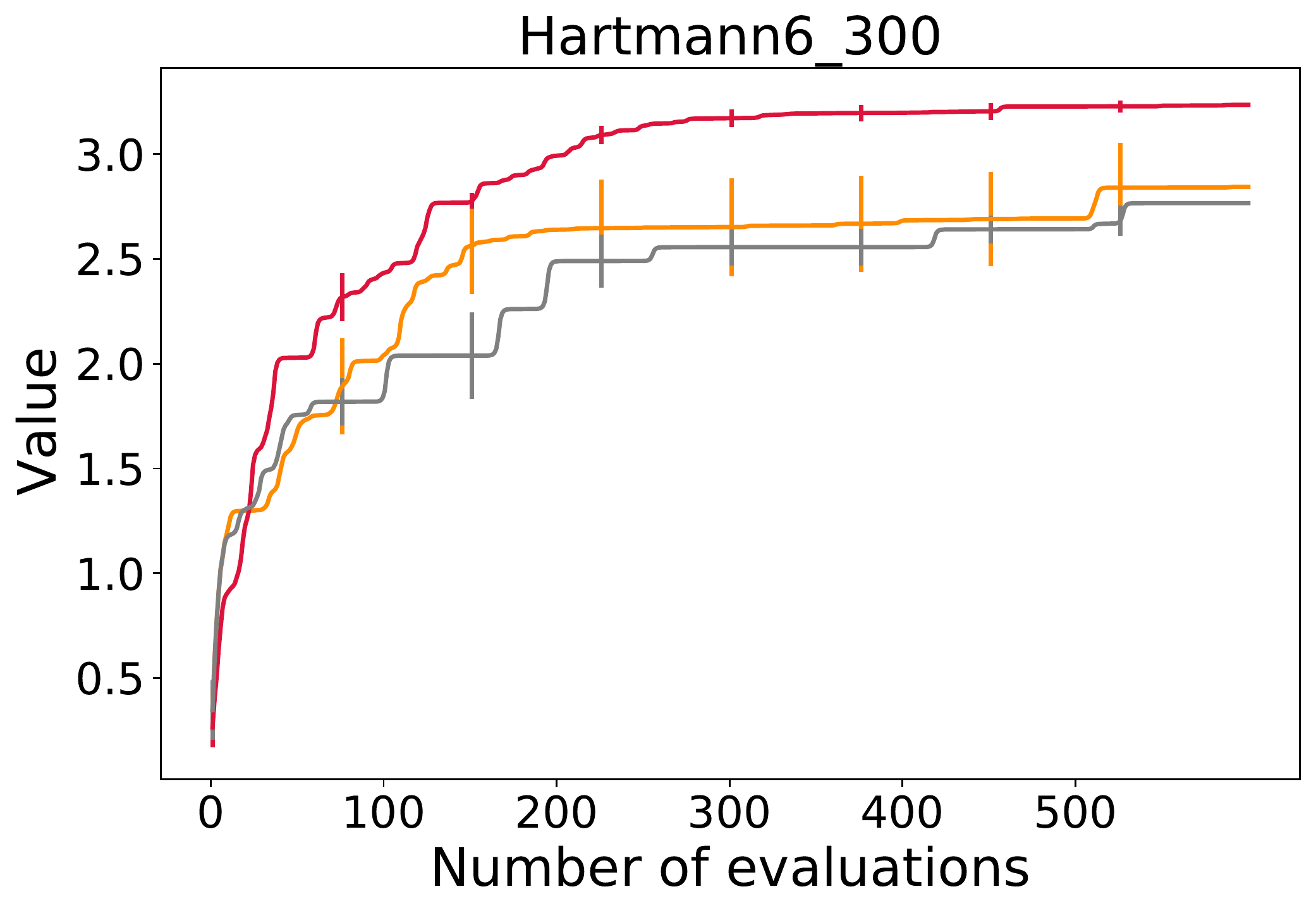}}
    \subfigure{\includegraphics[width=0.24\textwidth]{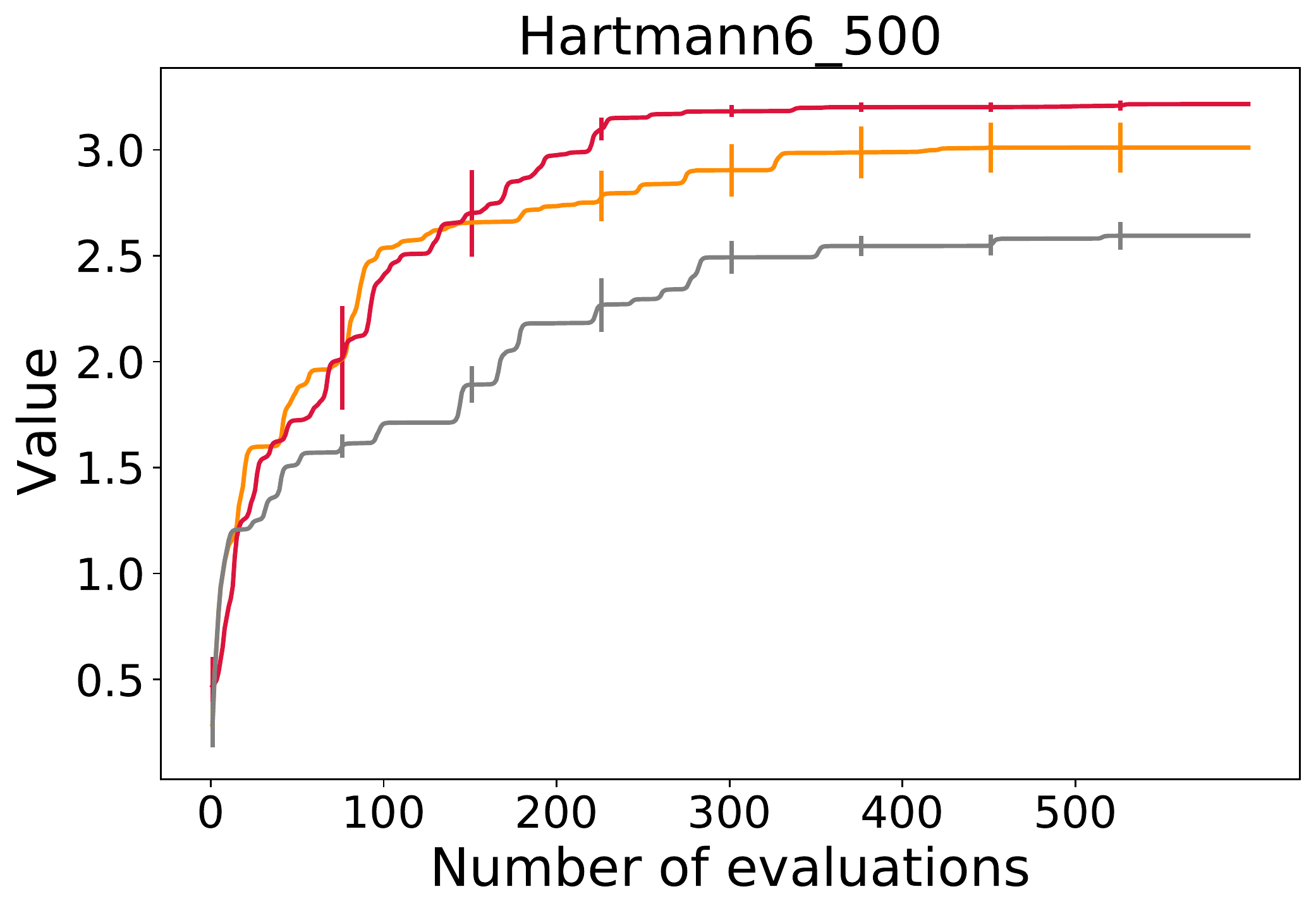}}
    \subfigure{\includegraphics[width=0.24\textwidth]{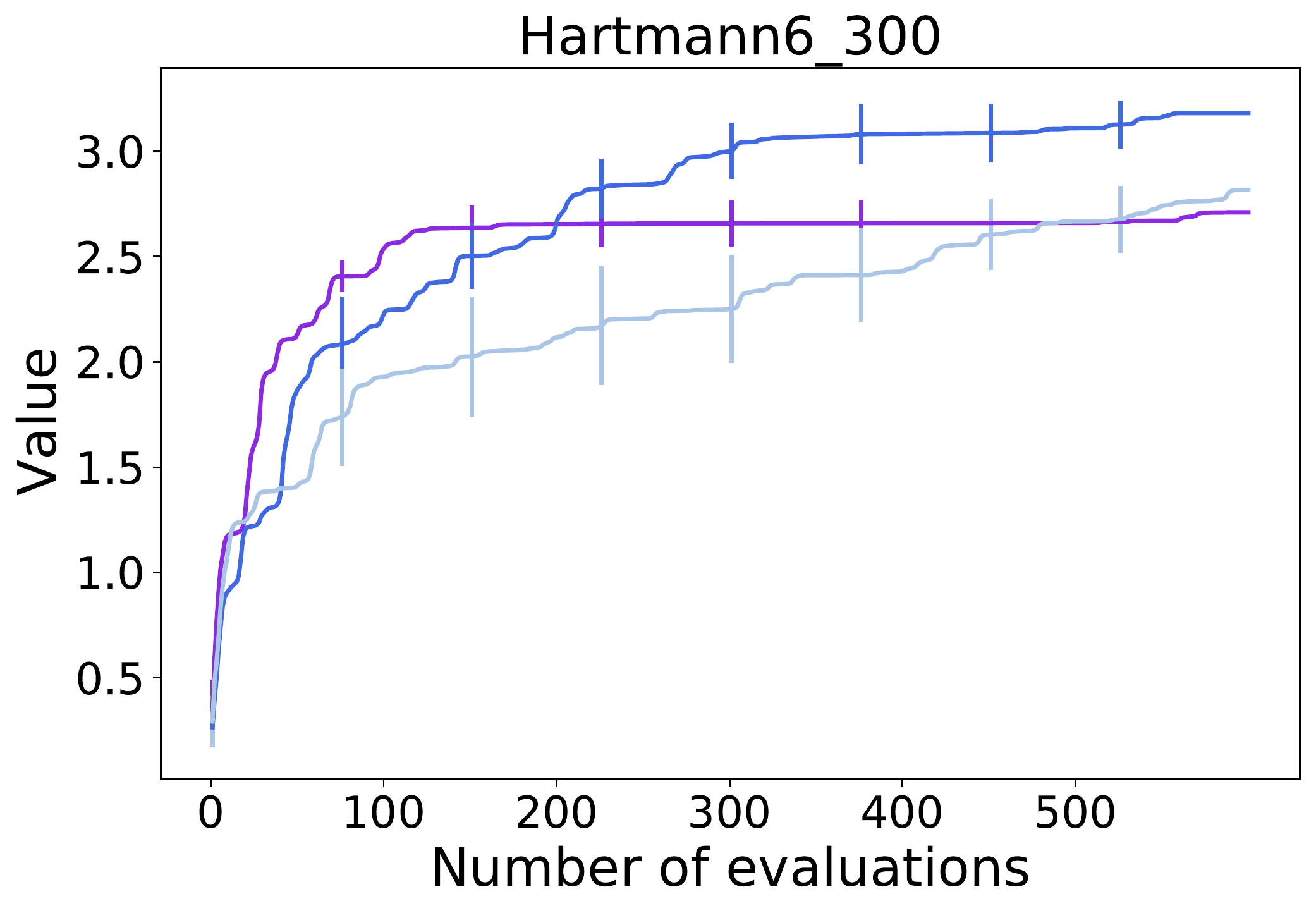}}
    \subfigure{\includegraphics[width=0.24\textwidth]{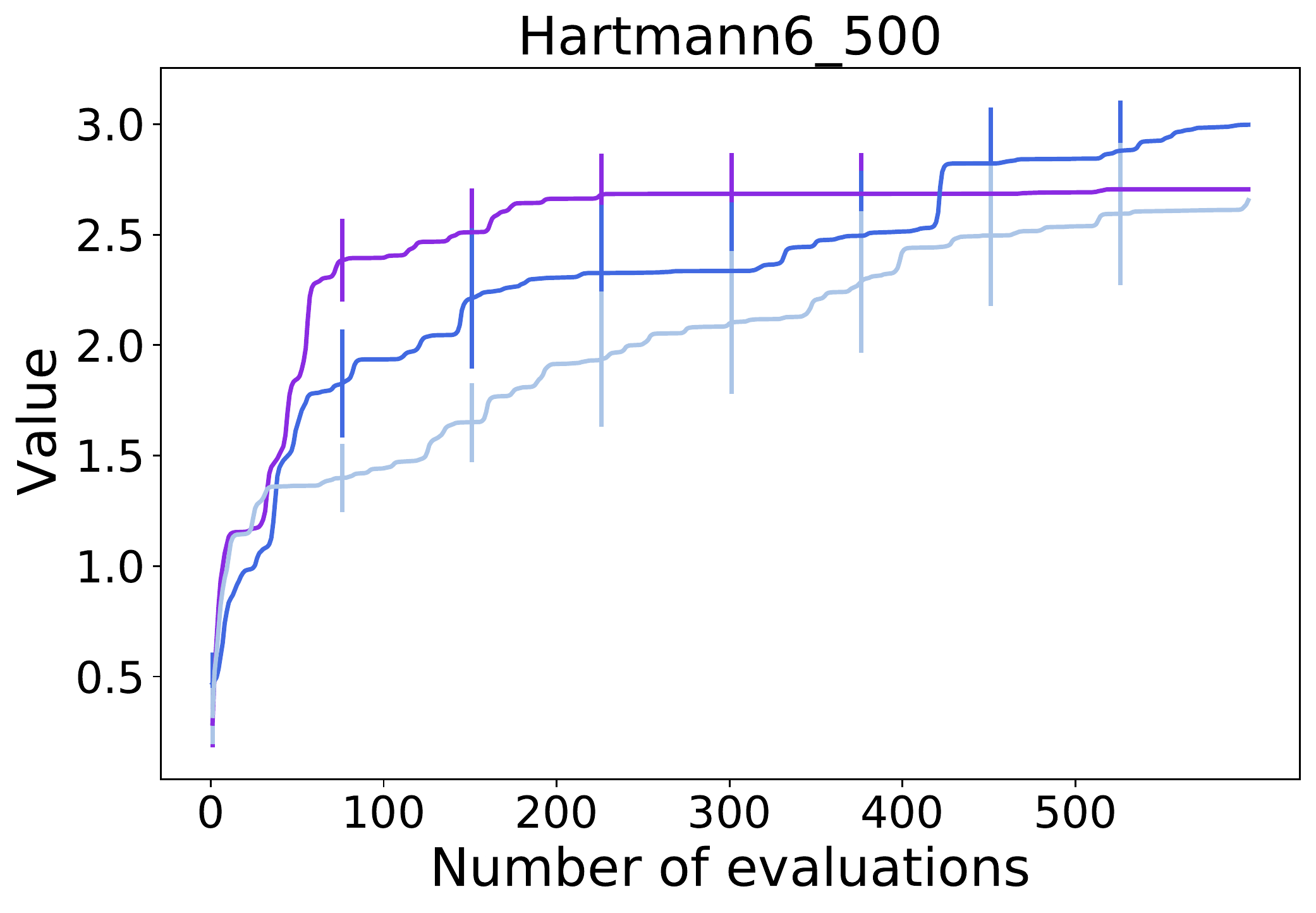}}\vspace{-0.8em}
    \caption{Performance comparison among the two variable selection methods (i.e., MCTS-VS and Dropout) and the BO methods (i.e., Vanilla BO and TuRBO) on \rbt{two} synthetic functions.}
    \label{fig:synthetic1}\vspace{-0.5em}
\end{figure*}

\textbf{Comparison with State-of-The-Art Methods.} We compare MCTS-VS with the state-of-the-art methods, including TuRBO~\cite{turbo}, LA-MCTS-TuRBO~\cite{lamcts}, \songl{SAASBO~\cite{saas}}, HeSBO~\cite{hesbo}, ALEBO~\cite{alebo} and CMA-ES~\citep{cmaes}. TuRBO fits a collection of local models to optimize in the trust regions for overcoming the homogeneity of the global model and over-exploration. LA-MCTS-TuRBO applies MCTS to partition the search space and uses TuRBO to optimize in a small sub-region. \songl{SAASBO uses sparsity-inducing prior to select variables implicitly.} HeSBO and ALEBO are state-of-the-art embedding methods. CMA-ES is a popular evolutionary algorithm. We also implement VAE-BO by combining VAE~\cite{Kingma2014AutoEncodingVB} with vanilla BO directly, as a baseline of learning-based embedding. For MCTS-VS, we implement the two versions of MCTS-VS-BO and MCTS-VS-TuRBO, i.e., MCTS-VS equipped with vanilla BO and TuRBO.

As shown in Figure~\ref{fig:synthetic2}, MCTS-VS can achieve the best performance except on Levy$10$\_$100$, where it is a little worse than TuRBO. For low-dimensional functions (e.g., $D=100$ for Levy$10$\_$100$), TuRBO can adjust the trust region quickly while MCTS-VS needs samples to estimate the variable score. But as the dimension increases, the search space increases exponentially and it becomes difficult for TuRBO to adjust the trust region; while the number of variables only increases linearly, making MCTS-VS more scalable. \songl{SAASBO has similar performance to MCTS-VS due to the advantage of sparsity-inducing prior.} HeSBO is not stable, which has a moderate performance on Hartmann but a relatively good performance on Levy. \songl{Note that we only run SAASBO and ALEBO for $200$ evaluations on Hartmann functions because it has already taken more than hours to finish one iteration when the number of samples is large. More details about runtime are shown in Table~\ref{tab:time}.} VAE-BO has the worst performance, suggesting that the learning algorithm in high-dimensional BO needs to be designed carefully. We also conduct experiments on extremely low and high dimensional variants of Hartmann (i.e., Hartmann$6$\_$100$ and Hartmann$6$\_$1000$), showing that MCTS-VS still performs well, and perform the significance test by running each method more times. Please see Appendix~\ref{appendix:additionalexperiments}.\vspace{-1em}

\begin{figure*}[h!]
    \centering
    \subfigure{\includegraphics[width=0.65\textwidth]{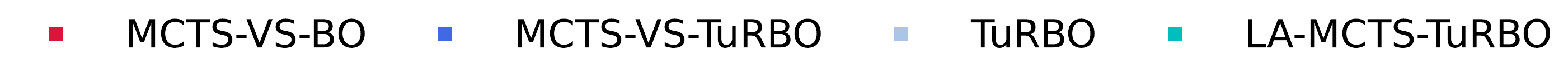}}\\
    \vspace{-1.3em}
    \subfigure{\includegraphics[width=0.6\textwidth]{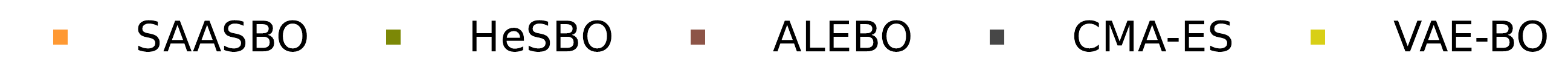}}\\
    \vspace{-1em}
    \centering
    \subfigure{\includegraphics[width=0.24\textwidth]{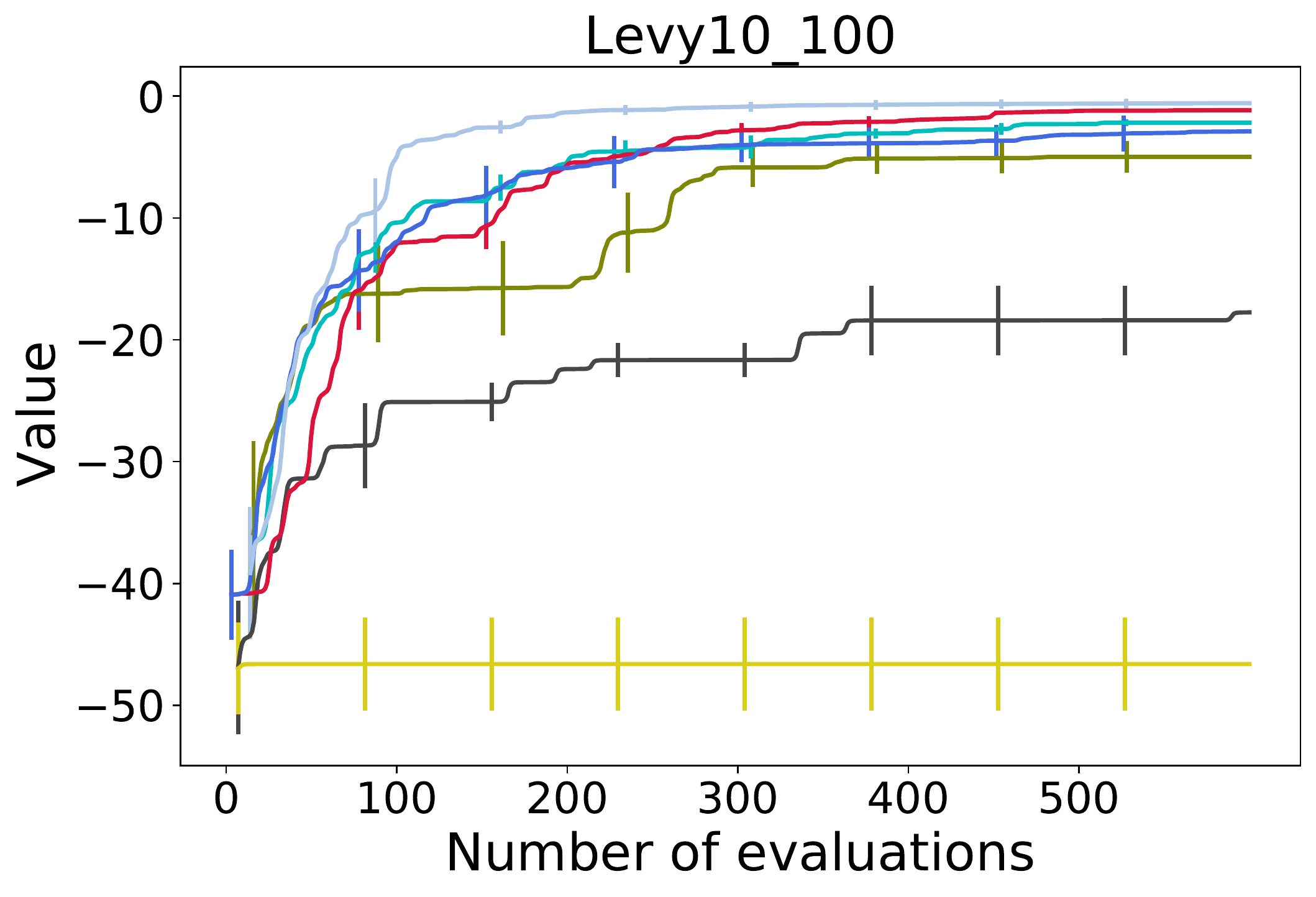}}
    \subfigure{\includegraphics[width=0.24\textwidth]{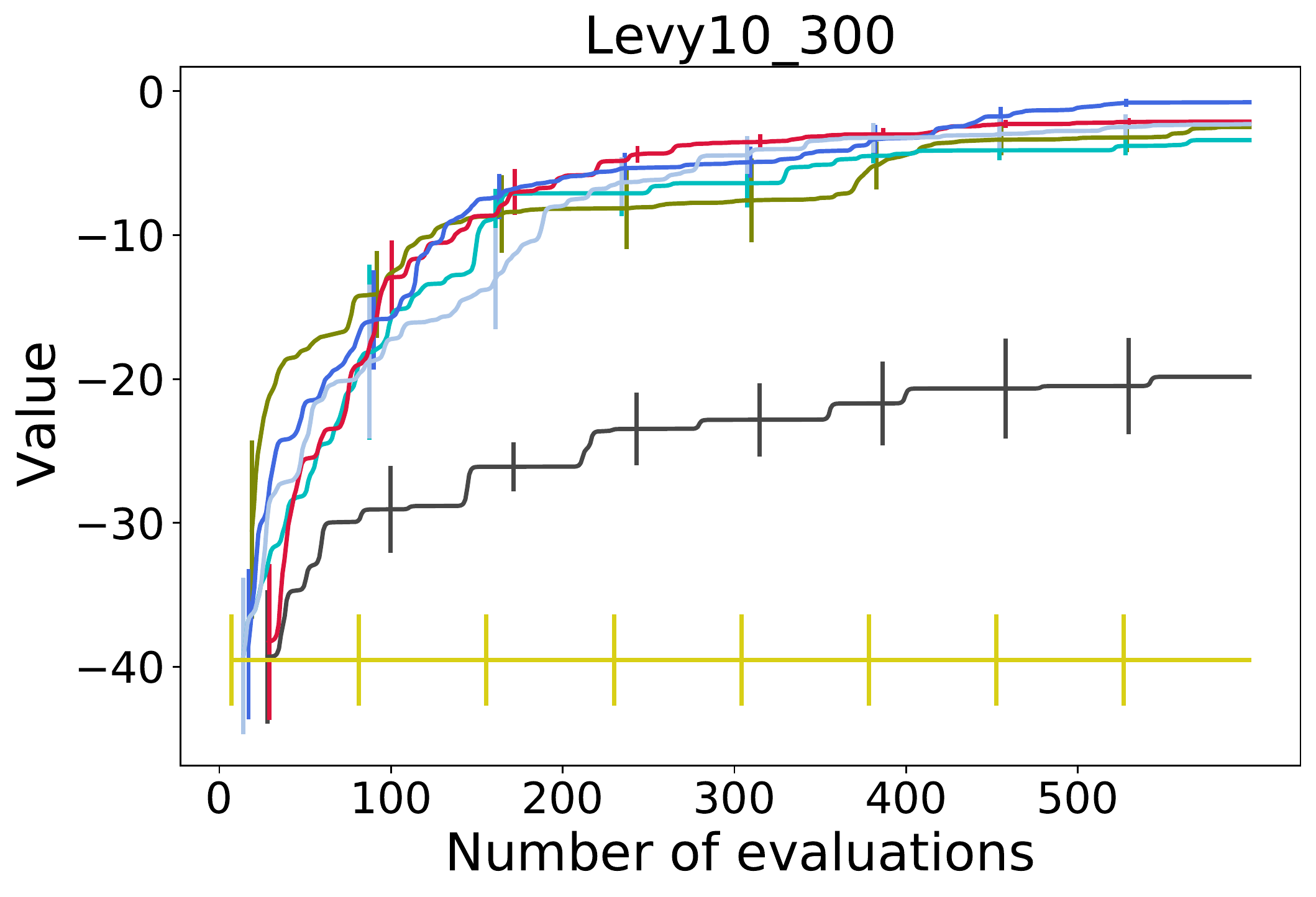}}
    \subfigure{\includegraphics[width=0.24\textwidth]{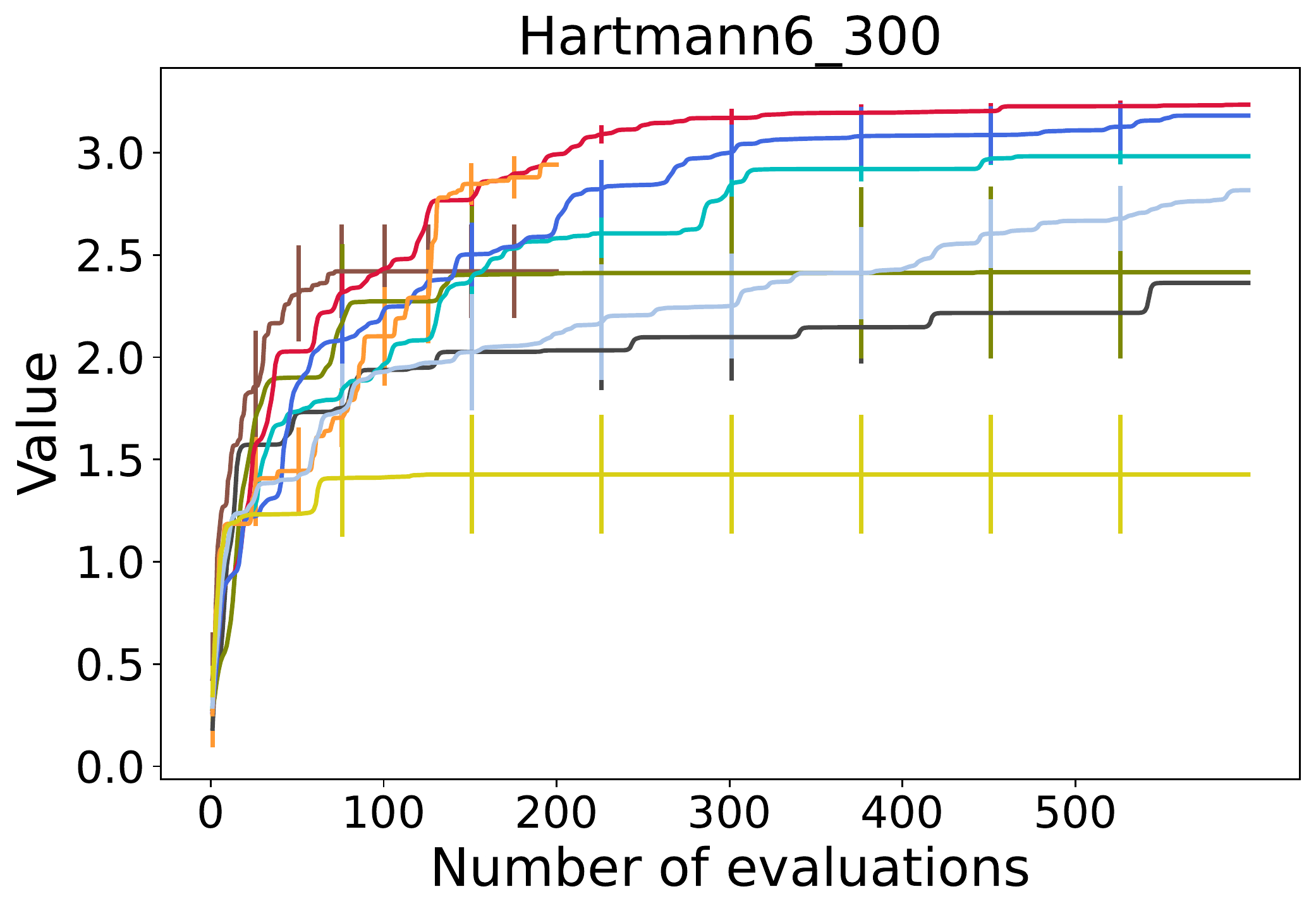}}
    \subfigure{\includegraphics[width=0.24\textwidth]{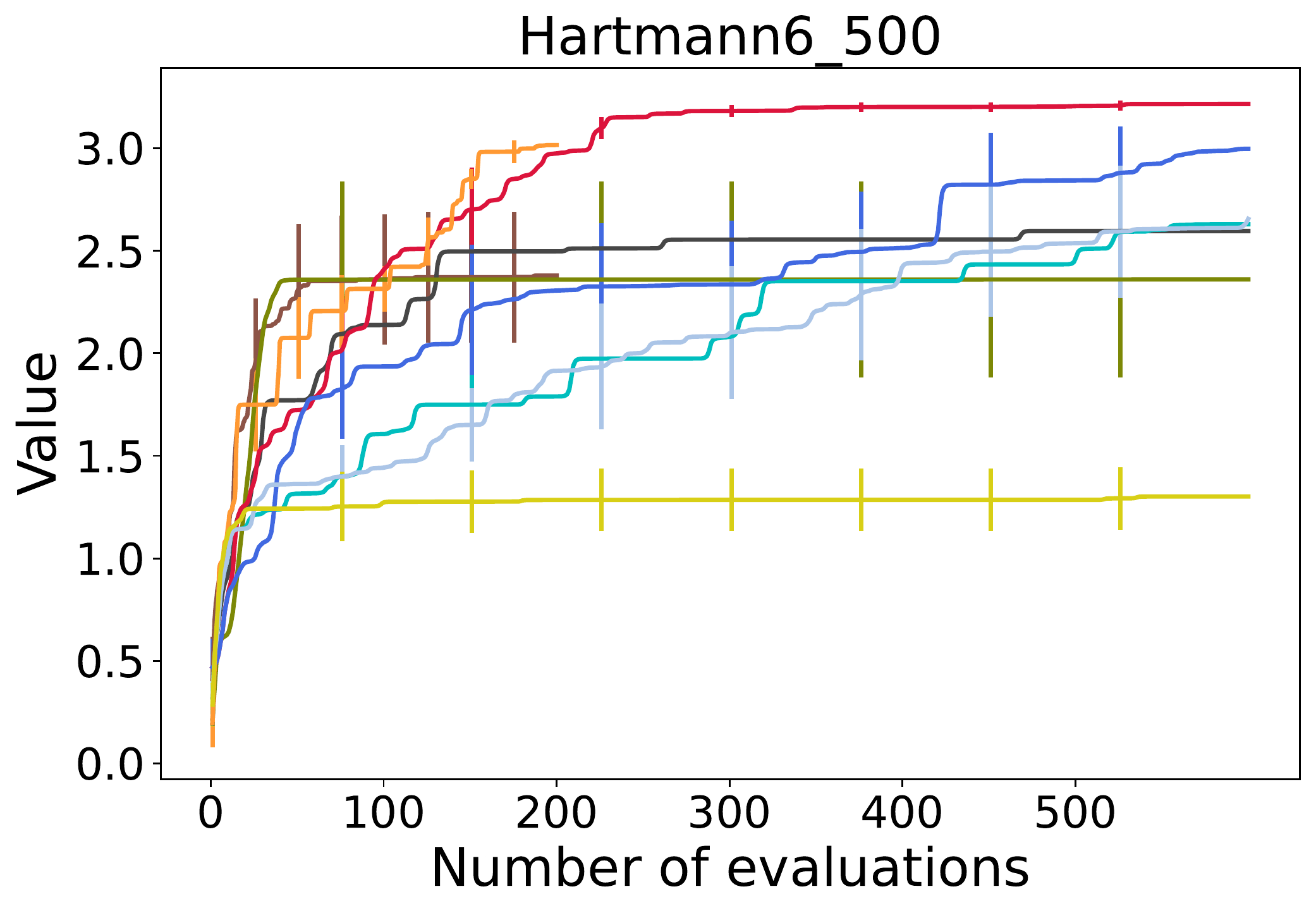}}\vspace{-0.8em}
    \caption{Comparison among MCTS-VS and state-of-the-art methods on synthetic functions.}
    \label{fig:synthetic2}\vspace{-0.5em}
\end{figure*}

Next, we compare the practical running overheads of these methods. We run each method for $100$ evaluations independently using $30$ different random seeds, and calculate the average wall clock time. The results are shown in Table~\ref{tab:time}. As expected, when using variable selection (i.e., Dropout and MCTS-VS), the time is less than that of Vanilla BO or TuRBO, because we only optimize a subset of variables. MCTS-VS is a little slower than Dropout, which is because MCTS-VS needs to build the search tree and calculate the variable score, while Dropout only randomly selects variables. MCTS-VS is much faster than LA-MCTS-TuRBO, showing the advantage of partitioning the variables to partitioning the search space. \songl{SAASBO optimizes all variables instead of only a subset of variables and uses No-U-Turn sampler (NUTS) to inference, consuming $\times 500 \text{ --} \times 1000$ time.} HeSBO and ALEBO consume $\times 10 \text{ --} \times 500$ time compared with the variable selection methods. CMA-ES is very fast because it does not need to fit a GP model or optimize an acquisition function. The reasons for the small running overhead of MCTS-VS can be summarized as follows: 1)~it only optimizes a selected subset of variables; 2)~the depth of the search tree is shallow, i.e., $O(\log D)$ in expectation and less than $D$ in the worse case; 3)~the variable score vector in Eq.~(\refeq{eq-score}) is easy to calculate for bifurcating a tree node.

\textbf{Why MCTS-VS Can Perform Well.} The theoretical results have suggested that a good variable selection method should select as important variables as possible. Thus, we compare the quality of the variables selected by MCTS-VS and Dropout (i.e., random selection), measured by the recall $d^*_t/d$, where $d$ is the number of valid variables, and $d^*_t$ is the number of valid variables selected at iteration $t$. Dropout randomly selects $d$ variables at each iteration, and thus, the recall is $d/D$ in expectation. For MCTS-VS, we run MCTS-VS-BO for $600$ evaluations on five different random seeds, and calculate the average recall. As shown in Table~\ref{tab:theory}, the average recall of MCTS-VS is much larger than that of Dropout, implying that MCTS-VS can select better variables than random selection, and thus achieve a good performance as shown before. Meanwhile, the recall between $0.35$ and $0.433$ of MCTS-VS also implies that the variable selection method could be further improved.


\subsection{Real-World Problems}

We further compare MCTS-VS with the baselines on real-world problems, including NAS-Bench-101~\citep{nas101}, NAS-Bench-201~\citep{Dong2020NAS-Bench-201}, Hopper and Walker2d. NAS-Bench problems are popular benchmarks in high-dimensional BO. Hopper and Walker2d are robot locomotion tasks in MuJoCo~\citep{mujoco}, which is a popular black-box optimization benchmark and much more difficult than NAS-Bench problems. The experimental results on more real-world problems can refer to Appendix~\ref{appendix:additionalexperiments}.

\begin{table*}[t!]
\caption{Wall clock time (in seconds) comparison among different methods.}
\vskip 0.05in
\begin{center}
\begin{small}
\begin{sc}
\begin{tabular}{lcccc}
\toprule
Method & Levy10\_100 & Levy10\_300 & Hartmann6\_300 & Hartmann6\_500 \\
\midrule
Vanilla BO    & 3.190    & 4.140    & 4.844   & 5.540    \\
Dropout-BO    & 2.707    & 3.225    & 3.237   & 3.685 \\ 
MCTS-VS-BO    & 2.683    & 3.753    & 3.711   & 4.590 \\
\midrule
TuRBO         & 8.621    & 9.206    & 9.201   & 9.754 \\
LA-MCTS-TuRBO & 14.431   & 22.165   & 25.853  & 34.381 \\
MCTS-VS-TuRBO & 4.912    & 5.616    & 5.613   & 5.893 \\
\midrule
SAASBO        & \songl{/}  & \songl{/}  & 2185.678& 4163.121 \\
HeSBO         & 220.459  & 185.092  & 51.678  & 55.699 \\
ALEBO         & \songl{/} & \songl{/} & 470.714 & 512.641 \\
CMA-ES        & 0.030    & 0.043    & 0.043   & 0.045 \\
\bottomrule
\end{tabular}
\end{sc}
\end{small}
\end{center}
\vskip -0.2in
\label{tab:time}
\vspace{-0.5em}
\end{table*}

\begin{table*}[t!]
\caption{Recall comparison between MCTS-VS and Dropout.}
\vskip 0.05in
\begin{center}
\begin{small}
\begin{sc}
\begin{tabular}{lcccc}
\toprule
Method & Levy10\_100 & Levy10\_300 & Hartmann6\_300 & Hartmann6\_500  \\
\midrule
Dropout & 0.100 & 0.030 & 0.020 & 0.012  \\ 
MCTS-VS & 0.429 & 0.433 & 0.352 & 0.350  \\ 
\bottomrule
\end{tabular}
\end{sc}
\end{small}
\end{center}
\vskip -0.2in
\label{tab:theory}\vspace{-0.5em}
\end{table*}


\textbf{NAS-Bench Problems.} NAS-Bench-101 is a tabular data set that maps convolutional neural network architectures to their trained and evaluated performance on CIFAR-10, and we create a constrained problem with $D=36$ in the same way as~\citep{alebo}. NAS-Bench-201 is an extension to NAS-Bench-101, leading to a problem with $D=30$ but without constraints. Figure~\ref{fig:nas} shows the results with the wall clock time as the $x$-axis, where the gray dashed line denotes the optimum. The results using the number of evaluations as the $x$-axis are provided in Appendix~\ref{appendix:additionalexperiments}, showing that the performance of BO-style methods is similar, as already observed in~\citep{alebo}. This may be because there are many structures whose objective values are close to the optimum. But when considering the actual runtime, MCTS-VS-BO is still clearly better as shown in Figure~\ref{fig:nas}, due to the advantage of variable selection. \songl{We also provide results on more NAS-Bench problems, including NAS-Bench-1Shot1~\cite{Zela2020NAS-Bench-1Shot1}, TransNAS-Bench-101~\cite{transnas} and NAS-Bench-ASR~\cite{mehrotra2021nasbenchasr} in Appendix~\ref{appendix:additionalexperiments}.}

\textbf{MuJoCo Locomotion Tasks.} Next we turn to the more difficult MuJoCo tasks in RL. The goal is to find the parameters of a linear policy maximizing the accumulative reward. Different from previous problems, the objective $f$ (i.e., the accumulative reward) is highly stochastic here, making it difficult to solve. We use the mean of three independent evaluations to estimate $f$, and limit the evaluation budget to $1200$ due to expensive evaluation. Note that we do not run \songl{SAASBO,} ALEBO, and VAE-BO because \songl{SAASBO} and ALEBO are extremely time-consuming, and VAE-BO behaves badly in previous experiments. The results are shown in Figure~\ref{fig:rl}. TuRBO behaves well on Hopper with a low dimension $D=33$, and MCTS-VS-TuRBO, combining the advantage of variable selection and TuRBO, achieves better performance, outperforming all the other baselines. On Walker2d with a high dimension $D=102$, MCTS-VS-BO performs the best, because of the good scalability. Most methods have large variance due to the randomness of $f$. For HeSBO, we have little knowledge about the parameter $d$, and use $10$ and $20$ for Hopper and Walker2d, respectively. Its performance may be improved by choosing a better $d$, which, however, requires running the experiment many times, and is time-consuming. Note that on the two MuJoCo tasks, Hopper and Walker2d, each variable is valid. The good performance of MCTS-VS may be because optimizing only a subset of variables is sufficient for
achieving the goal and MCTS-VS can select them. For example, the Walker2D robot consists of four main body parts: a torso, two thighs, two legs and two feet, where adjacent ones are connected by two hinges. The goal is to move forward by optimizing the hinges, each of which is valid. But even locking the hinges between legs and feet, the robot can still move forward by optimizing the other hinges. This is similar to that when the ankles are fixed, a person can still walk.

\begin{figure}[t!]
    \centering
    \subfigure{\includegraphics[width=0.65\textwidth]{final-version/legend/exp2_legend_1.pdf}}\\
    \vspace{-1.3em}
    \subfigure{\includegraphics[width=0.6\textwidth]{final-version/legend/exp2_legend_2.pdf}}\\
    \vspace{-0.5em}
    \centering
    \begin{minipage}[t]{0.495\textwidth}
    \includegraphics[width=0.495\textwidth]{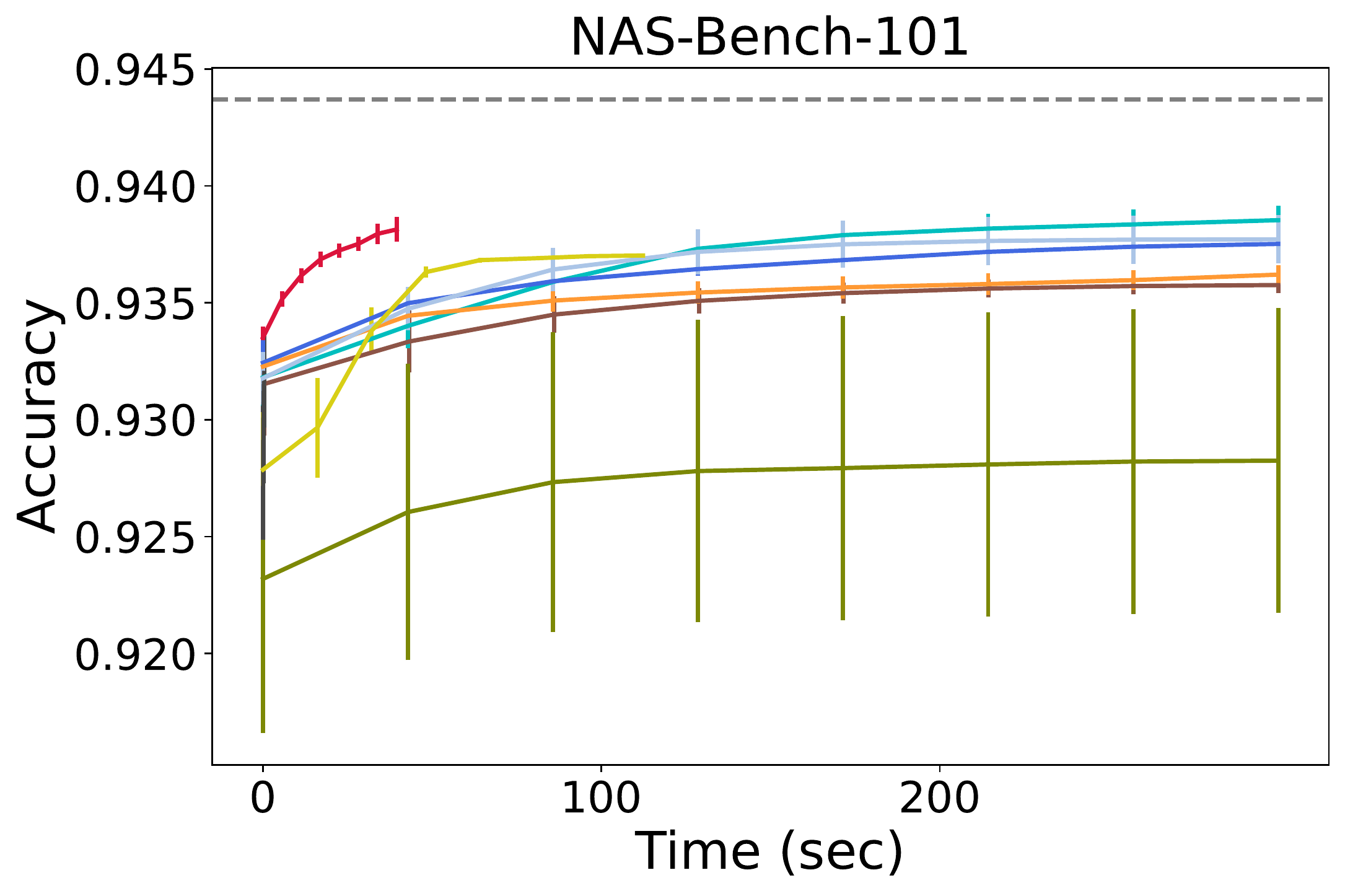}
    \includegraphics[width=0.495\textwidth]{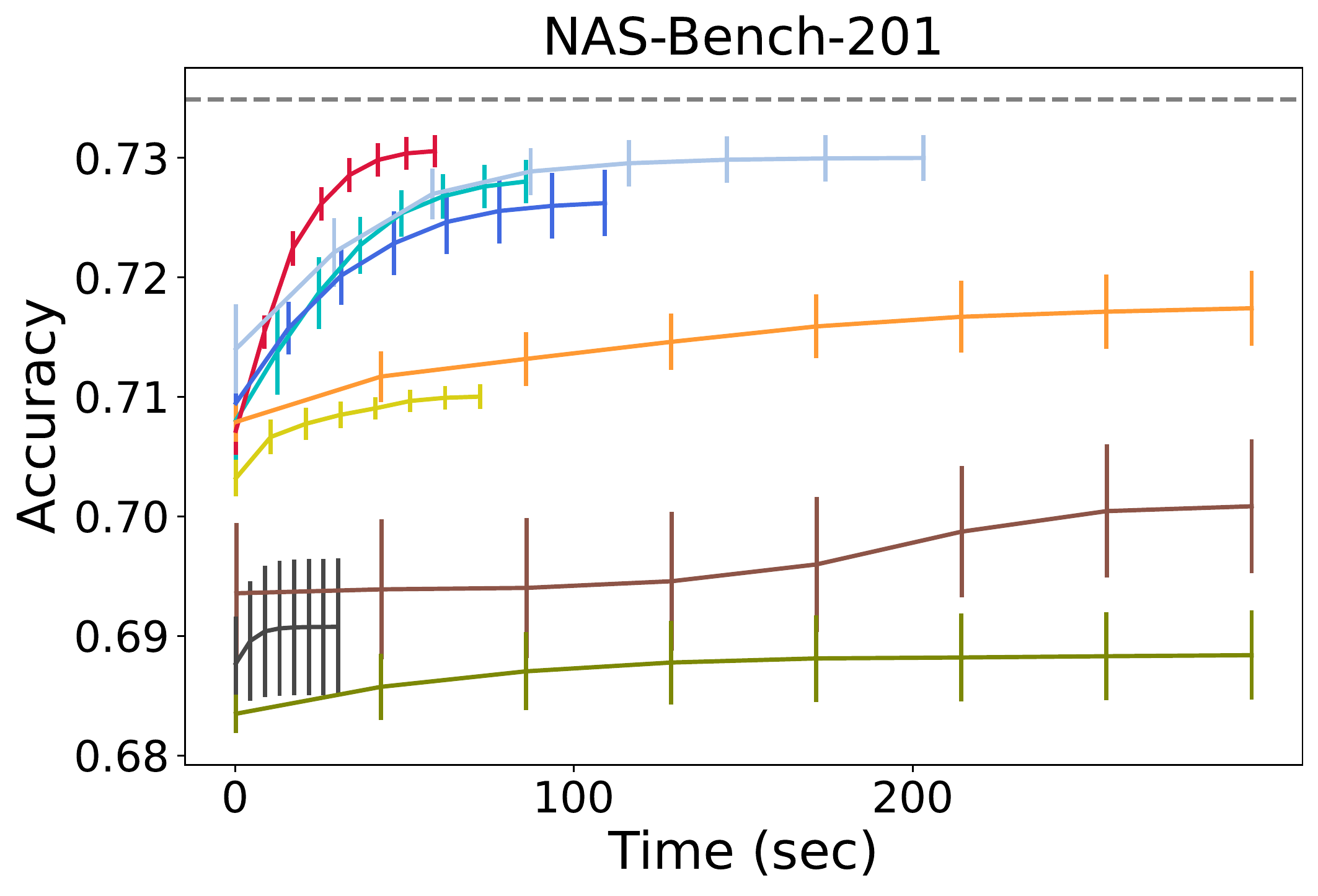}
    \vspace{-1.6em}
    \caption{Comparison on NAS-Bench.}
    \label{fig:nas}
    \end{minipage}
    \vspace{-1em}
    \begin{minipage}[t]{0.495\textwidth}
    \includegraphics[width=0.495\textwidth]{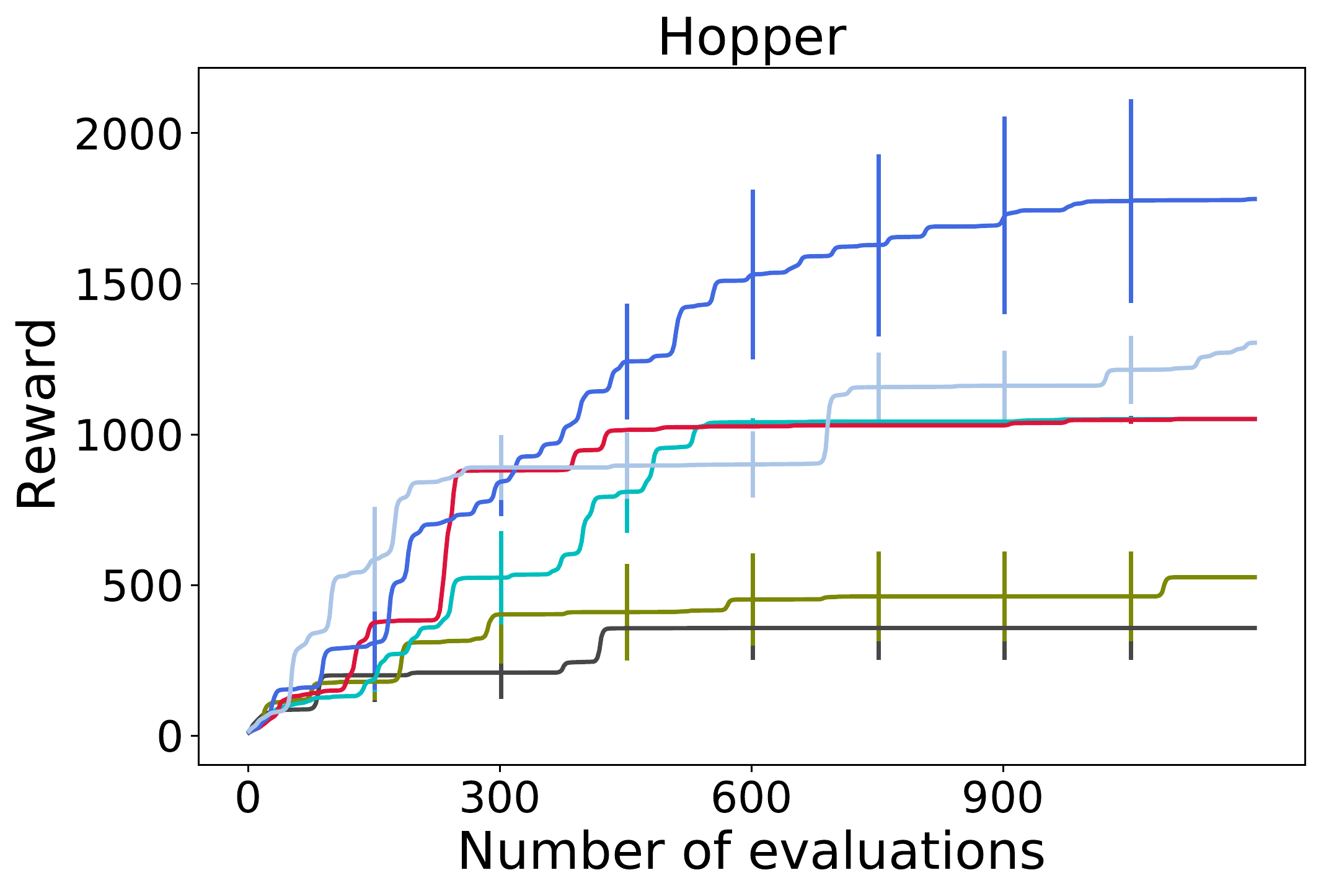}
    \includegraphics[width=0.495\textwidth]{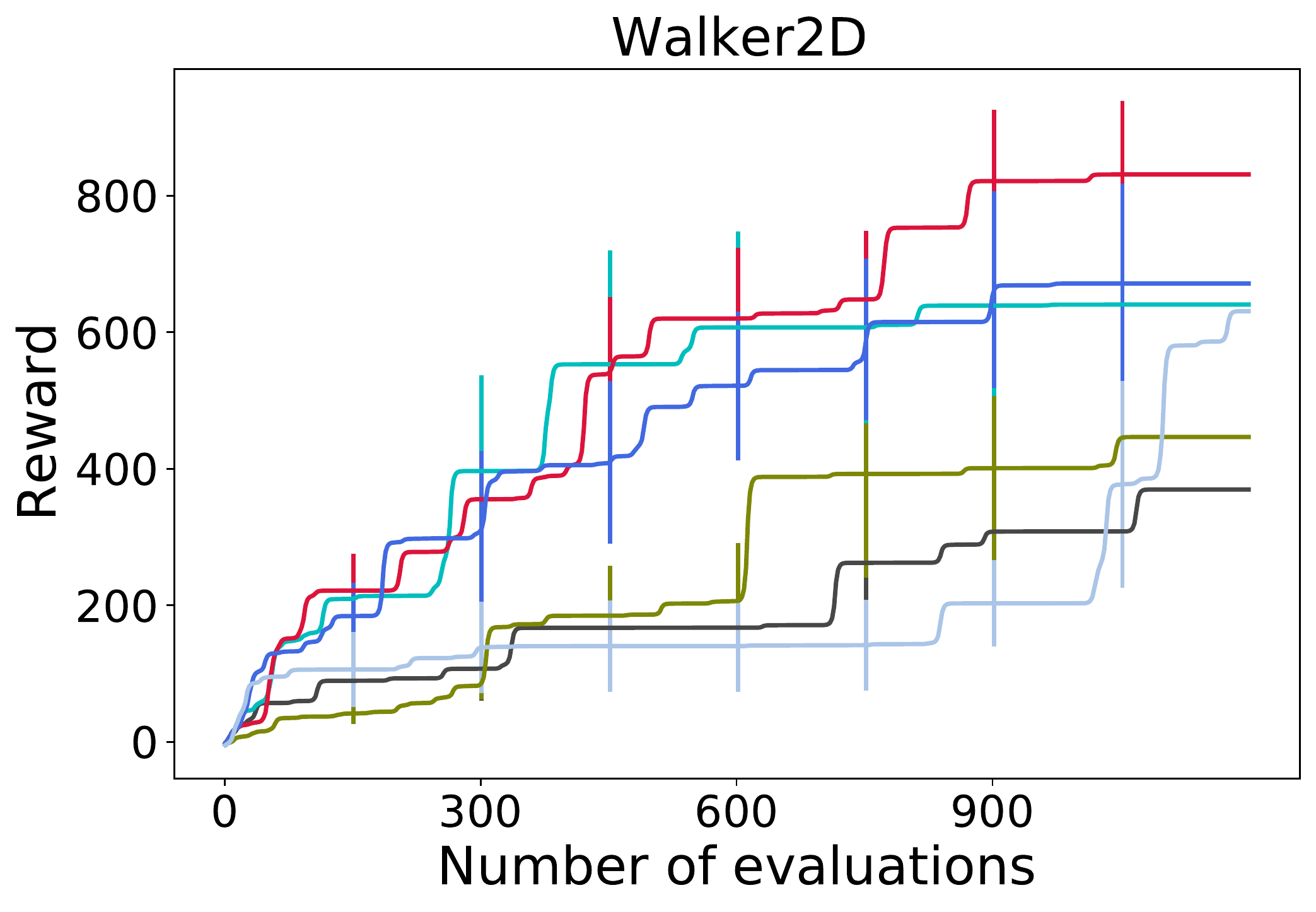}
    \vspace{-1.6em}
    \caption{Comparison on MuJoCo.}
    \label{fig:rl}
    \end{minipage}
    \vspace{-1em}
\end{figure}


\textbf{Further Studies.} We further perform sensitivity analysis about the hyper-parameters of MCTS-VS, including the employed optimizer, ``fill-in'' strategy, $C_p$ for calculating UCB in Eq.~(\ref{eq-MCTS-UCB}), number $2\times N_v\times N_s$ of sampled data in each iteration, threshold $N_{bad}$ for re-initializing a tree and $N_{split}$ for splitting a tree node. Please see Appendix~\ref{appendix:ablation}. \songl{We conduct additional experiments in Appendix~\ref{appendix:additionalexperiments}, including experiments on synthetic functions depending on a subset of variables to various extent and with increasing ratio of valid variables, examination of combining MCTS-VS with SAASBO (which can be viewed as a hierarchical variable selection method), and comparison with other variable selection methods (e.g., LASSO).}

\section{Conclusion}

In this paper, we propose the MCTS-VS method for variable selection in high-dimensional BO, which uses MCTS to recursively partition the variables into important and unimportant ones, and only optimizes those important variables. Theoretical results suggest selecting as important variables as possible, which may be of independent interest for variable selection. Comprehensive experiments on synthetic, NAS-bench and MuJoCo problems demonstrate the effectiveness of MCTS-VS. 

However, MCTS-VS relies on the assumption of low effective dimensionality, and might not work well if the percentage of valid variables is high. The amount of hyper-parameters might be another limitation, though our sensitivity analysis has shown that the performance of MCTS-VS is not sensitive to most hyper-parameters. The current theoretical analysis is for general variable selection, while it will be very interesting to perform specific theoretical analysis for MCTS-VS.


\subsection*{Acknowledgement}
The authors would like to thank reviewers for their helpful comments and suggestions. This work was supported by the NSFC (62022039, 62276124) and the Fundamental Research Funds for the Central Universities (0221-14380014).

\bibliography{ref}
\bibliographystyle{plainnat}

\section*{Checklist}


\begin{enumerate}

\item For all authors...
\begin{enumerate}
  \item Do the main claims made in the abstract and introduction accurately reflect the paper's contributions and scope?
    \answerYes{}
  \item Did you describe the limitations of your work?
    \answerYes{See the end of Section~\ref{sec:synthetic} \rbt{and the last paragraph of the paper}.}
  \item Did you discuss any potential negative societal impacts of your work?
    \answerNA{}
  \item Have you read the ethics review guidelines and ensured that your paper conforms to them?
    \answerYes{}
\end{enumerate}

\item If you are including theoretical results...
\begin{enumerate}
  \item Did you state the full set of assumptions of all theoretical results?
    \answerYes{See Section~\ref{sec:theory_analysis}.}
        \item Did you include complete proofs of all theoretical results?
    \answerYes{See Appendix~\ref{sec:theory}.}
\end{enumerate}

\item If you ran experiments...
\begin{enumerate}
  \item Did you include the code, data, and instructions needed to reproduce the main experimental results (either in the supplemental material or as a URL)?
    \answerYes{See Appendix~\ref{appendix:hp}, and the code is provided in GitHub.}
  \item Did you specify all the training details (e.g., data splits, hyperparameters, how they were chosen)?
    \answerYes{See Appendix~\ref{appendix:hp}.}
        \item Did you report error bars (e.g., with respect to the random seed after running experiments multiple times)?
    \answerYes{We show error bars by the length of vertical bars in the figures.}
        \item Did you include the total amount of compute and the type of resources used (e.g., type of GPUs, internal cluster, or cloud provider)?
    \answerYes{}
\end{enumerate}

\item If you are using existing assets (e.g., code, data, models) or curating/releasing new assets...
\begin{enumerate}
  \item If your work uses existing assets, did you cite the creators?
    \answerYes{}
  \item Did you mention the license of the assets?
    \answerYes{}
  \item Did you include any new assets either in the supplemental material or as a URL?
    \answerYes{}
  \item Did you discuss whether and how consent was obtained from people whose data you're using/curating?
    \answerYes{}
  \item Did you discuss whether the data you are using/curating contains personally identifiable information or offensive content?
    \answerNA{}
\end{enumerate}

\item If you used crowdsourcing or conducted research with human subjects...
\begin{enumerate}
  \item Did you include the full text of instructions given to participants and screenshots, if applicable?
    \answerNA{}
  \item Did you describe any potential participant risks, with links to Institutional Review Board (IRB) approvals, if applicable?
    \answerNA{}
  \item Did you include the estimated hourly wage paid to participants and the total amount spent on participant compensation?
    \answerNA{}
\end{enumerate}

\end{enumerate}


\newpage

\appendix

\section{Example Illustration of MCTS-VS}
\label{sec:example illustration}

\begin{figure}[htbp]
    \centering
    \includegraphics[width=0.95\textwidth]{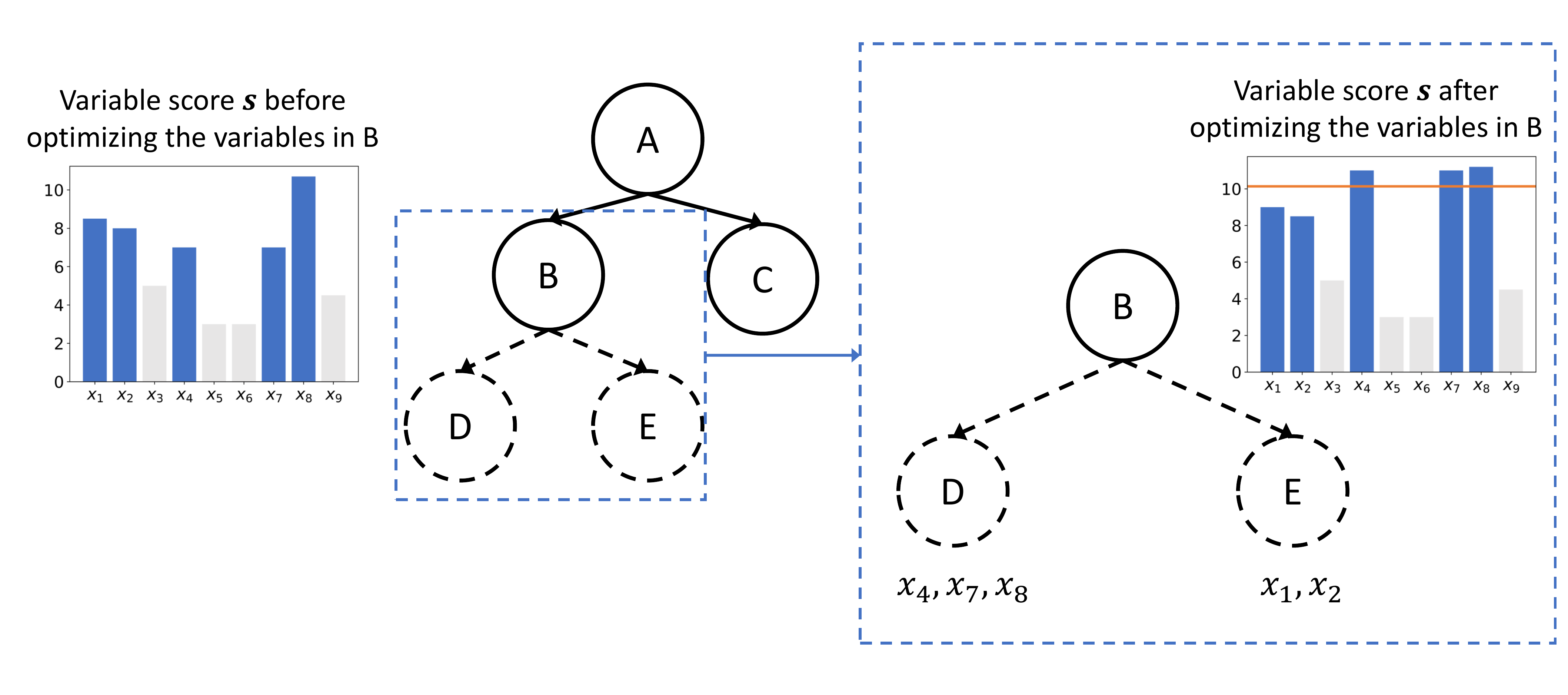}
    \caption{Example illustration of how MCTS-VS bifurcates a leaf node.}
    \label{fig:example_mcts_vs}
\end{figure}

Figure~\ref{fig:example_mcts_vs} gives an example of how MCTS-VS bifurcates a leaf node. Assume that we are to optimize a problem with dimension $D=9$, and the variables are denoted as $x_1,x_2,\ldots,x_9$. The Monte Carlo tree shown in the middle of Figure~\ref{fig:example_mcts_vs} now has three nodes, i.e., $A$, $B$ and $C$, denoted as the solid circles. The root $A$ contains all the nine variables. The current variable score vector $\bm s=[8.5, 8, 5, 7, 3, 3, 7, 10.7, 4.5]$, which is represented by the bar graph as shown in the left of Figure~\ref{fig:example_mcts_vs}. For each $i$, the value of $s_i$ represents the importance of the corresponding variable $x_i$. The blue and gray bars denote the important and unimportant variables, respectively, which are contained by the leaf nodes $B$ and $C$, respectively. That is, the leaf $B$ contains $x_1, x_2, x_4, x_7, x_8$, and $C$ contains the remaining $x_3,x_5,x_6,x_9$. The current $v$ values (i.e., the average scores of the contained variables) of the three nodes $A$, $B$ and $C$ are $v_A = (8.5 + 8 + 5 + 7 + 3 + 3 + 7 + 10.7 + 4.5) / 9 = 6.3$, $v_B = (8.5 + 8 + 7 + 7 + 10.7) / 5 = 8.24$ and $v_C = (5 + 3 + 3 + 4.5) / 4 = 3.875$, respectively. For the number that they have been visited, we have $n_A=1$, $n_B=0$ and $n_C=0$.

MCTS-VS starts from the root node $A$ at one iteration and recursively selects a node with a larger UCB value until a leaf node. According to the way of calculating UCB in Eq.~(\ref{eq-MCTS-UCB}), the UCB values of the leaf nodes $B$ and $C$ are both $\infty$ as $n_B=n_C=0$. In this case, MCTS-VS will select $B$ and $C$ randomly. Assume that $B$ is selected. The variables (i.e., $x_1, x_2, x_4, x_7$ and $x_8$) contained by $B$ will then be optimized by BO with $\mathbb A_B=\{1,2,4,7,8\}$, as in lines~13--23 in Algorithm~\ref{lamcts_vs}. After that, the variable score vector $\bm s$ will be re-calculated, which is assumed to be $[9, 8.5, 5, 11, 3, 3, 11, 11.2, 4.5]$, as shown in the right of Figure~\ref{fig:example_mcts_vs}. The average score of the five variables in $B$ is denoted as the orange horizontal line, calculated by $(9 + 8.5 + 11 + 11 + 11.2) / 5 = 10.14$. We can see that the variables $x_4, x_7$ and $x_8$ have score larger than the average value $10.14$, which are regarded as more important variables in $B$. We use the left child $D$ to represent these variables. The scores of variables $x_1$ and $x_2$ are smaller than the average, which are regarded as less important variables in $B$. We use the right child $E$ to represent them. Thus, the node $B$ has been partitioned into two children $D$ and $E$, denoted as the dashed circles in Figure~\ref{fig:example_mcts_vs}. The $v$ values and the number of visits of these two new leaf nodes are then calculated. The $v$ value of a node is the average score of the contained variables. Thus, $v_D$ is the average score of $x_4, x_7$ and $x_8$, i.e., $(11 + 11 + 11.2) / 3 = 11.067$, and $v_E$ is the average score of $x_1$ and $x_2$, i.e., $(9 + 8.5) / 2 = 8.75$. For the number of visits, obviously $n_D=n_E=0$. Finally, back-propagation is performed to update the $v$ value and the number of visits of the nodes along the path from the root $A$ to the node $B$. $v_A$ is the average score of all the variables, i.e., $(9 + 8.5 + 5 + 11 + 3 + 3 + 11 + 11.2 + 4.5) / 9 = 7.356$. $v_B$ is the average score of $x_1, x_2, x_4, x_7$ and $x_8$, i.e., $(9 + 8.5 + 11 + 11 + 11.2) / 5 = 10.14$. Their number of visits will be increased by 1. That is, $n_A=2$ and $n_B=1$. By far, one iteration of MCTS-VS has been finished, and this process will be performed iteratively.

\section{Details of Theoretical Analysis}
\label{sec:theory}

\subsection{Detailed Proof of Theorem~\ref{the:regret}}\label{sec:theory:theorem}

The proof is inspired by~\citep{gpucb}. To prove the upper bound on the cumulative regret $R_T$ in Theorem~\ref{the:regret}, we analyze the instantaneous regret $r_t = f(\bm{x}^*) - f(\bm{x}_{\mathbb{M}_t}^t)$, i.e., the gap between the function values of the optimal point $\bm{x}^*$ and the sampled point $\bm{x}_{\mathbb{M}_t}^t$ at iteration $t$. Note that $R_T=\sum^T_{t=1} r_t$. Let $\mu_{t-1}(\cdot)$ and $\sigma_{t-1}^2(\cdot)$ denote the posterior mean and variance after running $t-1$ iterations, respectively. Lemma~\ref{lem:1} gives a confidence bound \rbt{on} $f(\bm x_{\mathbb M_t}^t)$, leading to a lower bound on $f(\bm{x}_{\mathbb{M}_t}^t)$, i.e., $f(\bm{x}_{\mathbb{M}_t}^t) \ge \mu_{t-1}(\bm{x}_{\mathbb{M}_t}^t)-\beta_{t}^{1/2}\sigma_{t-1}(\bm{x}_{\mathbb{M}_t}^t)$. Note that $\mathbb M_t$ denotes the sampled variable index subset at iteration $t$, and $|\mathbb M_t| = d_t$.

\begin{lemma}
$\forall \delta \in (0,1), \forall t\geq1$, let $\beta_t = 2\log(\pi_{t}/\delta)$, where $\sum_{t\geq 1} \pi^{-1}_t=1, \pi_t >0$. Then, $\forall t\geq1$,
\[ |f(\bm x_{\mathbb M_t}^t)-\mu_{t-1}(\bm x_{\mathbb M_t}^t)| \leq \beta_{t}^{1/2}\sigma_{t-1}(\bm x_{\mathbb M_t}^t) \]holds with probability at least $1-\delta$, where $\bm x_{\mathbb M_t}^t$ is the point obtained at iteration $t$. 
\label{lem:1}
\end{lemma}
\begin{proof}
At iteration $t$, $f(\bm x_{\mathbb M_t}^t) \sim \mathcal{N}(\mu_{t-1}(\bm x_{\mathbb M_t}^t),\sigma_{t-1}^2(\bm x_{\mathbb M_t}^t))$, and thus, $Y=\frac{f(\bm x_{\mathbb M_t}^t)-\mu_{t-1}(\bm x_{\mathbb M_t}^t)}{\sigma_{t-1}(\bm x_{\mathbb M_t}^t)} \sim \mathcal{N}(0,1)$. We have
\begin{flalign*}
& P\left(|f(\bm x_{\mathbb M_t}^t)-\mu_{t-1}(\bm x_{\mathbb M_t}^t)| > \beta_{t}^{1/2}\sigma_{t-1}(\bm x_{\mathbb M_t}^t)\right) \\
& = P\left(|Y| > \beta_{t}^{1/2}\right)
= 2\int_{\beta_{t}^{1/2}}^{\infty}(2\pi)^{(-1/2)}\exp\left(-\frac{y^2}{2}\right)dy \\
& = 2\exp\left(-\frac{\beta_{t}}{2}\right)\int_{\beta_{t}^{1/2}}^{\infty}(2\pi)^{(-1/2)} \exp\left(-\frac{(y-\beta_{t}^{1/2})^2}{2}\right) \exp\left(-\beta_{t}^{1/2}(y-\beta_{t}^{1/2})\right)dy \\
& \leq 2\exp\left(-\frac{\beta_{t}}{2}\right)P(Y>0) 
\leq \exp\left(-\frac{\beta_{t}}{2}\right) = \frac{\delta}{\pi_t}.
\end{flalign*}
Using the union bound for all $t \in \mathbb{N}$, we have 
\begin{flalign*}
& P\left(\forall t \geq 1: |f(\bm x_{\mathbb M_t}^t)-\mu_{t-1}(\bm x_{\mathbb M_t}^t)| \leq  \beta_{t}^{1/2}\sigma_{t-1}(\bm x_{\mathbb M_t}^t)\right) \geq 1- \sum\limits_{t\geq 1}\frac{\delta}{\pi_t}=1-\delta,
\end{flalign*}
where the equality holds by $\sum_{t\geq 1} \pi^{-1}_t=1$. Thus, the lemma holds.
\end{proof}

Next we are to analyze the upper bound on $f(\bm{x}^*)$, which can be represented as $(f(\bm{x}^*) - f(\bm{x}_{\mathbb M_t}^*)) + f(\bm x_{\mathbb M_t}^*)$, where $\bm{x}^*_{\mathbb{M}_t}$ denotes the point obtained by projecting $\bm{x}^*$ onto $\mathbb{M}_t$. The first term $f(\bm{x}^*) - f(\bm{x}_{\mathbb M_t}^*)$ can be upper bounded by Assumption~\ref{ass:1}. To upper bound the second term $f(\bm x_{\mathbb M_t}^*)$, we need to discretize the decision space $\mathcal X_{\mathbb M_t}$ at iteration $t$ into $\tilde{\mathcal X}_{\mathbb M_t}$, where $|\tilde{\mathcal X}_{\mathbb M_t}|=(\tau_t)^{d_t}$, i.e., we divide each variable of $\mathcal X_{\mathbb M_t}$ into $\tau_t$ parts equally. Let $\tilde{\bm{x}}_{\mathbb{M}_t}^*$ denote the point closest to $\bm{x}^*_{\mathbb{M}_t}$ in the discretized space $\tilde{\mathcal X}_{\mathbb M_t}$. Then, we can write $f(\bm x_{\mathbb M_t}^*)$ as $(f(\bm x_{\mathbb M_t}^*) -f(\tilde{\bm{x}}_{\mathbb M_t}^*))+f(\tilde{\bm{x}}_{\mathbb M_t}^*)$. The first term $f(\bm x_{\mathbb M_t}^*) -f(\tilde{\bm{x}}_{\mathbb M_t}^*)$ again can be upper bounded by Assumption~\ref{ass:1}. Lemma~\ref{lem:2} gives a confidence bound \rbt{on} $f(\tilde{\bm x}_{\mathbb M_t})$ for any discretized point $\tilde{\bm x}_{\mathbb M_t} \in \tilde{\mathcal X}_{\mathbb M_t}$, leading to an upper bound on $f(\tilde{\bm{x}}_{\mathbb M_t}^*)$, i.e., $f(\tilde{\bm{x}}_{\mathbb M_t}^*) \leq \mu_{t-1}(\tilde{\bm{x}}_{\mathbb M_t}^*)+\beta_{t}^{1/2}\sigma_{t-1}(\tilde{\bm{x}}_{\mathbb M_t}^*)$.

\begin{lemma}
$\forall \delta \in (0,1), \forall t \geq 1$, let $\beta_t = 2\log(|\tilde{\mathcal X}_{\mathbb M_t}|\pi_{t}/\delta)$, where $\sum_{t\geq 1} \pi^{-1}_t=1, \pi_t >0$. Then, $ \forall t \geq 1, \forall \tilde{\bm{x}}_{\mathbb M_t}\in \tilde{\mathcal X}_{\mathbb M_t}$,
\[ |f(\tilde{\bm{x}}_{\mathbb M_t})-\mu_{t-1}(\tilde{\bm{x}}_{\mathbb M_t})| \le \beta_{t}^{1/2}\sigma_{t-1}(\tilde{\bm{x}}_{\mathbb M_t})\]holds with probability at least $1-\delta$.
\label{lem:2}
\end{lemma}
\begin{proof}
Similar to Lemma~\ref{lem:1}, we can derive
\begin{flalign*}
P\left( |f(\tilde{\bm{x}}_{\mathbb M_t})-\mu_{t-1}(\tilde{\bm{x}}_{\mathbb M_t})| > \beta_{t}^{1/2}\sigma_{t-1}(\tilde{\bm{x}}_{\mathbb M_t}) \right) \leq \exp\left(-\frac{\beta_{t}}{2}\right)= \frac{\delta}{|\tilde{\mathcal X}_{\mathbb M_t}|\pi_{t}}.
\end{flalign*}
Using the union bound for all $t \in \mathbb{N}$ and $\tilde{\bm{x}}_{\mathbb M_t}\in \tilde{\mathcal X}_{\mathbb M_t}$, we have
\begin{flalign*}
&P\left(\forall t \geq 1,\forall \tilde{\bm{x}}_{\mathbb M_t}\in \tilde{\mathcal X}_{\mathbb M_t}: |f(\tilde{\bm{x}}_{\mathbb M_t})-\mu_{t-1}(\tilde{\bm{x}}_{\mathbb M_t})| \leq \beta_{t}^{1/2}\sigma_{t-1}(\tilde{\bm{x}}_{\mathbb M_t})\right) \\
&\quad \geq 1- \sum\limits_{t\geq 1}\sum\limits_{\tilde{\bm{x}}_{\mathbb M_t}\in \tilde{\mathcal X}_{\mathbb M_t}}\frac{\delta}{|\tilde{\mathcal X}_{\mathbb M_t}|\pi_{t}}=1-\delta.
\end{flalign*}
Thus, the lemma holds.
\end{proof}

Now, we can upper bound $f(\bm x_{\mathbb M_t}^*)$ based on Assumption~\ref{ass:1} and Lemma~\ref{lem:2}, as shown in Lemma~\ref{lem:3}. Note that $\bm x_{\mathbb M_t}^*$ denotes the point obtained by projecting $\bm{x}^*$ onto $\mathbb{M}_t$, and $\tilde{\bm x}_{\mathbb M_t}^*$ denotes the point closest to $\bm x_{\mathbb M_t}^*$ in $\tilde{\mathcal X}_{\mathbb M_t}$. 

\begin{lemma}
$\forall \delta \in (0,1), t \geq 1$, let $\beta_t = 2\log(2\pi_{t}/\delta)+2d_t\log\left(d_tt^2br\sqrt{\log\frac{2Da}{\delta}}\right)$, where $\sum_{t\geq 1} \pi^{-1}_t=1, \pi_t >0$. Set $\tau_{t}=d_tt^{2}br\sqrt{\log\frac{2Da}{\delta}}$ and $L=b\sqrt{\log\frac{2Da}{\delta}}$. Then, $\forall t \geq 1$,
\begin{equation*}
\begin{split}
   |f(\bm x_{\mathbb M_t}^*)-\mu_{t-1}(\tilde{\bm x}_{\mathbb M_t}^*)|
 \le \beta_{t}^{1/2}\sigma_{t-1}(\tilde{\bm x}_{\mathbb M_t}^*)+ \frac{\alpha_{\max}}{t^2} + \sum_{i\in [D] \setminus \mathbb M_t} \alpha_{i}^* Lr
\end{split}
\end{equation*}
holds with probability at least $1-\delta$. 
\label{lem:3}
\end{lemma}
\begin{proof}
First, we have
\begin{align}
|f(\bm x_{\mathbb M_t}^*)-\mu_{t-1}(\tilde{\bm x}_{\mathbb M_t}^*)|&=|f(\bm x_{\mathbb M_t}^*)-f(\tilde{\bm x}_{\mathbb M_t}^*)+f(\tilde{\bm x}_{\mathbb M_t}^*)-\mu_{t-1}(\tilde{\bm x}_{\mathbb M_t}^*)|\nonumber\\
& \leq |f(\bm x_{\mathbb M_t}^*)-f(\tilde{\bm x}_{\mathbb M_t}^*)|+|f(\tilde{\bm x}_{\mathbb M_t}^*)-\mu_{t-1}(\tilde{\bm x}_{\mathbb M_t}^*)|. \label{eq-mid-3}
\end{align}
By Assumption~\ref{ass:1} with $L=b\sqrt{\log\frac{2Da}{\delta}}$, we have $\forall \bm x, \bm y\in \mathcal X$, with probability at least $1 - D\cdot ae^{-(L/b)^2}=1-\delta/2$, 
\begin{align}
|f(\bm x)-f(\bm y)|
& \leq \sum_{i=1}^{D} \alpha_{i}^*L| x_i- y_i| \nonumber\\
& \leq \sum_{i\in \mathbb M_t}\alpha_{i}^*L|x_i-y_i|+\sum_{i\in [D] \setminus \mathbb M_t}\alpha_{i}^*Lr    \nonumber\\
& \leq \alpha_{\max} L\Vert\bm x_{\mathbb M_t}- \bm y_{\mathbb M_t}\Vert_1 + \sum_{i\in [D]\setminus\mathbb M_t}\alpha_{i}^*Lr,\label{eq-mid-4}
\end{align}
where the second inequality holds by $\mathcal X \subset [0, r]^D$, and the last inequality holds by $\alpha_{\max} = \max\nolimits_{i\in [D]}\alpha_{i}^*$. Thus, it holds with probability at least $1-\delta/2$ that
\begin{align}\label{eq-mid-1}
|f(\bm x_{\mathbb M_t}^*)-f(\tilde{\bm x}_{\mathbb M_t}^*)|\leq \alpha_{\max} L\Vert\bm x_{\mathbb M_t}^*- \tilde{\bm x}_{\mathbb M_t}^*\Vert_1 + \sum_{i\in [D]\setminus\mathbb M_t}\alpha_{i}^*Lr.   
\end{align}
By Lemma~\ref{lem:2} with $\beta_{t} = 2\log(2(\tau_{t})^{d_t}\pi_{t}/\delta) = 2\log(2|\tilde{\mathcal X}_{\mathbb M_t}|\pi_{t}/\delta) $, we have, with probability at least $1-\delta/2$,
\begin{align}\label{eq-mid-2}
|f(\tilde{\bm x}_{\mathbb M_t}^*)-\mu_{t-1}(\tilde{\bm x}_{\mathbb M_t}^*)|\leq  \beta_{t}^{1/2}\sigma_{t-1}(\tilde{\bm x}_{\mathbb M_t}^*). 
\end{align}

Applying Eqs.~(\refeq{eq-mid-1}) and~(\refeq{eq-mid-2}) to Eq.~(\refeq{eq-mid-3}), it holds with probability at least $1-\delta$ that
\begin{flalign*}
|f(\bm x_{\mathbb M_t}^*)-\mu_{t-1}(\tilde{\bm x}_{\mathbb M_t}^*)| 
& \le  \alpha_{\max} L\Vert\bm x_{\mathbb M_t}^*- \tilde{\bm x}_{\mathbb M_t}^*\Vert_1 + \sum_{i\in [D]\setminus\mathbb M_t}\alpha_{i}^*Lr+\beta_{t}^{1/2}\sigma_{t-1}(\tilde{\bm x}_{\mathbb M_t}^*)\\
& \le \alpha_{\max}L\frac{d_tr}{\tau_t}+\sum\limits_{i\in [D]\setminus \mathbb M_t} \alpha_{i}^*Lr +  \beta_{t}^{1/2}\sigma_{t-1}(\tilde{\bm x}_{\mathbb M_t}^*)\\
& \le \frac{\alpha_{\max}}{t^2}+ \sum\limits_{i\in [D]\setminus \mathbb M_t} \alpha_{i}^* Lr + \beta_{t}^{1/2}\sigma_{t-1}(\tilde{\bm x}_{\mathbb M_t}^*),
\end{flalign*}
where the second inequality holds by $|\mathbb M_t|=d_t$ and the way of discretization (i.e., each variable is discretized into $\tau_t$ parts equally), and the last inequality holds by the definition of $\tau_t$ and $L$. Thus, the lemma holds.
\end{proof}

Lemma~\ref{lem:3} implies an upper bound on $f(\bm x_{\mathbb M_t}^*)$, i.e., $f(\bm x_{\mathbb M_t}^*)\leq \mu_{t-1}(\tilde{\bm x}_{\mathbb M_t}^*)+\beta_{t}^{1/2}\sigma_{t-1}(\tilde{\bm x}_{\mathbb M_t}^*)+ \alpha_{\max}/t^2 + \sum_{i\in [D] \setminus \mathbb M_t} \alpha_{i}^* Lr$. Combining this upper bound on $f(\bm x_{\mathbb M_t}^*)$ with $f(\bm{x}^*) - f(\bm{x}_{\mathbb M_t}^*)$ (which can be upper bounded by Assumption~\ref{ass:1}), we can derive an upper bound on $f(\bm{x}^*)$. Together with the lower bound on $f(\bm x_{\mathbb M_t}^t)$ given by Lemma~\ref{lem:1}, we can derive an upper bound on the instantaneous regret $r_t$. Thus, we are now ready to prove the upper bound on the cumulative regret $R_T$ in Theorem~\ref{the:regret}, which is re-stated in Theorem~\ref{appen-theo} for clearness.

\begin{theorem}\label{appen-theo}
$\forall \delta \in (0,1)$, let $\beta_t = 2\log(4\pi_{t}/\delta)+2d_t\log(d_tt^2br\sqrt{\log(4Da/\delta)})$ and $L=b\sqrt{\log(4Da/\delta)}$, where $\{\pi_t\}_{t\ge 1}$ satisfies $\sum_{t\geq 1} \pi^{-1}_t=1$ and $\pi_t >0$. Let $\beta_{T}^* = \max\nolimits_{1\le i\le T} \beta_{t}$.  At iteration $T$, the cumulative regret
\begin{align*}
R_T \le \sqrt{C_1T\beta_{T}^*\gamma_T}+2\alpha_{\max}+ 2\sum_{t=1}^{T}\sum_{i\in [D]\setminus \mathbb M_t}\alpha^*_{i}Lr
\end{align*}
holds with probability at least $1-\delta$, where $C_1>0$ is a constant, $\gamma_{T} = \max\nolimits_{|\mathcal D| = T} I(\bm{y}_{\mathcal D},\bm{f}_{\mathcal D})$, $I(\cdot,\cdot)$ denotes the information gain, and $\bm{y}_{\mathcal D}$ and $\bm{f}_{\mathcal D}$ are the noisy and true observations of a set $\mathcal{D}$ of points, respectively.
\end{theorem}
\begin{proof}
For all $t \ge 1$, we have
\begin{align}
r_t = f(\bm x^*)-f(\bm x_{\mathbb M_t}^t)  = f(\bm x^*) - f(\bm x_{\mathbb M_t}^*) + f(\bm x_{\mathbb M_t}^*) - f(\bm x_{\mathbb M_t}^t).\label{eq-mid-6} 
\end{align}
By Eq.~(\refeq{eq-mid-4}), we have 
\begin{align}
f(\bm x^*) - f(\bm x_{\mathbb M_t}^*)\leq  \alpha_{\max} L\Vert\bm x_{\mathbb M_t}^*- \bm x_{\mathbb M_t}^*\Vert_1 + \sum_{i\in [D]\setminus\mathbb M_t}\alpha_{i}^*Lr=\sum_{i\in [D]\setminus\mathbb M_t}\alpha_{i}^*Lr.\label{eq-mid-5}
\end{align}
Note that $L=b\sqrt{\log(4Da/\delta)}$ here, and thus Eq.~(\refeq{eq-mid-5}) holds with probability at least $1-\delta/4$. By Lemma~\ref{lem:3} with $\beta_t = 2\log(4\pi_{t}/\delta)+2d_t\log(d_tt^2br\sqrt{\log(4Da/\delta)})$ and $L=b\sqrt{\log(4Da/\delta)}$, setting $\tau_{t}=d_tt^{2}br\sqrt{\log(4Da/\delta)}$ leads to that
\begin{align}
f(\bm x_{\mathbb M_t}^*)\leq \mu_{t-1}(\tilde{\bm x}_{\mathbb M_t}^*)+\beta_{t}^{1/2}\sigma_{t-1}(\tilde{\bm x}_{\mathbb M_t}^*)+ \frac{\alpha_{\max}}{t^2} + \sum_{i\in [D] \setminus \mathbb M_t} \alpha_{i}^* Lr\label{eq-mid-7}
\end{align}
holds with probability at least $1-\delta/2$. By Lemma~\ref{lem:1} with $\beta_t = 2\log(4\pi_{t}/\delta)+2d_t\log(d_tt^2br\sqrt{\log(4Da/\delta)})\geq 2\log(4\pi_{t}/\delta)$, it holds with probability at least $1-\delta/4$ that \begin{align}f(\bm{x}_{\mathbb{M}_t}^t) \ge \mu_{t-1}(\bm{x}_{\mathbb{M}_t}^t)-\beta_{t}^{1/2}\sigma_{t-1}(\bm{x}_{\mathbb{M}_t}^t).\label{eq-mid-8}\end{align}

Applying Eqs.~(\refeq{eq-mid-5}),~(\refeq{eq-mid-7}) and~(\refeq{eq-mid-8}) to Eq.~(\refeq{eq-mid-6}), it holds with probability at least $1-\delta$ that $\forall t \ge 1$,
\begin{flalign*}
r_t &\le  \sum_{i\in [D]\setminus \mathbb M_t}\alpha_i^*Lr + \mu_{t-1}(\tilde{\bm x}_{\mathbb M_t}^*) + \beta_{t}^{1/2}\sigma_{t-1}(\tilde{\bm x}_{\mathbb M_t}^*)+ \frac{\alpha_{\max}}{t^2} + \sum_{i\in [D]\setminus \mathbb M_t} \alpha_{i}^* Lr  \\
& \quad - \mu_{t-1}(\bm{x}_{\mathbb{M}_t}^t)+\beta_{t}^{1/2}\sigma_{t-1}(\bm{x}_{\mathbb{M}_t}^t) \\ 
& \le \mu_{t-1}(\bm x_{\mathbb M_t}^t) + \beta_{t}^{1/2}\sigma_{t-1}(\bm x_{\mathbb M_t}^t) + \frac{\alpha_{\max}}{t^2} + 2\sum_{i\in [D]\setminus \mathbb M_t} \alpha_{i}^* Lr  - \mu_{t-1}(\bm{x}_{\mathbb{M}_t}^t)+\beta_{t}^{1/2}\sigma_{t-1}(\bm{x}_{\mathbb{M}_t}^t) \\
&=  2\beta_{t}^{1/2}\sigma_{t-1}(\bm x_{\mathbb M_t}^t) +\frac{\alpha_{\max}}{t^2} + 2\sum_{i\in [D]\setminus \mathbb M_t} \alpha_{i}^* Lr,
\end{flalign*}
where the second inequality holds because $\bm x_{\mathbb M_t}^t$ is generated by maximizing GP-UCB, and thus $\mu_{t-1}(\tilde{\bm x}_{\mathbb M_t}^*) + \beta_{t}^{1/2}\sigma_{t-1}(\tilde{\bm x}_{\mathbb M_t}^*)\le \mu_{t-1}(\bm x_{\mathbb M_t}^t) + \beta_{t}^{1/2}\sigma_{t-1}(\bm x_{\mathbb M_t}^t)$.

By summing up $r_t$ from $t=1$ to $T$, we have with probability at least $1-\delta$ that, $\forall T \ge 1$, 
\begin{align}
R_T = \sum_{t=1}^{T} r_t  &\le \sum_{t=1}^{T}2\beta_{t}^{1/2}\sigma_{t-1}(\bm x_{\mathbb M_t}^t) + \sum_{t=1}^{T}\frac{\alpha_{\max}}{t^2} + \sum_{t=1}^{T}2\sum_{i\in [D]\setminus \mathbb M_t} \alpha_{i}^* Lr \nonumber\\
&\leq \sum_{t=1}^{T}2\beta_{t}^{1/2}\sigma_{t-1}(\bm x_{\mathbb M_t}^t) + 2\alpha_{\max} + 2\sum_{t=1}^{T}\sum_{i\in [D]\setminus \mathbb M_t} \alpha_{i}^* Lr,\label{eq-mid-9}
\end{align}
where the second inequality holds by $\sum_{t=1}^{T}1/t^2\leq \pi^2/6\leq 2$. Furthermore, let $C_1=8/\log(1+\eta^{-2})$, and Lemma~5.4 in~\cite{gpucb} has shown that $\sum_{t=1}^{T}2\beta_{t}^{1/2}\sigma_{t-1}(\bm x_{\mathbb M_t}^t) \le \sqrt{C_{1}T\beta_{T}^*\sum_{t=1}^{T}\log(1+\eta^{-2}\sigma_{t-1}^2(\bm x_{\mathbb M_t}^t))/2} \leq \sqrt{C_{1}T\beta_{T}^*\gamma_{T}}$. Finally, by applying this inequality to Eq.~(\refeq{eq-mid-9}), the theorem holds.
\end{proof}

We also summarize the main idea of the above proof. The proof is inspired by~\citep{gpucb}, i.e., to derive the upper bound on the gap $r_t = f(\bm{x}^*) - f(\bm{x}_{\mathbb{M}_t}^t)$ between the function values of the optimal point $\bm{x}^*$ and the sampled point $\bm{x}_{\mathbb{M}_t}^t$ at iteration $t$. Let $\bm{x}^*_{\mathbb{M}_t}$ denote the point obtained by projecting $\bm{x}^*$ onto $\mathbb{M}_t$, and $\tilde{\bm x}_{\mathbb M_t}^*$ denote its closest discretized point. By utilizing the posterior mean $\mu_{t-1}(\cdot)$ and variance $\sigma_{t-1}^2(\cdot)$ of $f(\bm{x}_{\mathbb{M}_t}^t)$ and $f(\tilde{\bm x}_{\mathbb M_t}^*)$, we can have $f(\bm{x}_{\mathbb{M}_t}^t) \ge \mu_{t-1}(\bm{x}_{\mathbb{M}_t}^t)-\beta_{t}^{1/2}\sigma_{t-1}(\bm{x}_{\mathbb{M}_t}^t)$ and $f(\bm{x}^*) = (f(\bm{x}^*) - f(\bm{x}_{\mathbb M_t}^*)) + (f(\bm x_{\mathbb M_t}^*) -f(\tilde{\bm x}_{\mathbb M_t}^*))+f(\tilde{\bm x}_{\mathbb M_t}^*) \le \sum_{i\in [D]\setminus \mathbb M_t}\alpha_i^*Lr+\alpha_{\max}/t^2+\sum_{i\in [D]\setminus \mathbb M_t}\alpha^*_{i}Lr + \mu_{t-1}(\tilde{\bm x}_{\mathbb M_t}^*)+\beta_{t}^{1/2}\sigma_{t-1}(\tilde{\bm x}_{\mathbb M_t}^*)$, where the terms $\sum_{i\in [D]\setminus \mathbb M_t}\alpha_i^*Lr$ and $\alpha_{\max}/t^2$ are led by variable selection and discretization, respectively. As $\bm x_{\mathbb M_t}^t$ is generated by maximizing GP-UCB, we have $\mu_{t-1}(\tilde{\bm x}_{\mathbb M_t}^*) + \beta_{t}^{1/2}\sigma_{t-1}(\tilde{\bm x}_{\mathbb M_t}^*)\le \mu_{t-1}(\bm x_{\mathbb M_t}^t) + \beta_{t}^{1/2}\sigma_{t-1}(\bm x_{\mathbb M_t}^t)$. Thus, $r_t \le 2\beta_{t}^{1/2}\sigma_{t-1}(\bm{x}_{\mathbb{M}_t}^t)+\alpha_{\max}/t^2+2\sum_{i\in [D]\setminus \mathbb M_t}\alpha^*_{i}Lr$. Finally, summing up $r_t$ from $t=1$ to $T$ and using Lemma~5.4 in~\cite{gpucb} can lead to Theorem~\ref{the:regret}.

The main difference from the proof of GP-UCB~\citep{gpucb} is that variable selection brings some uncertainty introduced by the unselected variables. Based on the Lipschitz condition in Assumption~\ref{ass:1}, the uncertainty by the $i$-th unselected variable can be upper bounded by $\alpha_i^*Lr$, leading to the additional regret $2\sum_{t=1}^{T}\sum_{i\in [D]\setminus \mathbb M_t}\alpha^*_{i}Lr$ in Eq.~(\refeq{eq-main-theo}).

\subsection{Details of Computational Complexity Analysis}
\label{sec:complexity}

The computational complexity of one iteration of BO depends on three critical components: fitting a GP surrogate model, maximizing an acquisition function and evaluating a sampled point. Assume that the kernel function is squared exponential kernel. At iteration $t$, the number of selected variables is $d_t$. When fitting a GP model, we calculate the marginal likelihood~\cite{gp} and gradient as follows:
\begin{align*}
\log P(\bm y_t\mid \mathbf{X}_t, \bm \theta) = & 
-\frac{1}{2}\bm y_t^{\mathsf T} (\mathbf{K}_{t}+\eta^2 \mathbf{I})^{-1}\bm y_t -\frac{1}{2}\log |\mathbf{K}_t+\eta^2 \mathbf{I} |-\frac{t}{2}\log (2\pi) \\
\nabla_{\bm \theta}\log P(\bm y_t\mid \mathbf{X}_t, \bm \theta) 
= & -\frac{1}{2}\bm y^{\mathsf T}_t (\mathbf{K}_{t}+\eta^2 \mathbf{I})^{-1}\nabla_{\bm \theta}(\mathbf{K}_{t}+\eta^2 \mathbf{I})(\mathbf{K}_{t}+\eta^2 \mathbf{I})^{-1}\bm y_t \\
& -\frac{1}{2}\mathrm{tr} \left( (\mathbf{K}_{t}+\eta^2 \mathbf{I})^{-1} \nabla_{\bm \theta}(\mathbf{K}_{t}+\eta^2 \mathbf{I}) \right)
\end{align*}
where $\bm y_{t}=[y^1, \ldots, y^{t}]^{\mathsf T}$, $\mathbf{X}_t = [\bm x^1, \ldots, \bm x^{t}]$, $\bm \theta$ are the kernel parameters, $\mathbf{K}_t$ is the covariance matrix, \rbt{$|\cdot|$ and $\mathrm{tr}(\cdot)$ denote the determinant and trace of a matrix, respectively}. Then, we can use the gradient-based methods to optimize the likelihood function. Therefore, the computational complexity of calculating the kernel parameters is $\mathcal{O}(t^3+t^2 d_t)$. Note that $\bm \theta$ has been ignored, because its dimension is much smaller than $d_t$ and $t$. When calculating the mean $\mu_t (\bm x)$ and variance $\sigma^2_t(\bm x)$, the computational complexity is $\mathcal{O}(t^3+t^2 d_t)$, due to the calculation of the kernel matrix and its inverse. Thus, the total computational complexity of fitting the  GP model is $\mathcal{O}(t^3+t^2 d_t)$. Maximizing an acquisition function is related to the optimization algorithm. If we use the Quasi-Newton method to optimize GP-UCB, the computational complexity is $\mathcal O(m(t^2 + td_t + d^2_t))$~\citep{NoceWrig06}, where $m$ denotes the Quasi-Newton's running rounds. We note that in BO setting, $t$ will not grow very large. The running rounds $m$, however, will grow with $d_t$. Thus, the complexity of optimizing the acquisition function can be much larger than the square of $d_t$. The cost of evaluating a sampled point is fixed. Thus, by selecting only a subset of variables, instead of all variables, to optimize, the computational complexity can decrease significantly.


\section{Method Implementation and Experimental Setting}
\label{appendix:hp}

\setcounter{footnote}{0}

We use the authors' reference implementations for TuRBO\footnote{\url{https://github.com/uber-research/TuRBO}}, LA-MCTS\footnote{\url{https://github.com/facebookresearch/LaMCTS}} and SAASBO.\footnote{\url{https://github.com/martinjankowiak/saasbo}} For HeSBO and ALEBO, their implementations in Adaptive Experimentation Platform (Ax\footnote{\url{https://github.com/facebook/Ax}}) are used. We use the pycma library for CMA-ES.\footnote{\url{https://github.com/CMA-ES/pycma}} Their hyper-parameters are summarized as follows.

\begin{itemize}
    \item \textbf{Vanilla BO.} We use the GP model in Scikit-learn\footnote{\url{https://github.com/scikit-learn/scikit-learn}} and the qExpectedImprovement acquisition function~\cite{qei}. \songl{For the optimization of acquisition function, we randomly generate numerous points and select some ones with the maximal expected improvements, which is similar to the implementation in TuRBO~\cite{turbo}, LA-MCTS~\cite{lamcts}, and HeSBO~\cite{hesbo}.}
    \item \textbf{MCTS-VS.} For the ``fill-in'' strategy, we use the best-$k$ strategy with $k=20$. The hyper-parameter $C_p$ for calculating UCB in Eq.~(\refeq{eq-MCTS-UCB}) varies on different problems, as shown in Table~\ref{tab:cp}. We set all the other parameters to be \emph{same} on different problems, where the batch size $N_v$ of variable index set is $2$, the sample batch size $N_s=3$, the threshold $N_{bad}$ for re-initializing a tree is $5$, and the threshold $N_{split}$ for splitting a node is $3$. When using TuRBO as the optimizer, we limit the maximal number of evaluations in TuRBO to $50$. 
    
\begin{table*}[htbp]
\caption{Setting of the hyper-parameter $C_p$ for calculating UCB on different problems.}
\vskip 0.15in
\begin{center}
\begin{small}
\begin{sc}
\begin{tabular}{lcccc}
\toprule
  & Levy & Hartmann & NAS-Bench & MuJoCo \\
$C_p$ & 10 & 0.1 & 0.1 & 50  \\
\bottomrule
\end{tabular}
\end{sc}
\end{small}
\end{center}
\vskip -0.1in
\label{tab:cp}
\end{table*}

    \item \textbf{Dropout.} We set the parameter $d$ to the number of valid dimensions for synthetic functions, and use the same ``fill-in'' strategy as MCTS-VS.
    \item \textbf{TuRBO.} We use the default parameter setting in the authors' reference implementation.
    \item \textbf{LA-MCTS-TuRBO.} We use the same TuRBO setting as MCTS-VS. The parameter $C_p$ is recommended between $1\%$ and $10\%$ of the optimum in LA-MCTS~\cite{lamcts}. Because all our selected values of $C_p$ for MCTS-VS have belonged to the recommended range for LA-MCTS, we use them directly. The RBF kernel is used for SVM classification.
    \item \songl{\textbf{SAASBO.} We use the default parameter setting in the authors' reference implementation, but modify the acquisition function optimization to the same as other methods for fair comparison. }
    \item \textbf{HeSBO and ALEBO.} We set the parameter $d$ to the number of valid dimensions for synthetic functions. For real-world problems, we do not know the number of valid dimensions, and thus we just set a reasonable value, i.e., $d=10$ for NAS-Bench, $d=10$ for Hopper, and $d=20$ for Walker2d. 
    \item \textbf{CMA-ES.} We only adjust the step-size parameter $\sigma$ for different problems, because the default setting $\sigma=0.01$ leads to extremely poor performance. We set $\sigma=0.8$ for Hartmann problems, $\sigma=10$ for Levy problems, $\sigma=0.1$ for NAS-Bench, and $\sigma=0.01$ for MuJoCo tasks. We set the population size to $20$ and maintain all the other parameters to default.
    \item \textbf{VAE-BO} uses VAE for embedding. That is, VAE-BO uses the encoder to embed the original high-dimensional space into a low-dimensional subspace, then optimizes via BO in the subspace and uses the decoder to project the new sampled point back for evaluation. We set the learning rate to $0.01$ and the interval of updating VAE to $30$. 
\end{itemize}

The experiments of comparing wall clock time are conducted on Intel(R) Core(TM) i7-10700 CPU @ 2.90GHz and use single thread.

\section{Sensitivity Analysis of Hyper-parameters of MCTS-VS} 
\label{appendix:ablation}

We provide further studies to examine the influence of the hyper-parameters of MCTS-VS, including the employed optimization algorithm for optimizing the selected variables in each iteration, the ``fill-in'' strategy, the hyper-parameter $k$ used in the best-$k$ strategy, the hyper-parameter $C_p$ for calculating UCB in Eq.~(\refeq{eq-MCTS-UCB}), the number $2\times N_v\times N_s$ of sampled data in each iteration, the threshold $N_{bad}$ for re-initializing a tree, and the threshold $N_{split}$ for splitting a tree node.

\textbf{The optimization algorithm} is employed by MCTS-VS to optimize the selected variables in each iteration. We compare three different optimization algorithms, i.e., random search (RS), BO and TuRBO. First, we conduct experiments similar to ``Effectiveness of Variable Selection'' in Section~\ref{sec:synthetic}, to show the effectiveness of MCTS-VS even when equipped with RS. Figure~\ref{fig:ablation:exp1} shows that MCTS-VS-RS is better than Dropout-RS and RS, revealing the advantage of MCTS-VS.


\begin{figure}[htbp]
    \centering
    \subfigure{\includegraphics[width=0.5\textwidth]{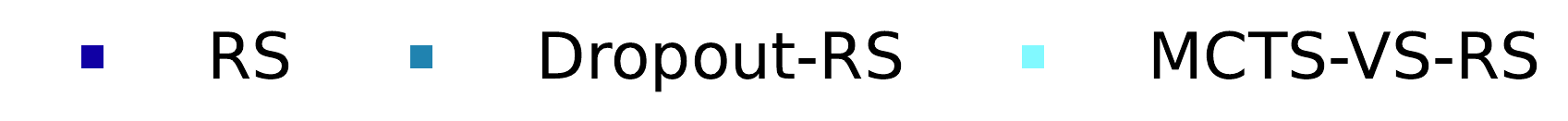}} \\
    \centering
    \subfigure{\includegraphics[width=0.4\textwidth]{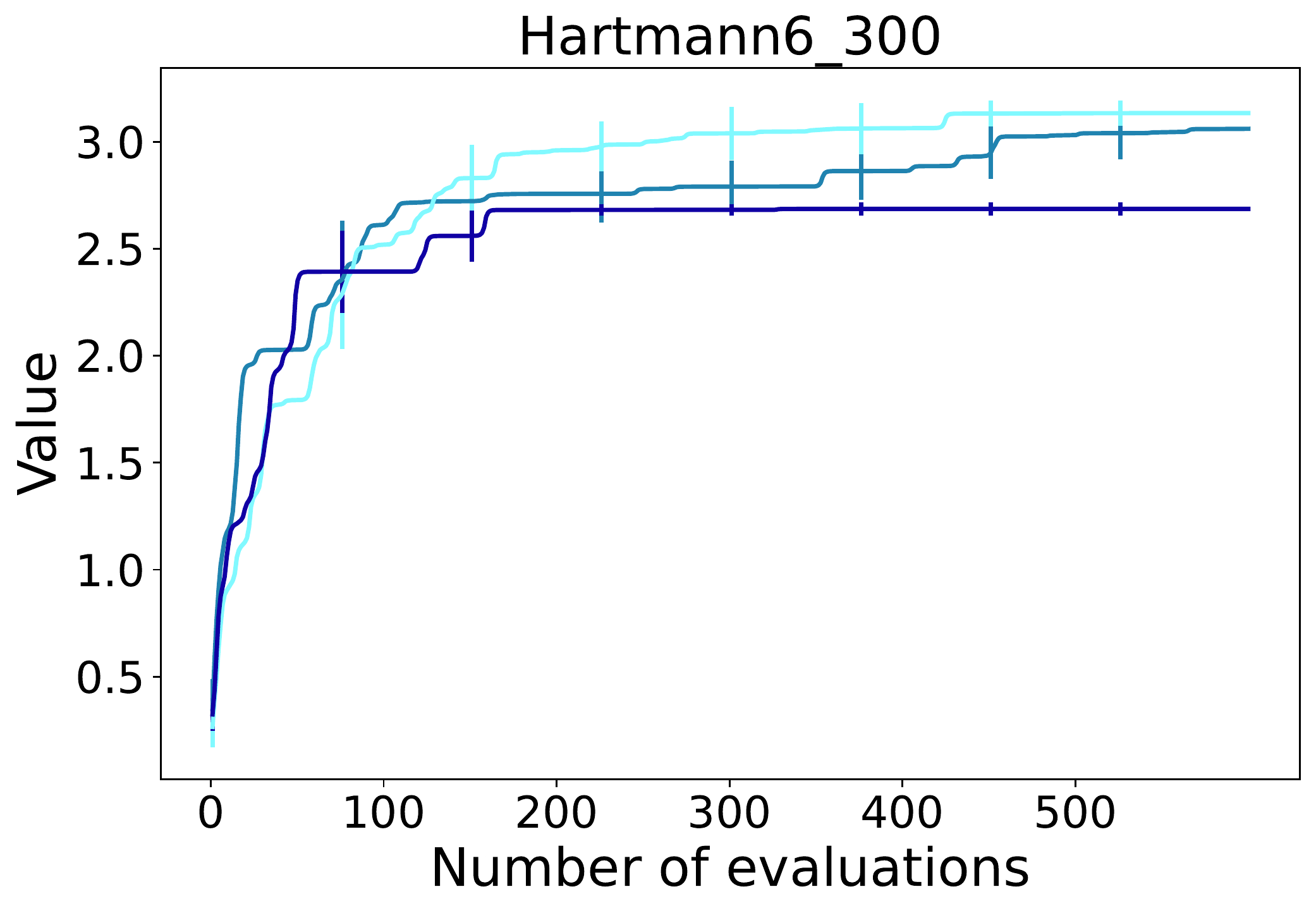}}
    \subfigure{\includegraphics[width=0.4\textwidth]{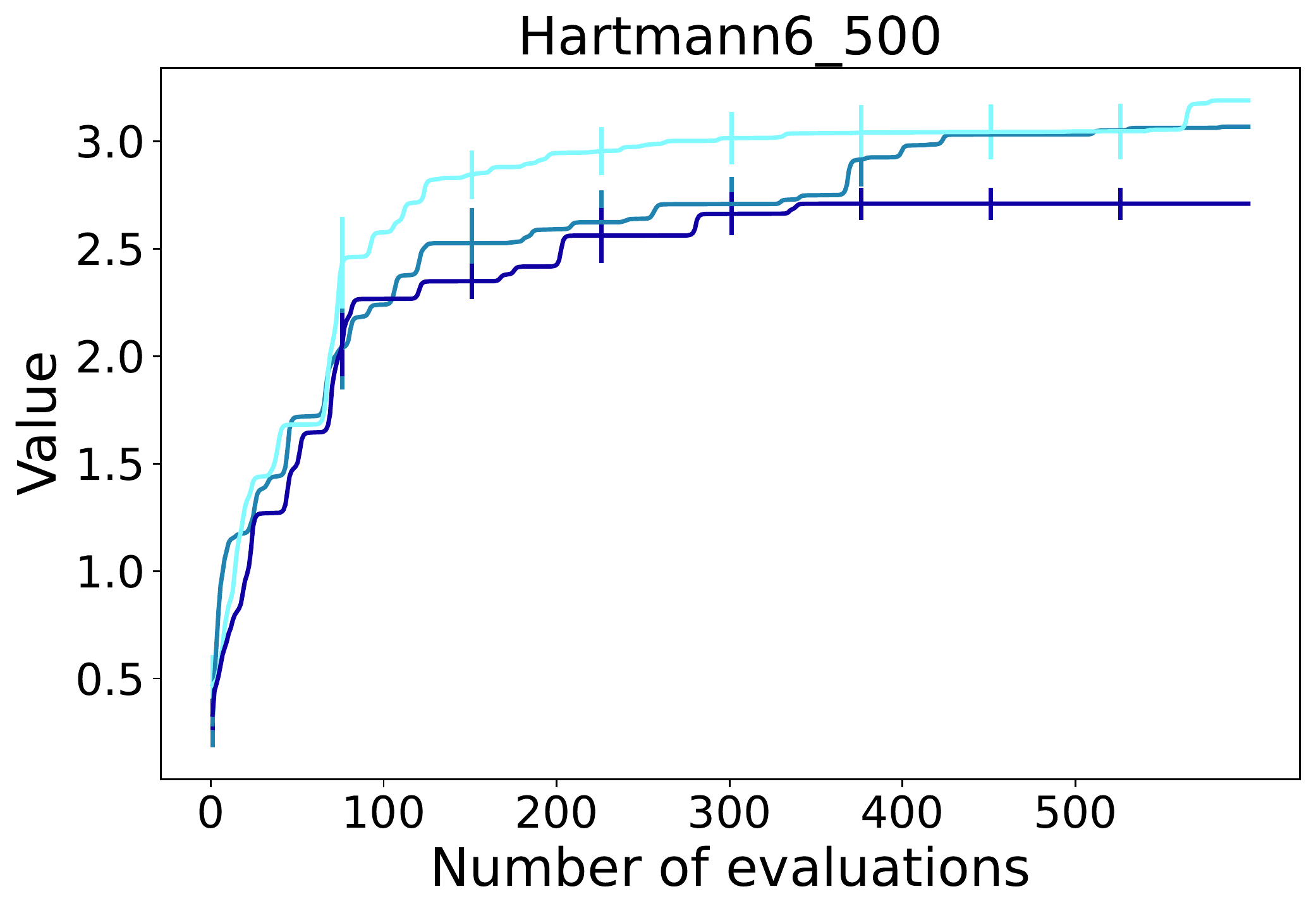}}
    \caption{Effectiveness of MCTS-VS when equipped with RS.}
    \label{fig:ablation:exp1}
\end{figure}

Next we compare the performance of MCTS-VS equipped with RS, BO and TuRBO, by experiments on the Hartmann functions with increasing ratio of valid variables. Hartmann$6$\_$500$ has $6$ valid variables. Hartmann$6$\_$5$\_$500$ is generated by mixing $5$ Hartmann$6$ functions as Hartmann$6(\bm x_{1:6}) +$ Hartmann$6(\bm x_{7:12})+\cdots+$ Hartmann$6(\bm x_{25:30})$, and appending 470 unrelated dimensions, where $\bm x_{i: j}$ denotes the $i$-th to $j$-th variables. Hartmann$6$\_$10$\_$500$ is generated alike. Thus, Hartmann$6$\_$5$\_$500$ and Hartmann$6$\_$10$\_$500$ have $30$ and $60$ valid variables, respectively. The results in Figure~\ref{fig:ablation:solver} show that as the ratio of valid variables increases, MCTS-VS-TuRBO gradually surpasses MCTS-VS-RS and MCTS-VS-BO, while MCTS-VS-RS becomes worse and worse. This is expected. If the ratio of valid variables is high, MCTS-VS is more likely to select the valid variables, so it is worth to use the expensive optimization algorithm, e.g., TuRBO, to optimize the selected variables. If the ratio is low, unrelated variables are more likely to be selected most of the time, so using a cheap optimization algorithm would be better. These observations also give us some guidance on selecting optimization algorithms in practice.

\begin{figure}[htbp]
    \centering
    \subfigure{\includegraphics[width=0.6\textwidth]{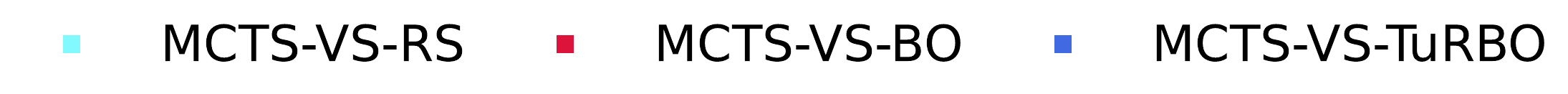}}\\
    \centering
    \subfigure{\includegraphics[width=0.33\textwidth]{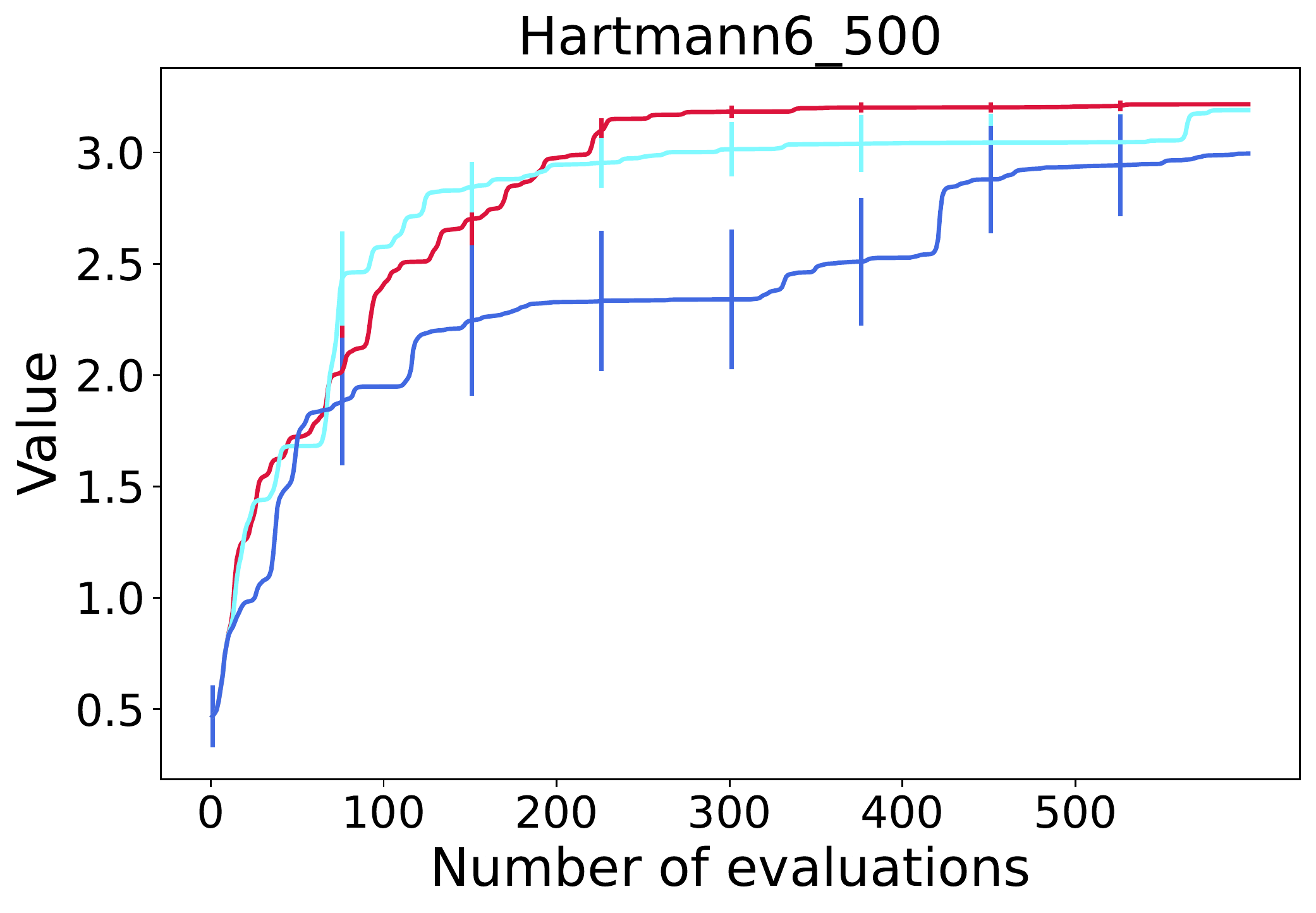}}
    \subfigure{\includegraphics[width=0.32\textwidth]{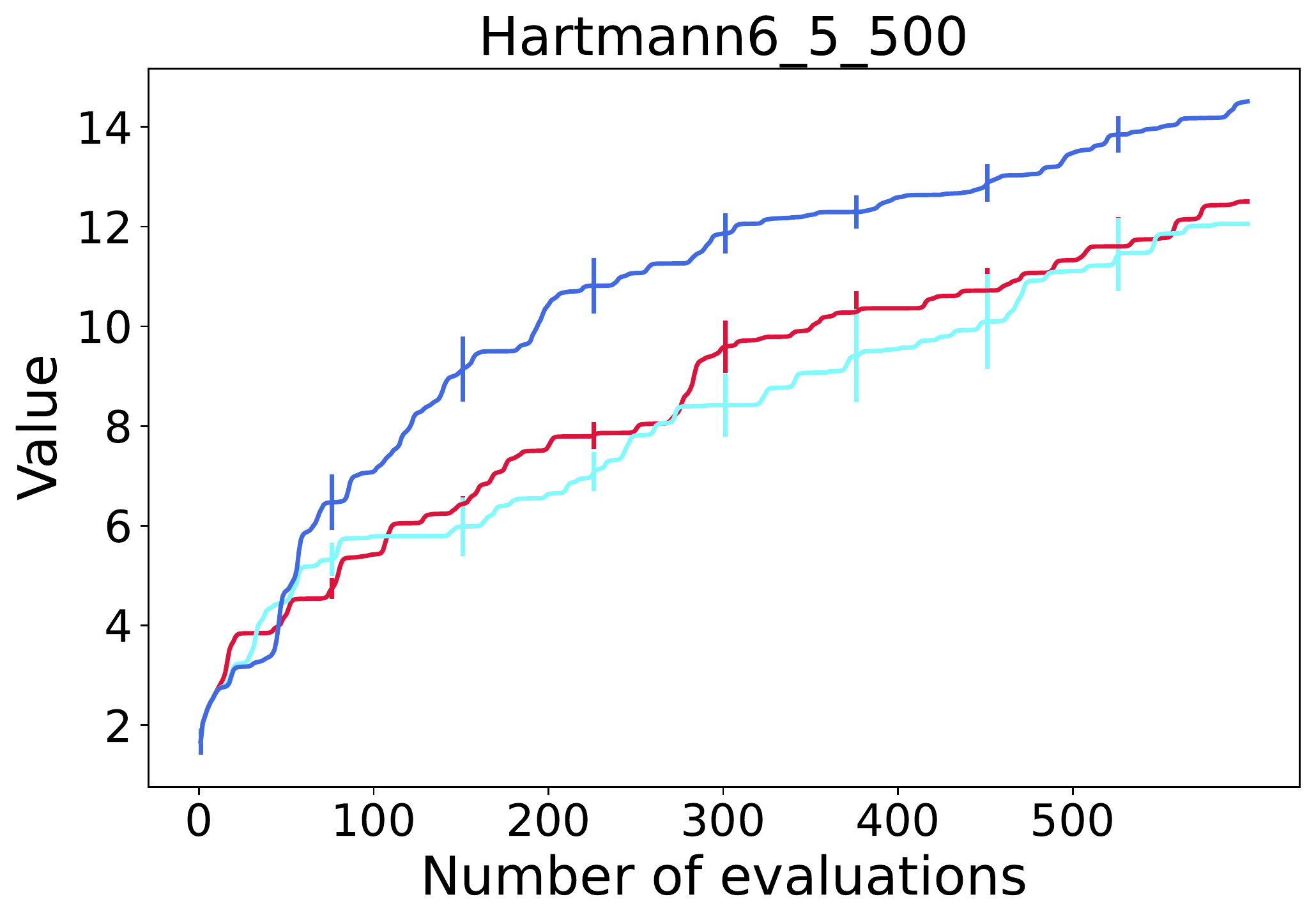}}
    \subfigure{\includegraphics[width=0.32\textwidth]{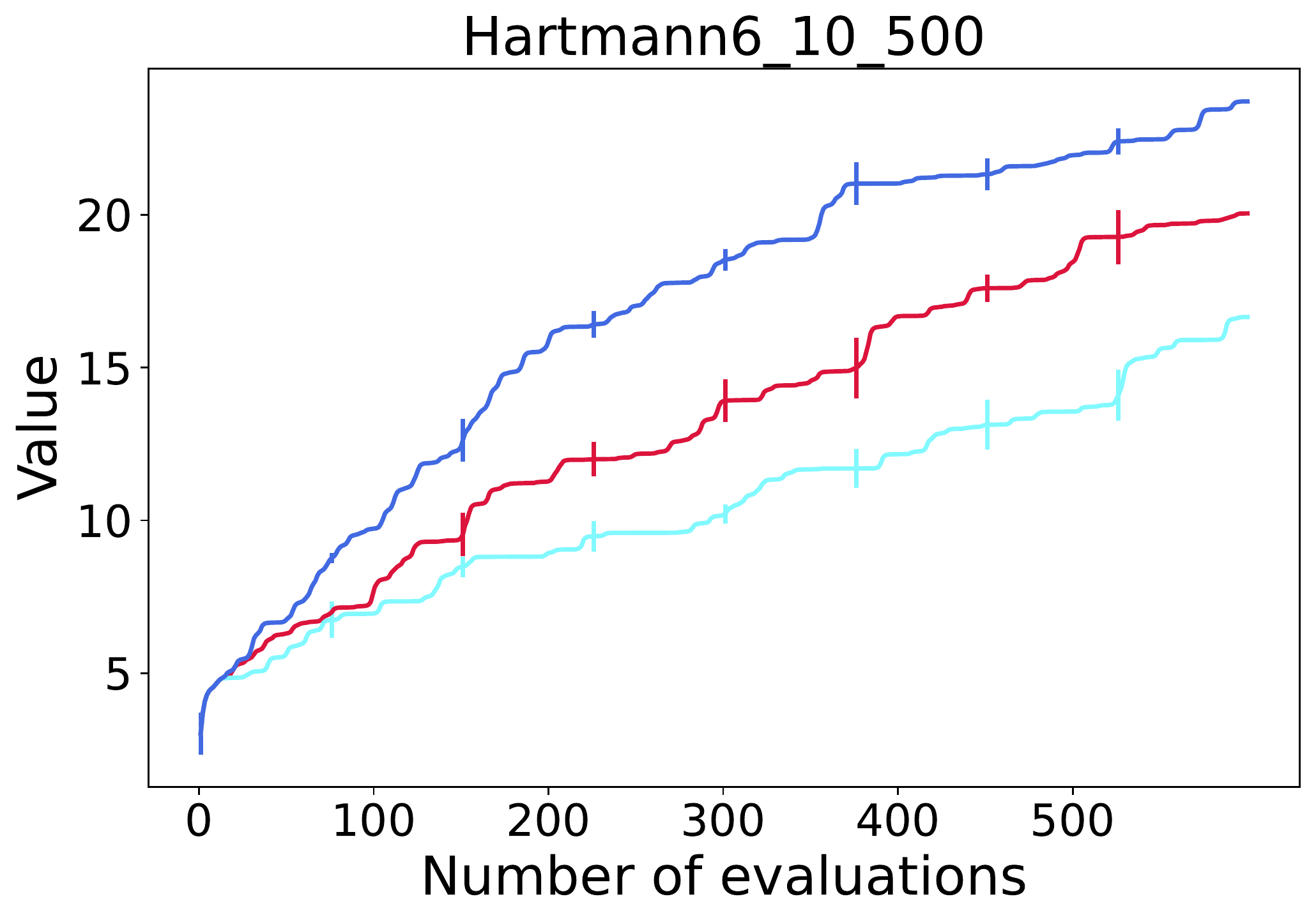}}
    \caption{Sensitivity analysis of the optimization algorithm.}
    \label{fig:ablation:solver}
\end{figure}

\textbf{``Fill-in'' strategy} is a basic component of variable selection methods, which influences the quality of the value of unselected variables. We compare the employed best-$k$ strategy ($k=20$) with \rbt{the average best-$k$ strategy} and the random strategy. \rbt{The average best-$k$ strategy uses the average of the best $k$ data points for the unselected variables,} and the random strategy samples the value of an unselected variable from its domain randomly. As shown in Figure~\ref{fig:fill in}, the random strategy leads to the poor performance of MCTS-VS-BO, which may be because it does not utilize the historical information and leads to over-exploration. The best-$k$ strategy utilizes the historical points that have high objective values to fill in the unselected variables, thus behaving much better. \rbt{The performance of the average strategy is between the best-$k$ and random strategies.} We recommend using the best-$k$ strategy in practice.

\textbf{The hyper-parameter $k$ used in the best-$k$ strategy} controls the degree of exploitation for the unselected variables. As shown in Figure~\ref{fig:paramk}, a smaller $k$ encourages exploitation, which results in better performance in the early stage, but easily leads to premature convergence. A larger $k$ encourages exploration and behaves worse in the early stage, but may converge to a better value. We recommend using a larger $k$ if allowing enough evaluations.

\begin{figure}[htbp]
    \centering
    \subfigure[\rbt{``Fill-in'' strategy}]{\includegraphics[width=0.4\textwidth]{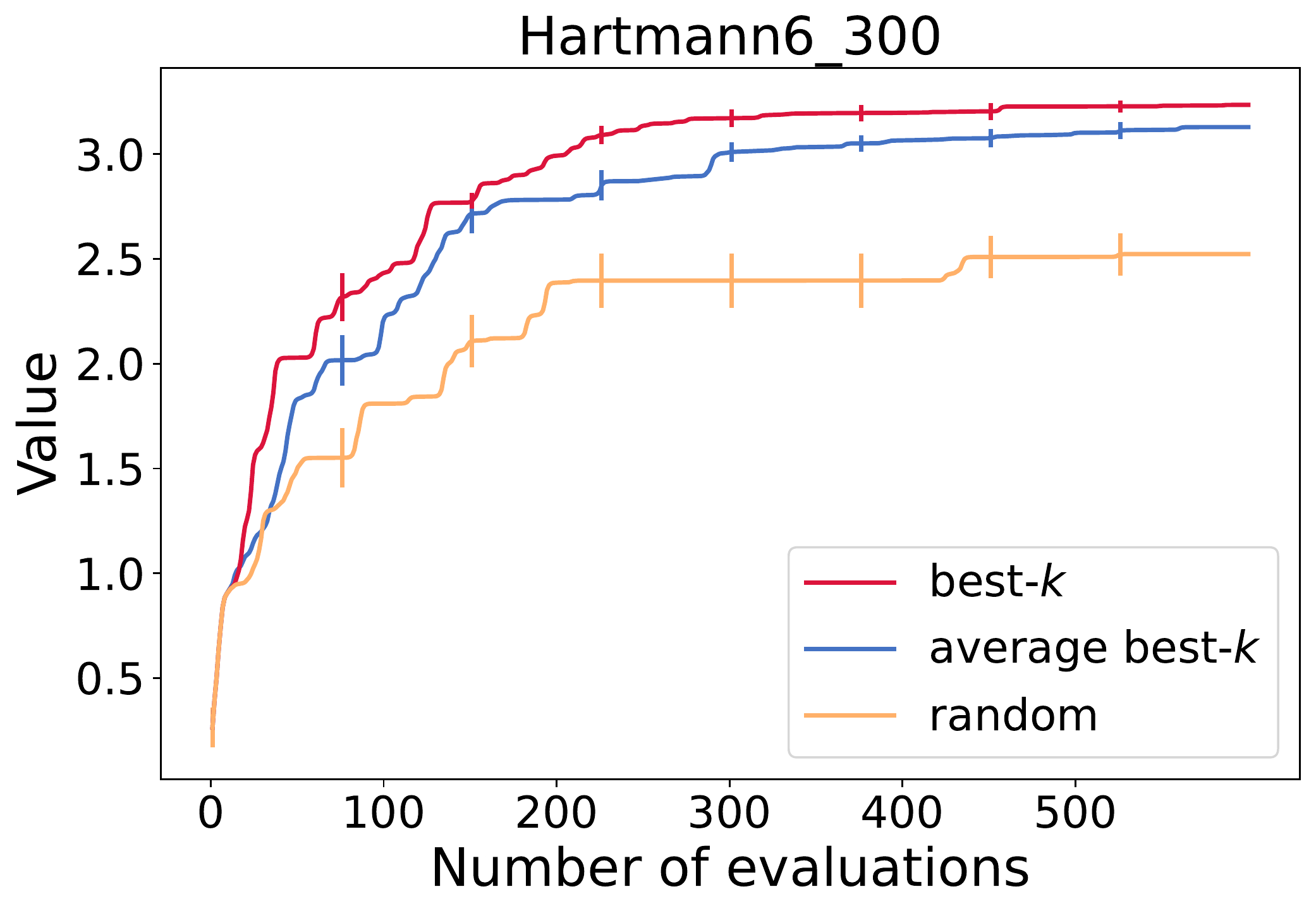}\label{fig:fill in}}
    \subfigure[Hyper-parameter $k$ of the best-$k$ strategy]{\includegraphics[width=0.4\textwidth]{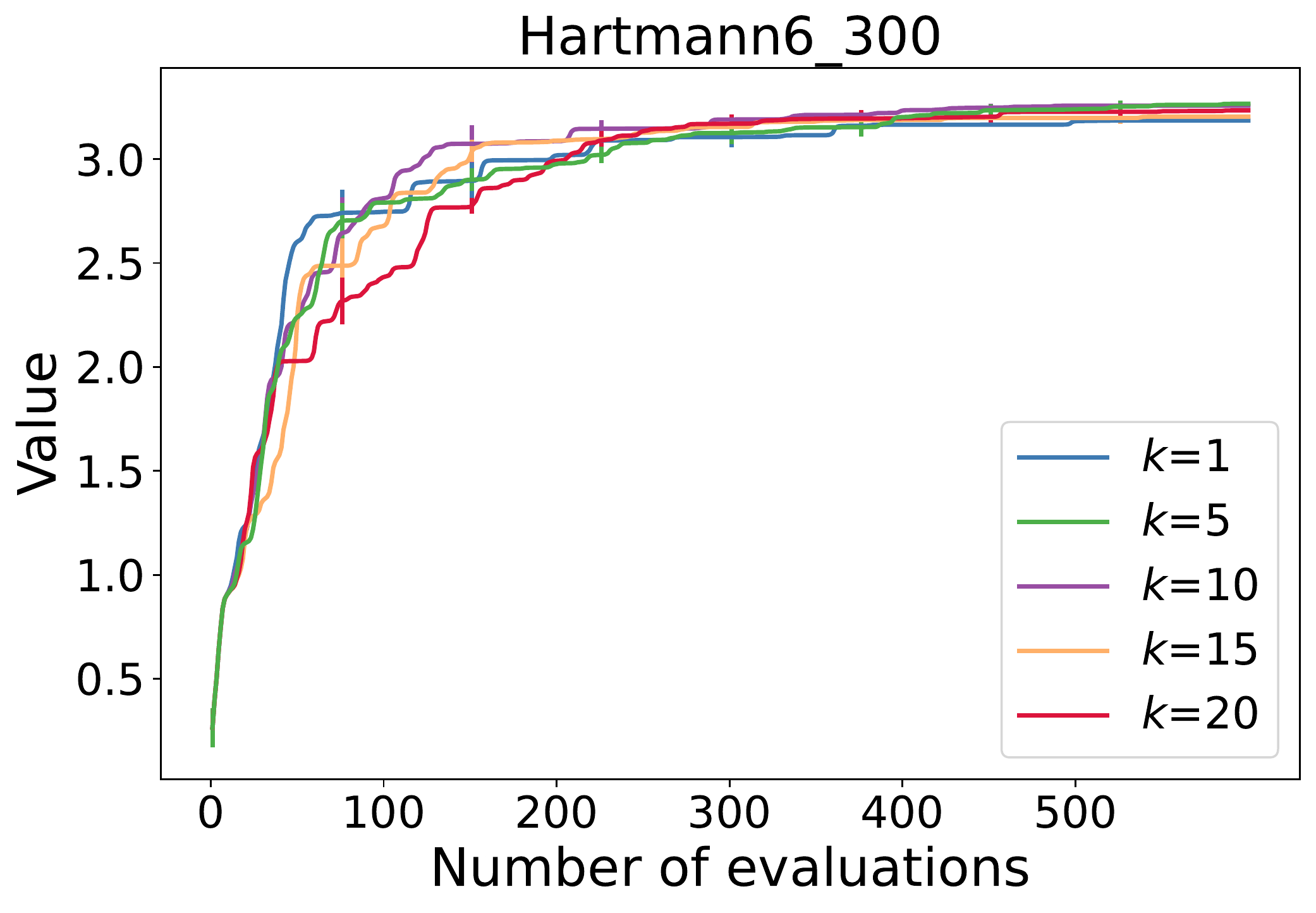}\label{fig:paramk}}
    \caption{Sensitivity analysis of the ``fill-in'' strategy and the hyper-parameter $k$ of the best-$k$ strategy, using MCTS-VS-BO on Hartmann$6$\_$300$.}
    \label{fig:ablation:fill in and k}
\end{figure}

\textbf{The hyper-parameter $C_p$ for calculating UCB in Eq.~(\refeq{eq-MCTS-UCB})} balances the exploration and exploitation of MCTS. As shown in Figure~\ref{fig:ablationcp}, a too small $C_p$ leads to relatively worse performance, highlighting the importance of exploration. A too large $C_p$ may also lead to over-exploration. But overall MCTS-VS is not very sensitive to $C_p$. We recommend setting $C_p$ between $1\%$ and $10\%$ of the optimum (i.e., $\max f(\bm x)$\rbt{)}, which is consistent with that for LA-MCTS~\citep{lamcts}.

\begin{figure}[htbp]
    \centering
    \subfigure{\includegraphics[width=0.24\textwidth]{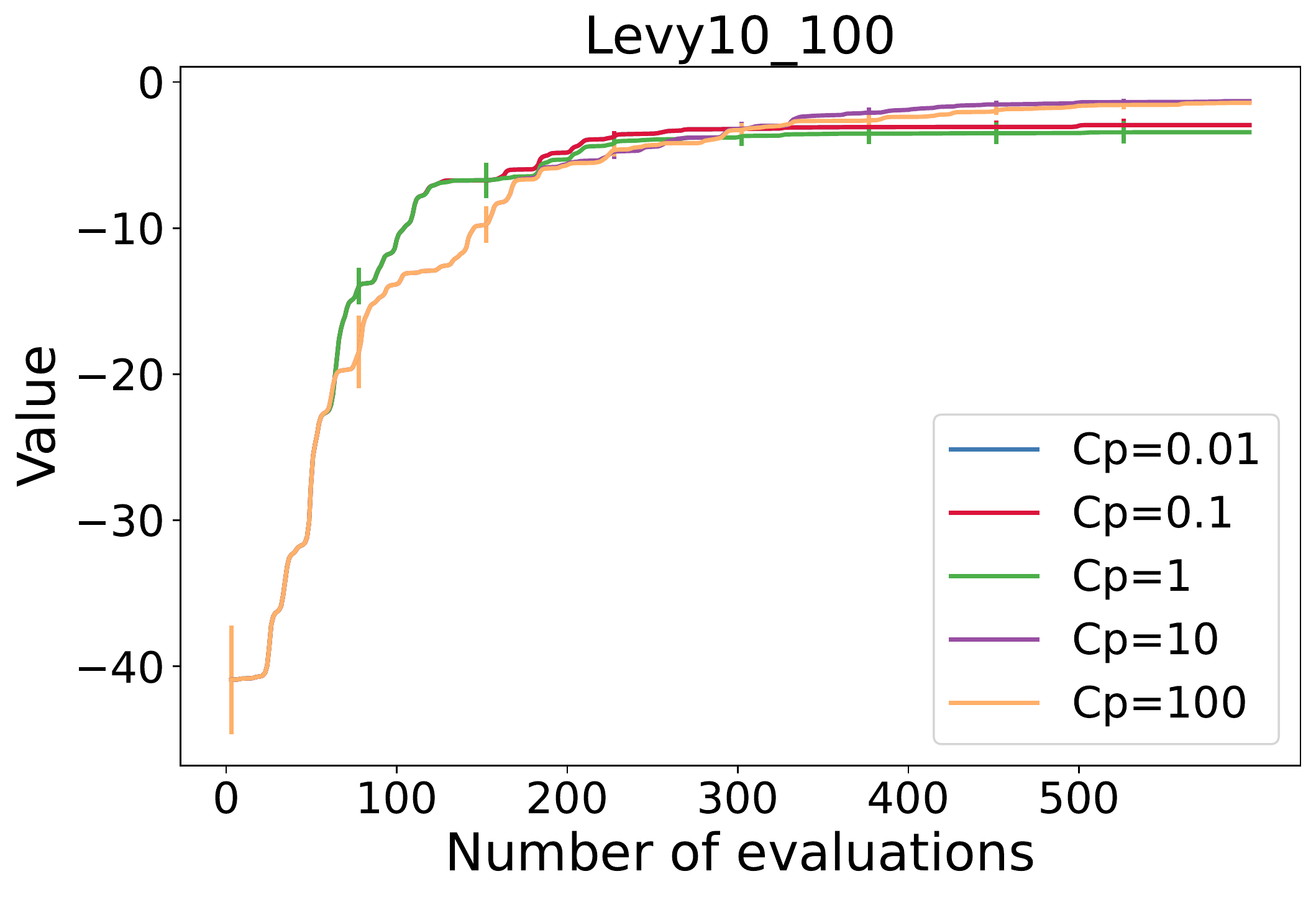}}
    \subfigure{\includegraphics[width=0.24\textwidth]{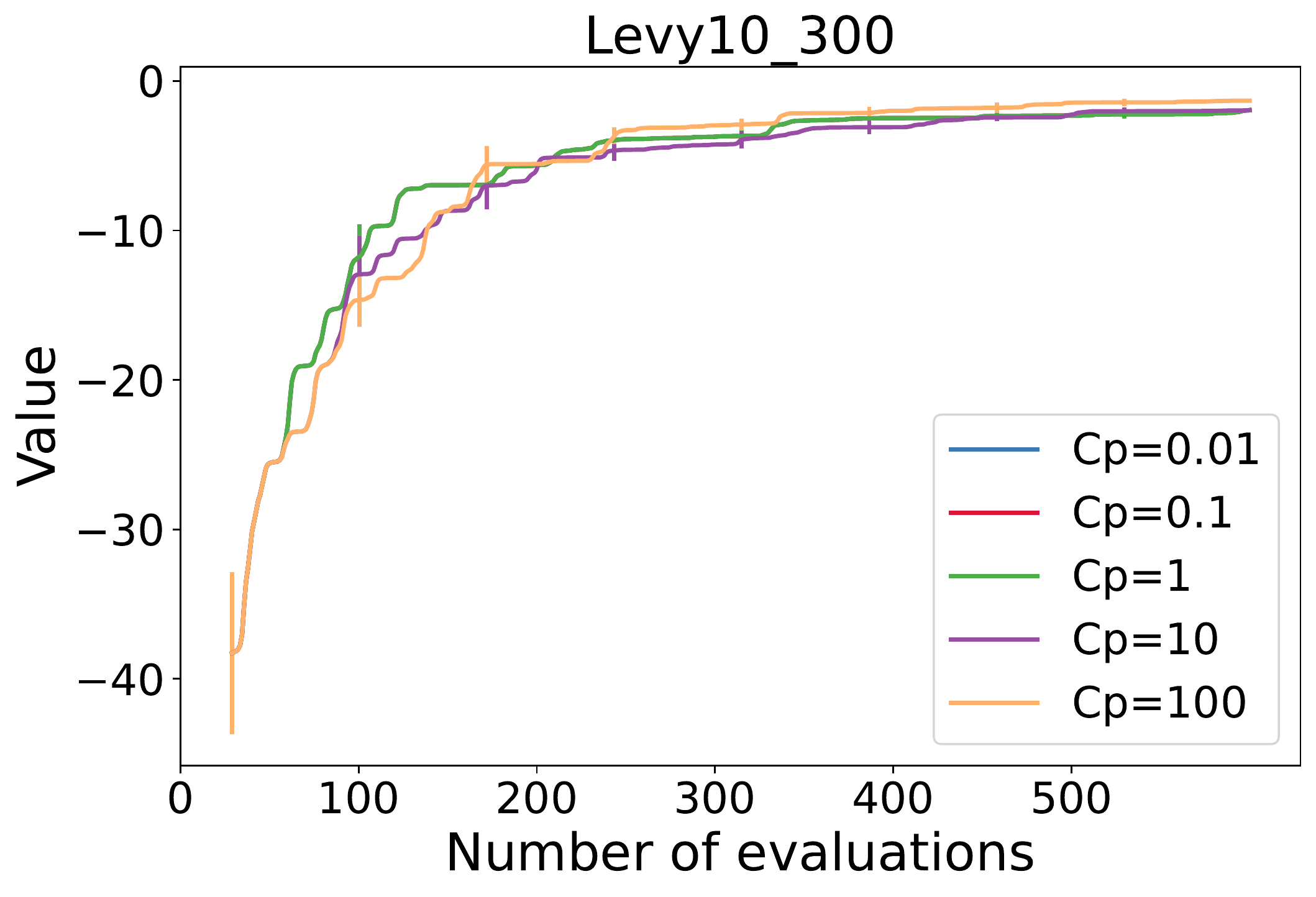}}
    \subfigure{\includegraphics[width=0.24\textwidth]{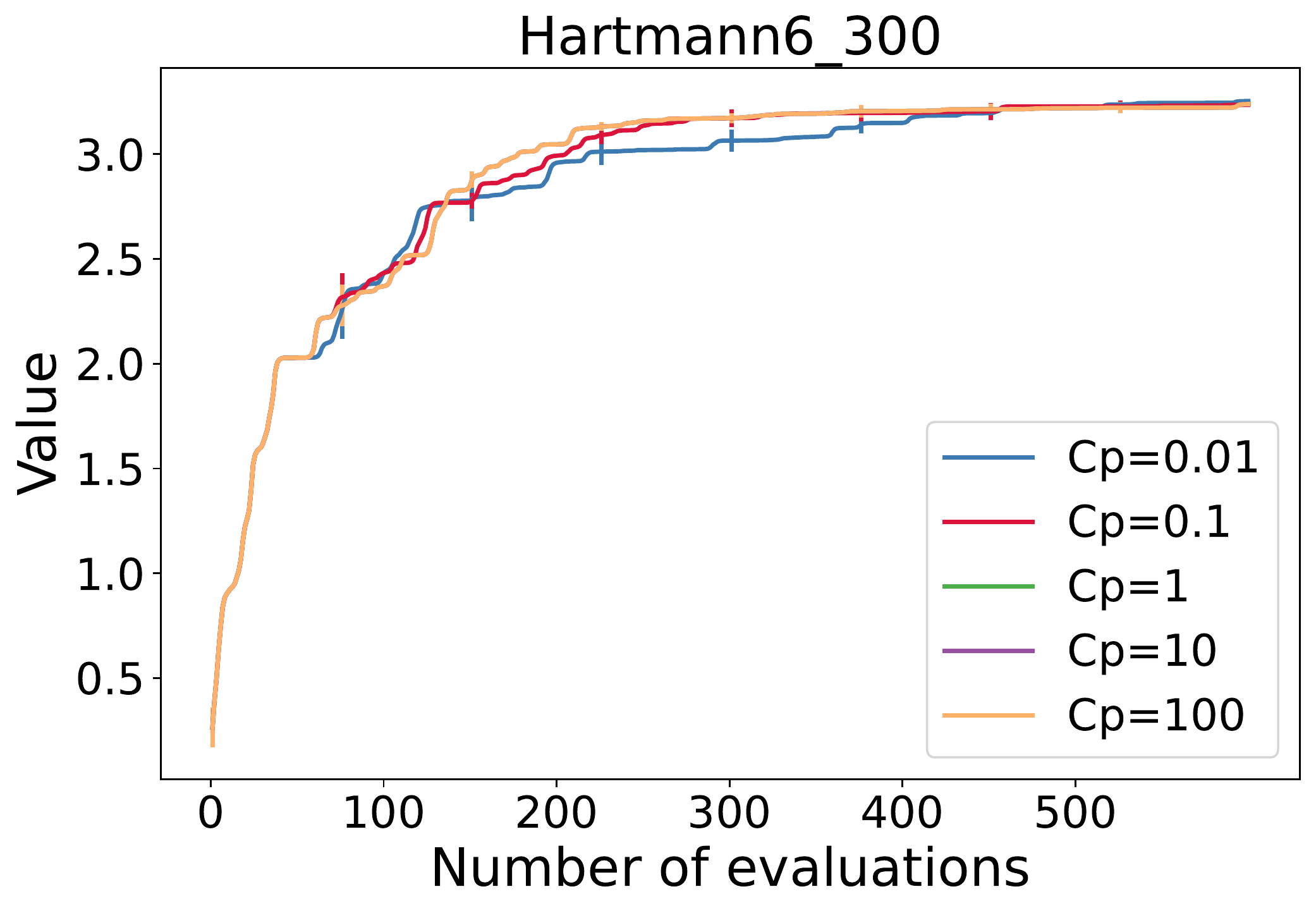}}
    \subfigure{\includegraphics[width=0.24\textwidth]{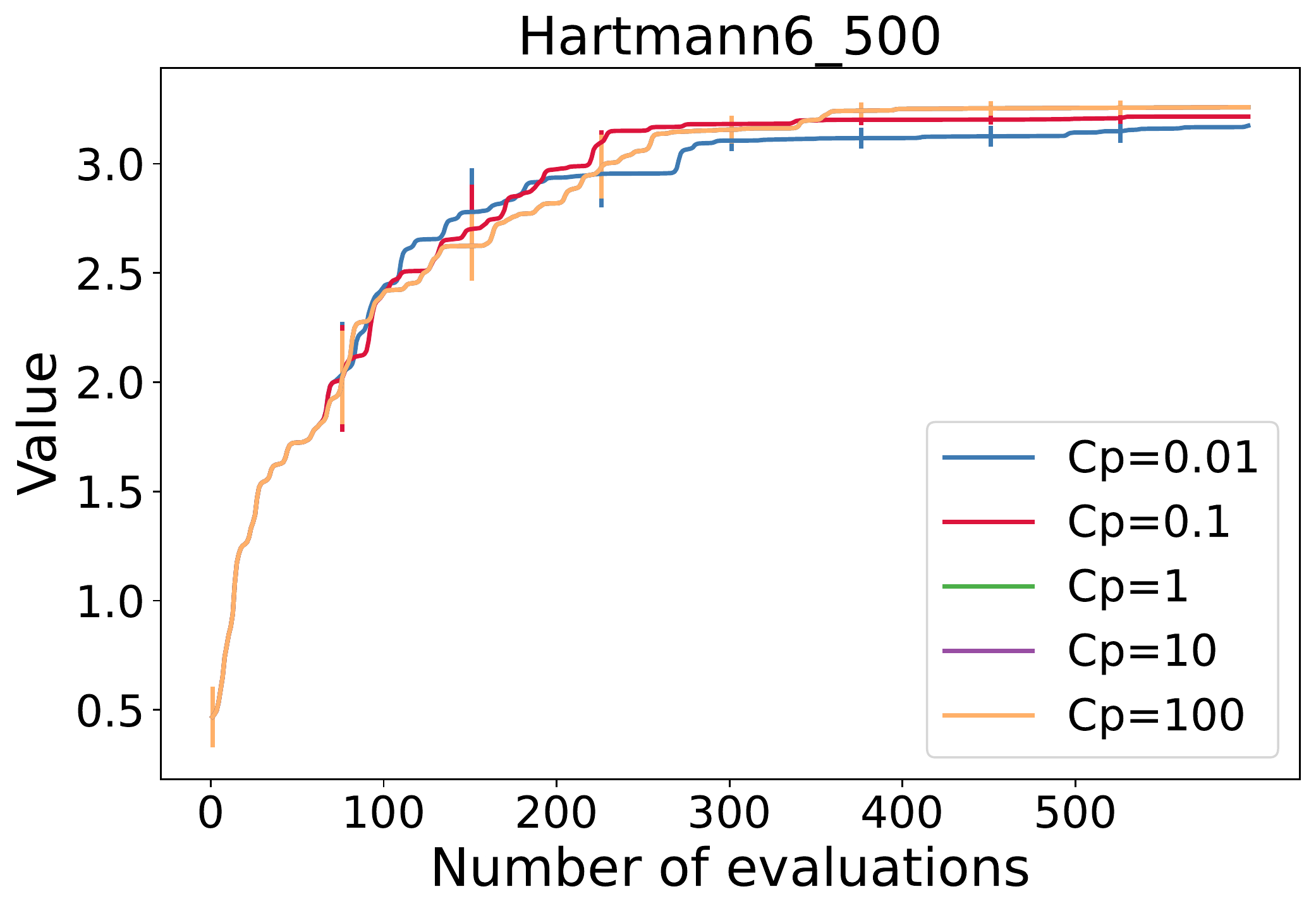}}
    \caption{Sensitivity analysis of the hyper-parameter $C_p$ for calculating UCB in Eq.~(\refeq{eq-MCTS-UCB}), using MCTS-VS-BO on Levy and Hartmann.}
    \label{fig:ablationcp}
\end{figure}

\textbf{The number $2\times N_v\times N_s$ of sampled data in each iteration} depends on the batch size $N_v$ of variable index subset and the sample batch size $N_s$, and will influence the accuracy of estimating the variable score vector in Eq.~(\refeq{eq-score}). If we increase $N_v$ and $N_s$, we can calculate the variable score more accurately, but also need more evaluations. Figure~\ref{fig:number of samples} shows that given the same number of evaluations, MCTS-VS-BO achieves the best performance when $N_v=2$ and $N_s=3$. Thus, this setting may be a good choice to balance the accuracy of variable score and the number of evaluations, which is also used throughout the experiments. 

\textbf{The threshold $N_{bad}$ for re-initializing a tree} controls the tolerance of selecting bad tree nodes (i.e., nodes containing unimportant variables). A smaller $N_{bad}$ leads to frequent re-initialization, which can adjust quickly but may cause under-exploitation of the tree. A larger $N_{bad}$ can make full use of the tree, but may optimize too much on unimportant variables. Figure~\ref{fig:nbad} shows that MCTS-VS achieves the best performance when $N_{bad}=5$. Thus, we recommend to use this setting, to balance the re-initialization and exploitation of the tree.

\textbf{The threshold $N_{split}$ for splitting a node.} If the number of variables in a node is larger than $N_{split}$, the node can be further partitioned. That is, the parameter $N_{split}$ controls the least number of variables in a leaf node and thus affects the number of selected variables, which has a direct influence on the wall clock time. Note that MCTS-VS selects a leaf node and optimizes the variables contained by this node in each iteration. The smaller $N_{split}$, the shorter the time. Figure~\ref{fig:nsplit} shows that $N_{split}$ has little influence on the performance of MCTS-VS-BO, and thus we recommend to set $N_{split}=3$ to reduce the wall clock time. 

\begin{figure}[htbp]
    \centering
    \subfigure[Number of samples]{\includegraphics[width=0.32\textwidth]{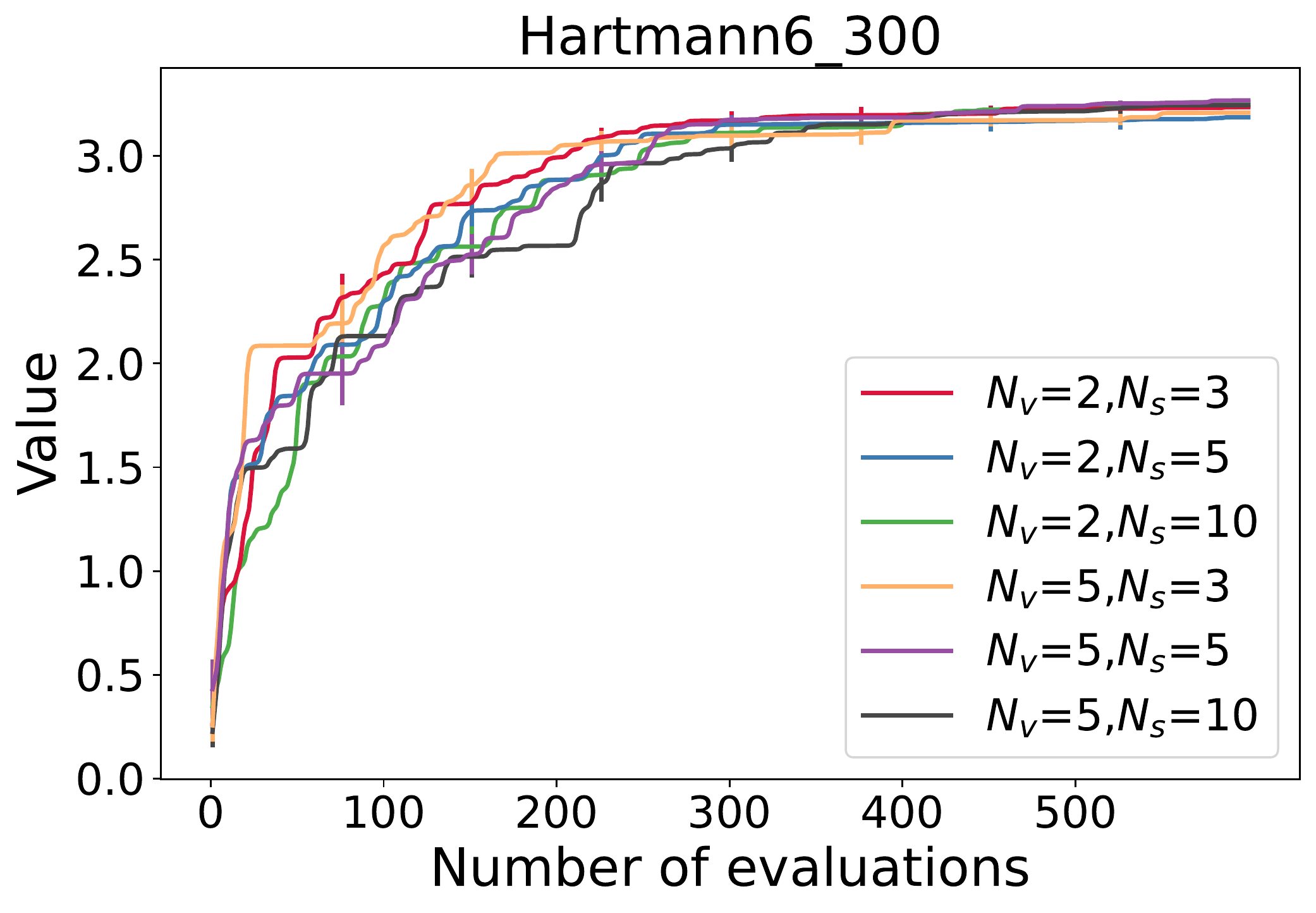}\label{fig:number of samples}}
    \subfigure[$N_{bad}$]{\includegraphics[width=0.32\textwidth]{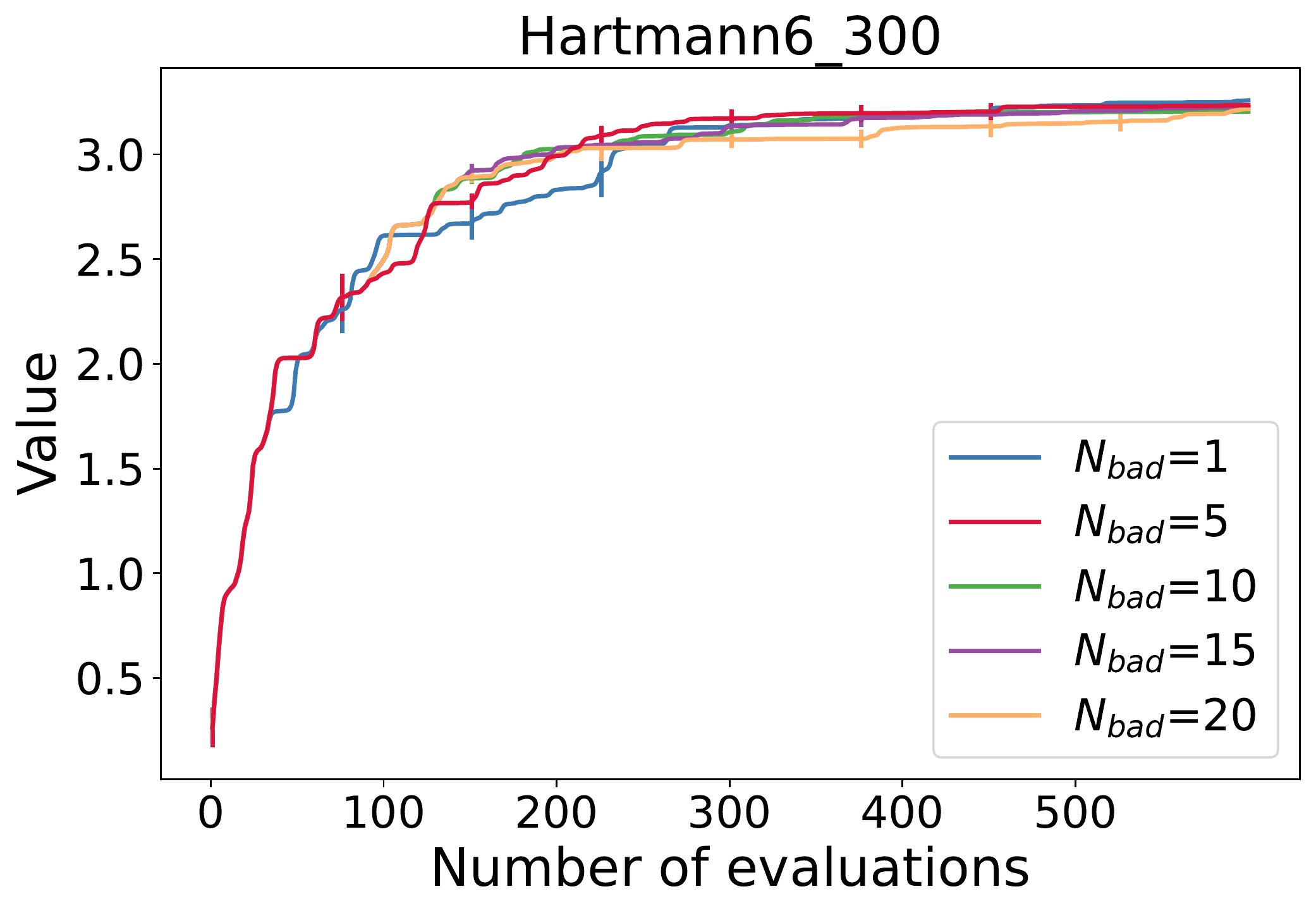}\label{fig:nbad}}
    \subfigure[$N_{split}$]{\includegraphics[width=0.32\textwidth]{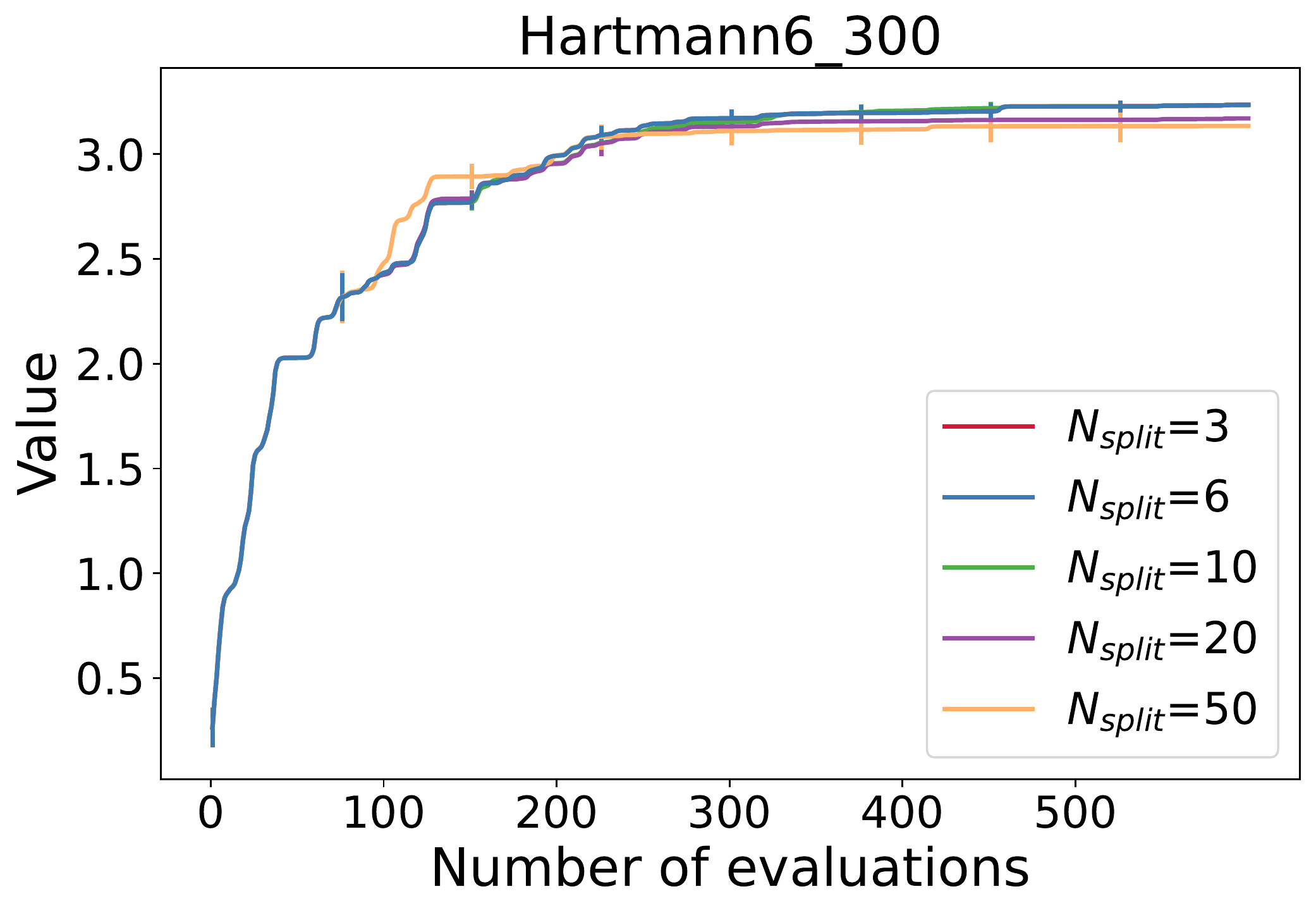}\label{fig:nsplit}}
    \caption{Sensitivity analysis of the number $2\times N_v\times N_s$ of sampled data in each iteration, the threshold $N_{bad}$ for re-initializing a tree and the threshold $N_{split}$ for splitting a node, using MCTS-VS-BO on Hartmann$6$\_$300$.}
    \label{fig:ablation1}
\end{figure}

\rbt{\textbf{Influence of the hyper-parameters on the runtime of MCTS-VS.} We also provide some intuitive explanation about the influence of the hyper-parameters on the runtime. The threshold $N_{split}$ for splitting a node has a direct impact on the runtime, because it controls the least number of variables to be optimized in a leaf node. That is, the runtime will increase with $N_{split}$. Other parameters may affect the depth of the tree and thus the runtime. For the threshold $N_{bad}$ for re-initializing a tree, if it is set to a small value, MCTS-VS will re-build the tree frequently and the depth of the tree is small. The shallow nodes have more variables, leading to more runtime to optimize. For the hyper-parameter $C_p$ for calculating UCB, if it is set to a large value, the exploration is preferred and MCTS-VS will tend to select the right node (regarded as containing unimportant variables). The tree thus will be re-built frequently, leading to more runtime. For the number $2\times N_v\times N_s$ of sampled data at each iteration, if $N_v$ and $N_s$ are set to large values, the depth of the tree will be small given the total number of evaluations, and thus lead to more runtime.}

\section{Additional Experiments}
\label{appendix:additionalexperiments}


\textbf{Detailed results on NAS-Bench-101 and NAS-Bench-201.} Figure~\ref{fig:appendixnas} shows the performance of the compared methods on the task of NAS-Bench-101 and NAS-Bench-201 when using the number of evaluations and wall clock time as the $x$-axis, respectively. Though most of their performance is similar in the left two subfigures, it can be clearly observed from the right two subfigures that MCTS-VS-BO uses the least time to achieve the best accuracy. 
Note that we only show the subfigures with the wall clock time as the $x$-axis in the main paper due to the space limitation. Besides, we also run a longer time here (i.e., in the right two subfigures) to provide a more complete observation. 

\begin{figure}[htbp]
    \centering
    \subfigure{\includegraphics[width=0.65\textwidth]{final-version/legend/exp2_legend_1.pdf}}\\
    \vspace{-1.3em}
    \subfigure{\includegraphics[width=0.6\textwidth]{final-version/legend/exp2_legend_2.pdf}}\\
    \centering
    \subfigure{\includegraphics[width=0.24\textwidth]{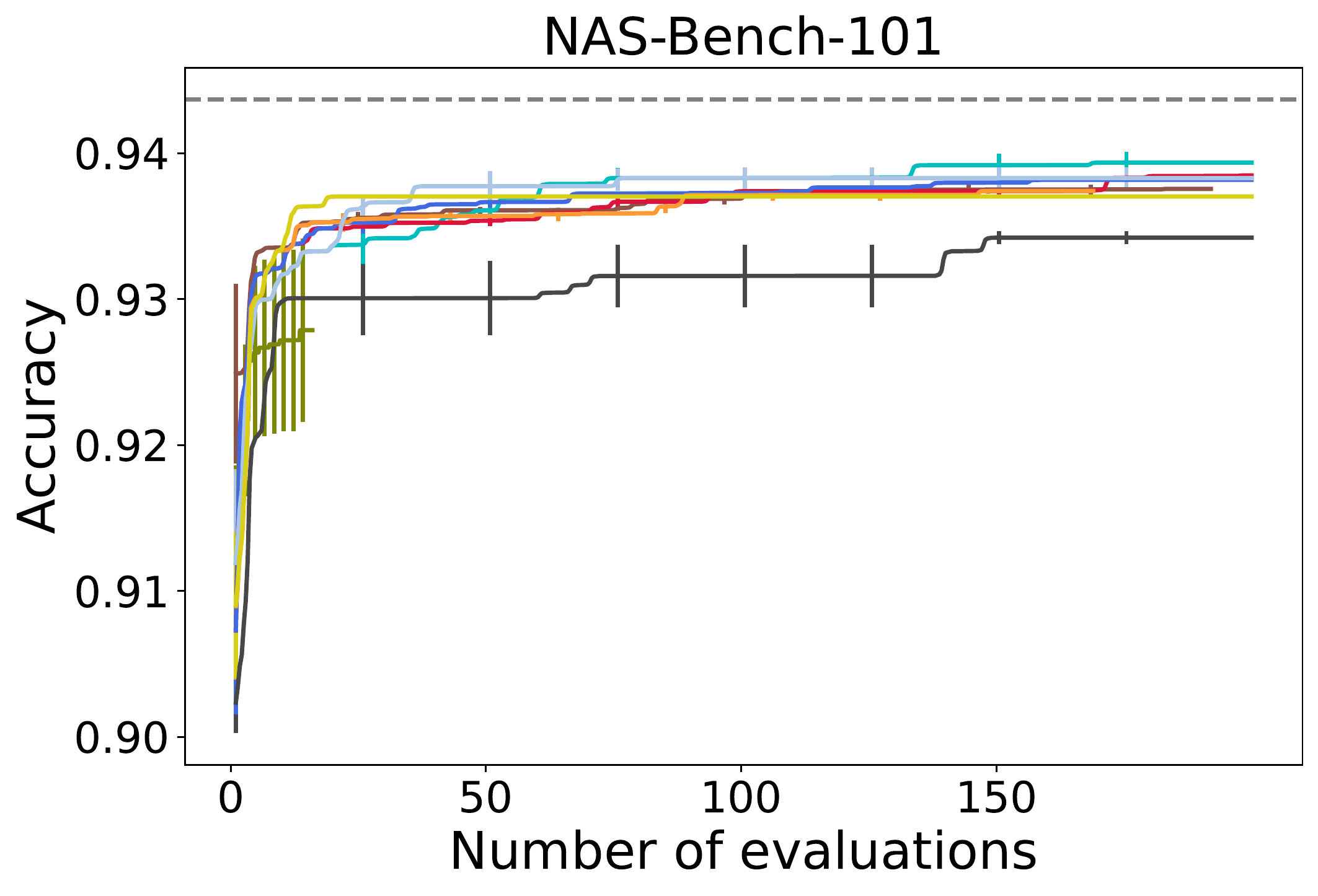}}
    \subfigure{\includegraphics[width=0.24\textwidth]{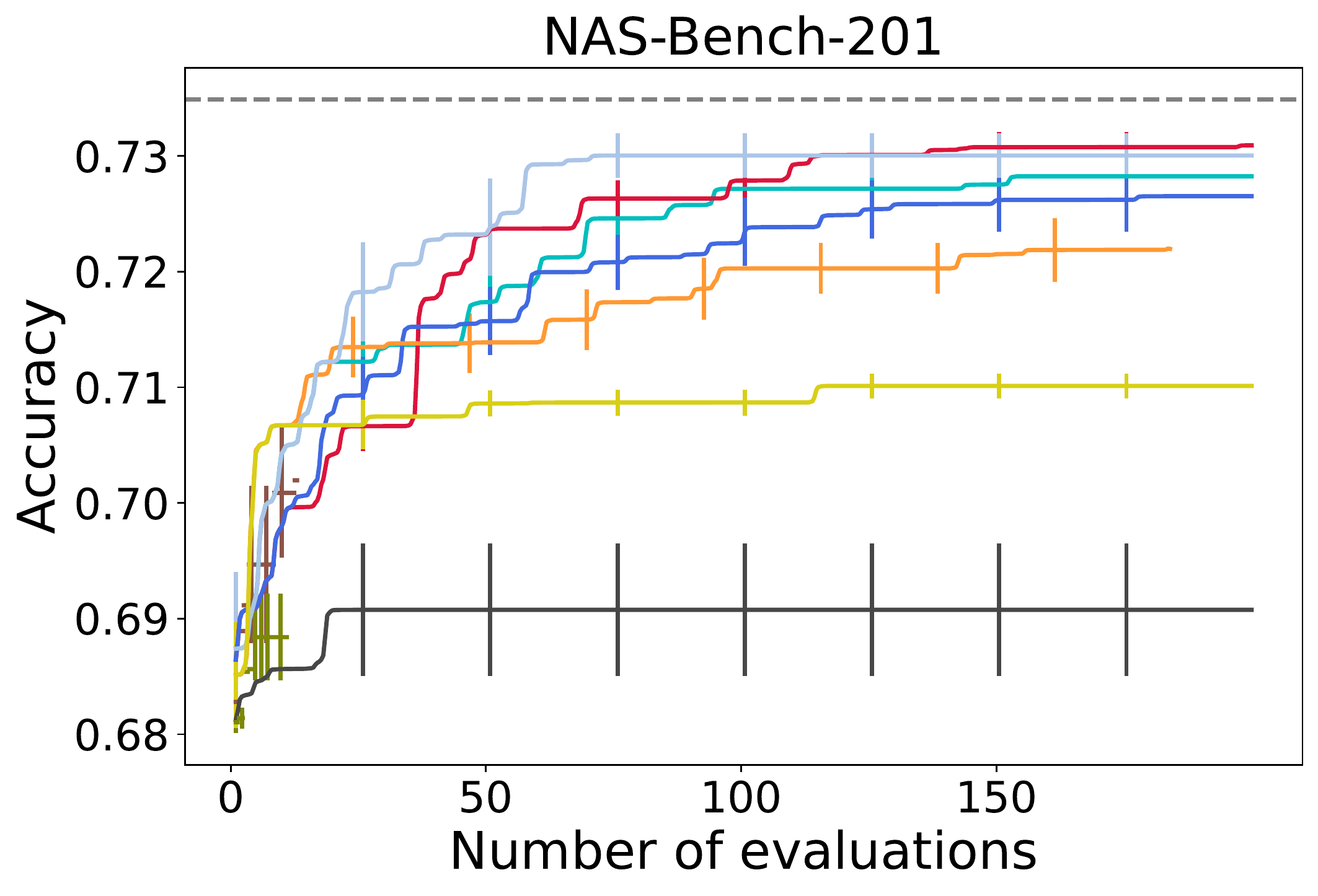}}
    \subfigure{\includegraphics[width=0.24\textwidth]{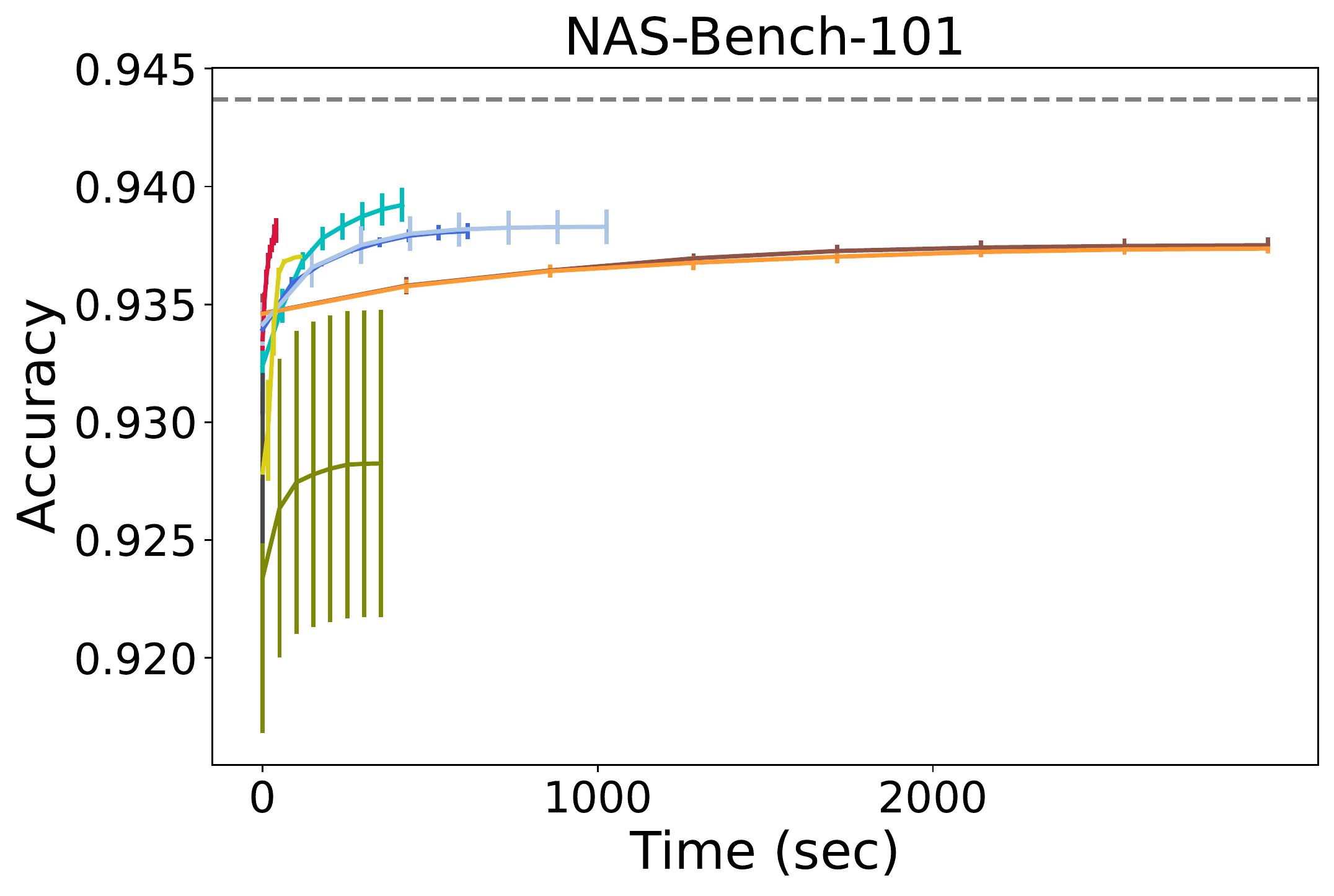}}
    \subfigure{\includegraphics[width=0.24\textwidth]{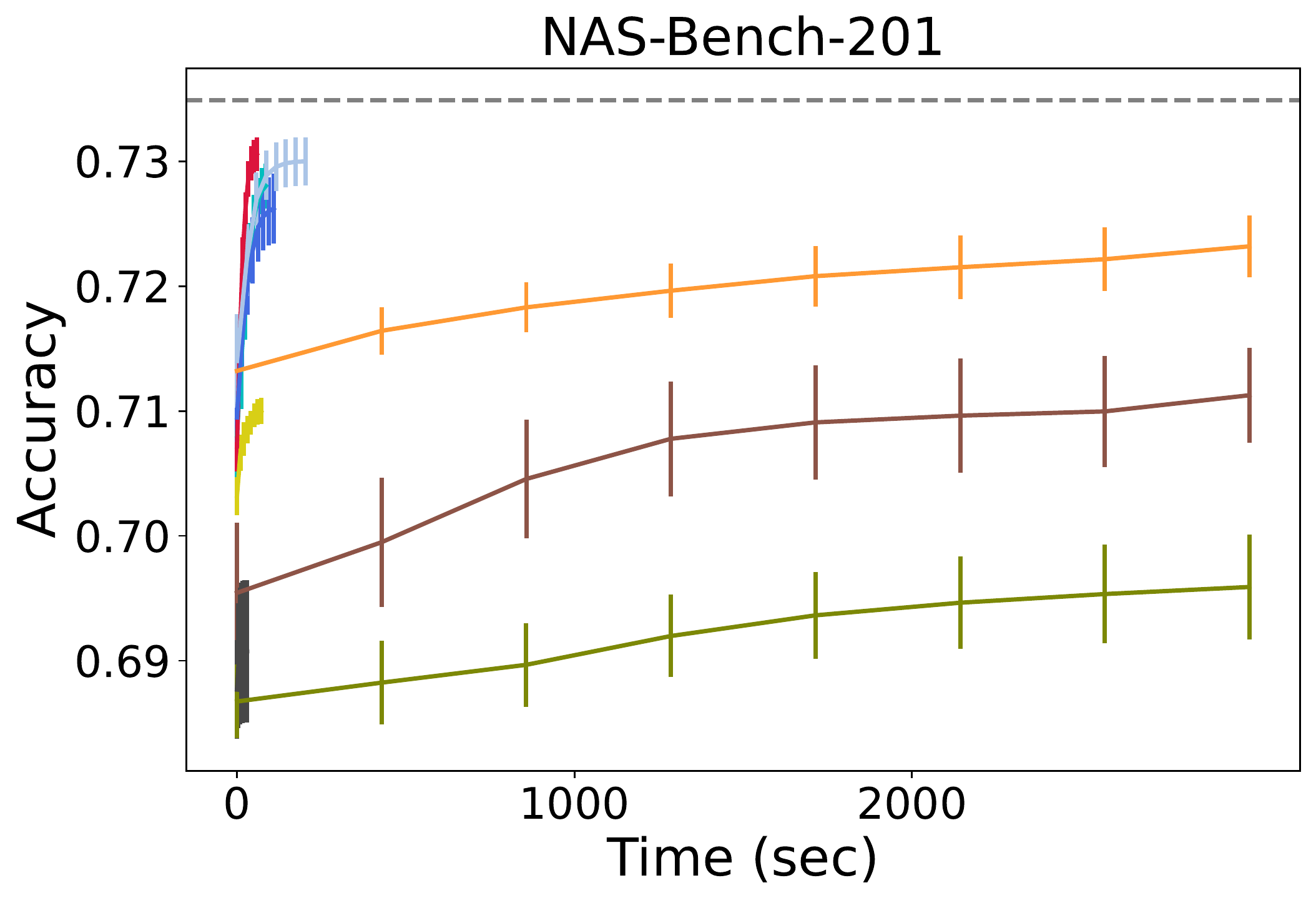}}
    \caption{Performance comparison on NAS-Bench-101 and NAS-Bench-201, using the number of evaluations and wall clock time as the $x$-axis, respectively.}
    \label{fig:appendixnas}
\end{figure}

\textbf{Experiments on more NAS-Bench problems.} We also conduct experiments on NAS-Bench-1Shot1~\cite{Zela2020NAS-Bench-1Shot1}, TransNAS-Bench-101~\cite{transnas} and NAS-Bench-ASR~\cite{mehrotra2021nasbenchasr}. NAS-Bench-1Shot1 is a weight-sharing benchmark based on one-shot NAS methods, deriving from the large architecture space of NASBench-101. TransNAS-Bench-101 is a benchmark dataset containing network performance across seven vision tasks, e.g., object classification, scene classification and so on. We use the scene classification task with cell-level search space in our experiments. NAS-Bench-ASR is a benchmark for Automatic Speech Recognition (ASR) and trained on the TIMIT audio dataset. For NAS-Bench-ASR, we use Phoneme Error Rate (PER) on the validation dataset as the metric. In the same way as~\cite{alebo}, we create problems with $D=33$, $D=24$ and $D=30$ for NAS-Bench-1Shot1, TransNAS-Bench-101 and NAS-Bench-ASR, respectively. The results in Figure~\ref{fig:additionalnas} show that MCTS-VS-BO still uses the least time to achieve the best performance.

\begin{figure}[htbp]
    \centering
    \subfigure{\includegraphics[width=0.65\textwidth]{final-version/legend/exp2_legend_1.pdf}}\\
    \vspace{-1.3em}
    \subfigure{\includegraphics[width=0.6\textwidth]{final-version/legend/exp2_legend_2.pdf}}\\
    \centering
    \includegraphics[width=0.32\textwidth]{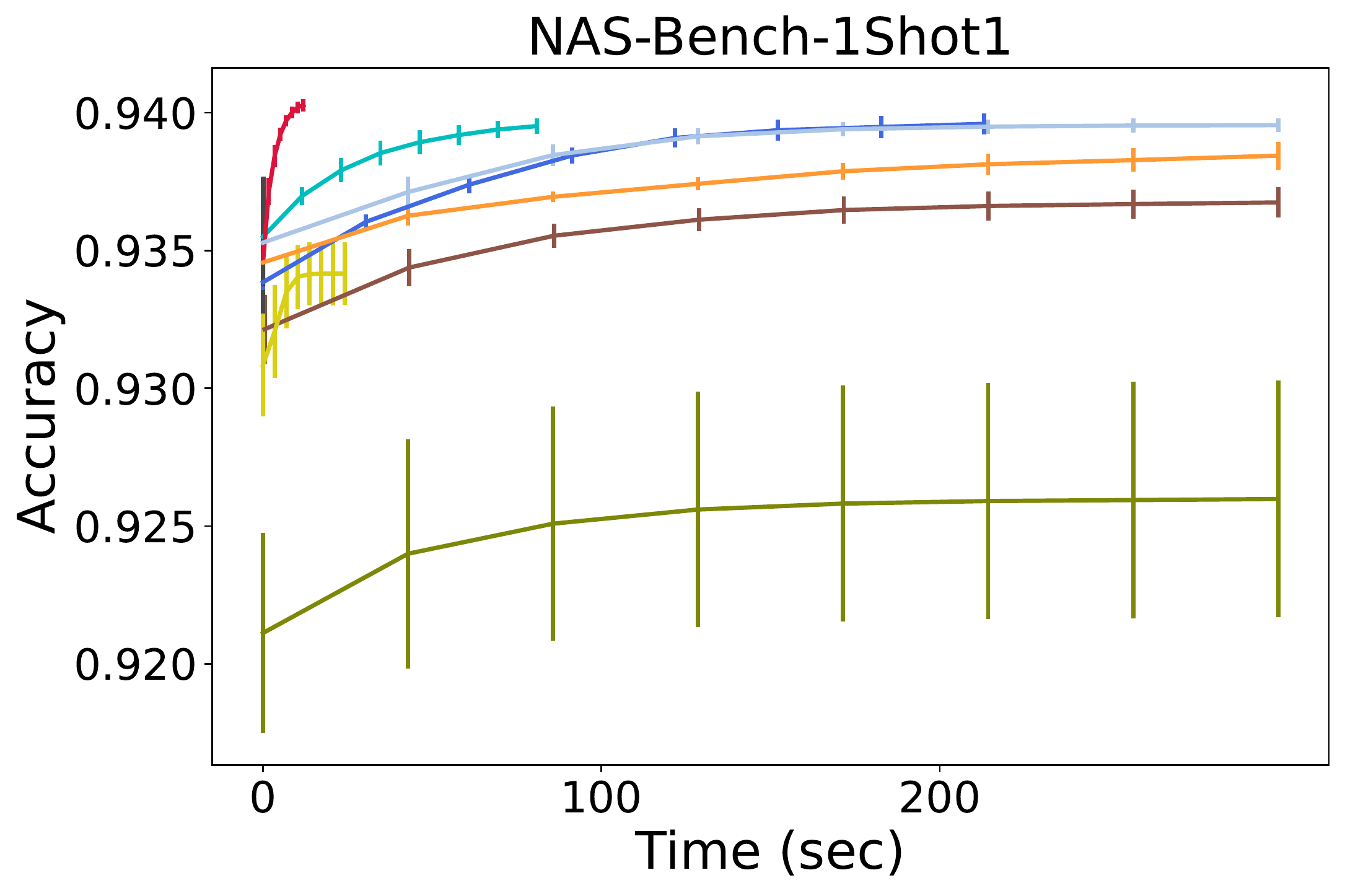}
    \includegraphics[width=0.32\textwidth]{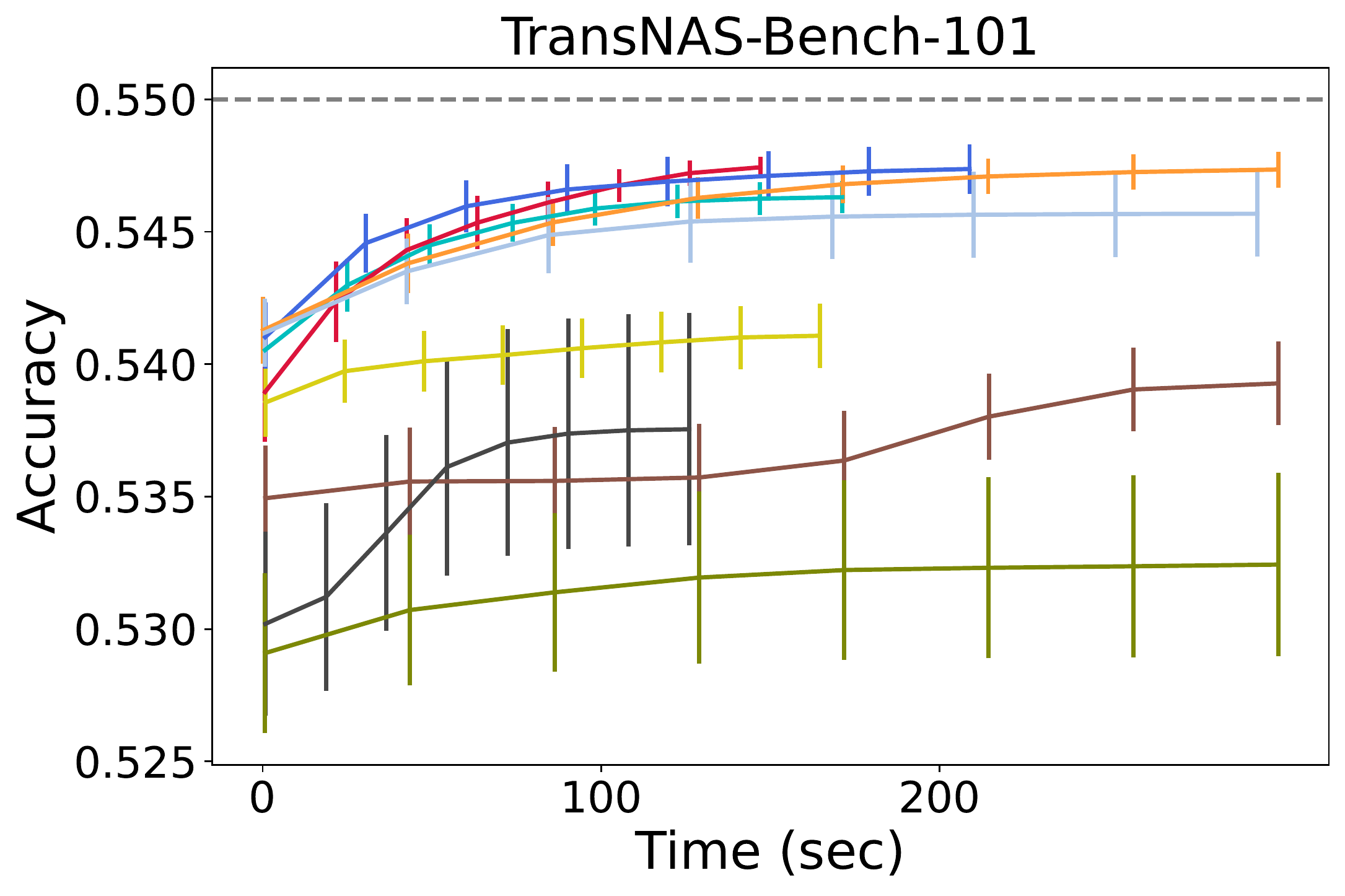}
    \includegraphics[width=0.32\textwidth]{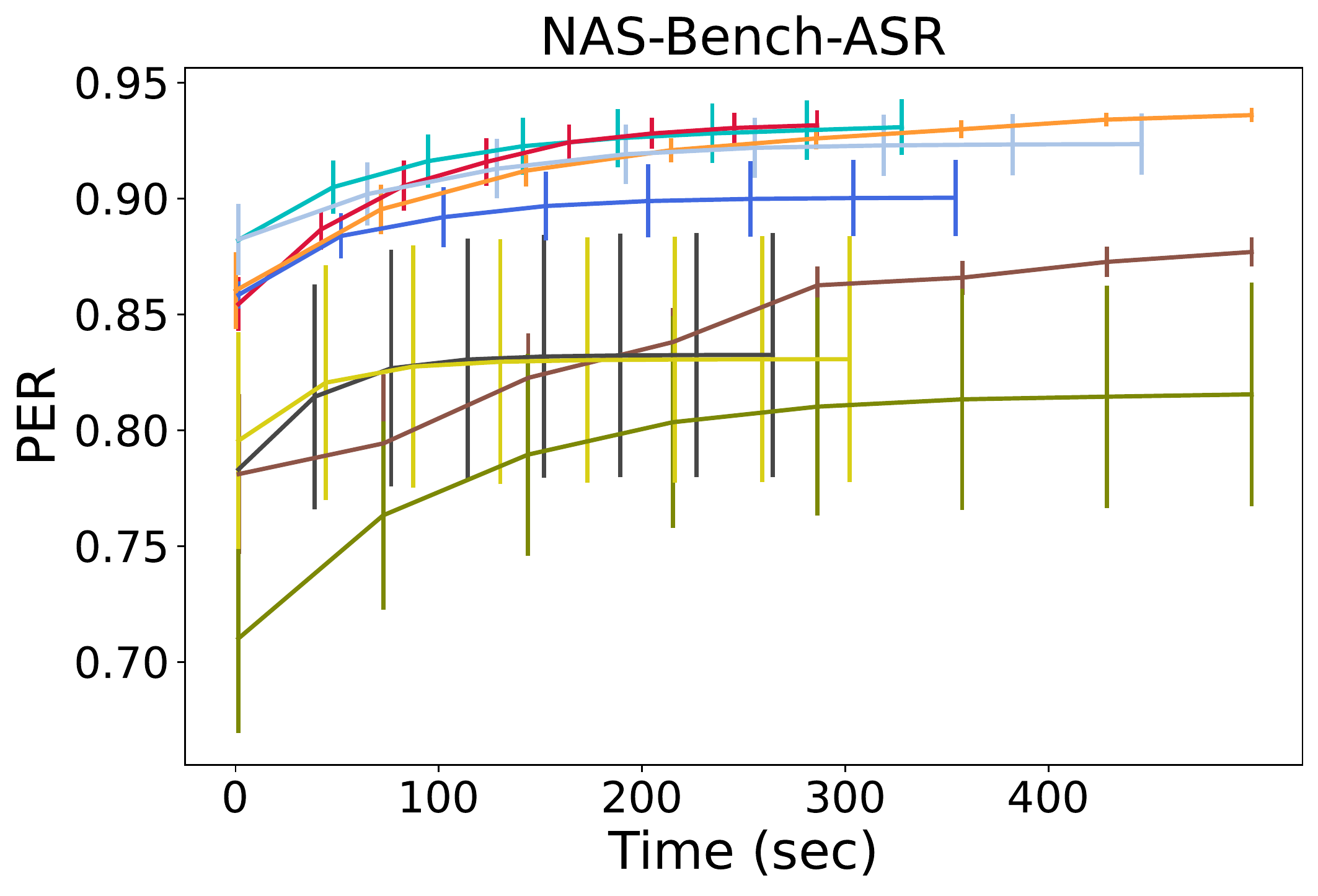}
    \caption{Performance comparison on more NAS-Bench problems.}
    \label{fig:additionalnas}
\end{figure}


\textbf{Experiments on extremely low and high dimensional problems.} We also evaluate the compared methods for extremely low and high dimensional problems by testing on Hartmann$6$\_$100$ and Hartmann$6$\_$1000$. We only run MCTS-VS, TuRBO, LA-MCTS-TuRBO and HeSBO here, because they behave well in the previous experiments. As expected, the right subfigure of Figure~\ref{fig:additionalexp} shows that MCTS-VS-BO has a clear advantage over the rest methods on the extremely high dimensional function Hartmann$6$\_$1000$. The left subfigure shows that on Hartmann$6$\_$100$, TuRBO behaves the best and MCTS-VS is the runner-up, implying that MCTS-VS can also tackle low dimensional problems to some degree.

\textbf{Experiments on synthetic functions depending on a subset of variables to various extent.} In the experiments, the synthetic functions are generated by adding unrelated variables directly. For example, Hartmann$6$\_$500$ has the dimension $D=500$, and is generated by appending $494$ unrelated dimensions to Hartmann with $6$ variables. Here, we test the performance of MCTS-VS on \rbt{a synthetic function} whose dependence on a subset of variables is more various. For this purpose, we generate Hartmann$6$\_$5$\_$500$\_$v$ by mixing five Hartmann$6$ functions as $0.5^0$Hartmann$6(\bm x_{1:6})+0.5^1\times $Hartmann$6(\bm x_{7:12})+\cdots+0.5^4$Hartmann$6(\bm x_{25:30})$, and appending 470 unrelated dimensions, where $\bm x_{i: j}$ denotes the $i$-th to $j$-th variables, and different coefficients represent various degrees of dependence. The results in Figure~\ref{fig:various_extent} show that MCTS-VS-BO performs the best.

\begin{figure}[htbp]
    \centering
    \subfigure{\includegraphics[width=0.8\textwidth]{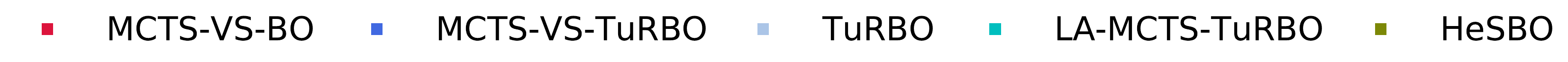}}\\\vspace{-0.5em}
    \centering
    \begin{minipage}[t]{0.66\textwidth}
    \includegraphics[width=0.49\textwidth]{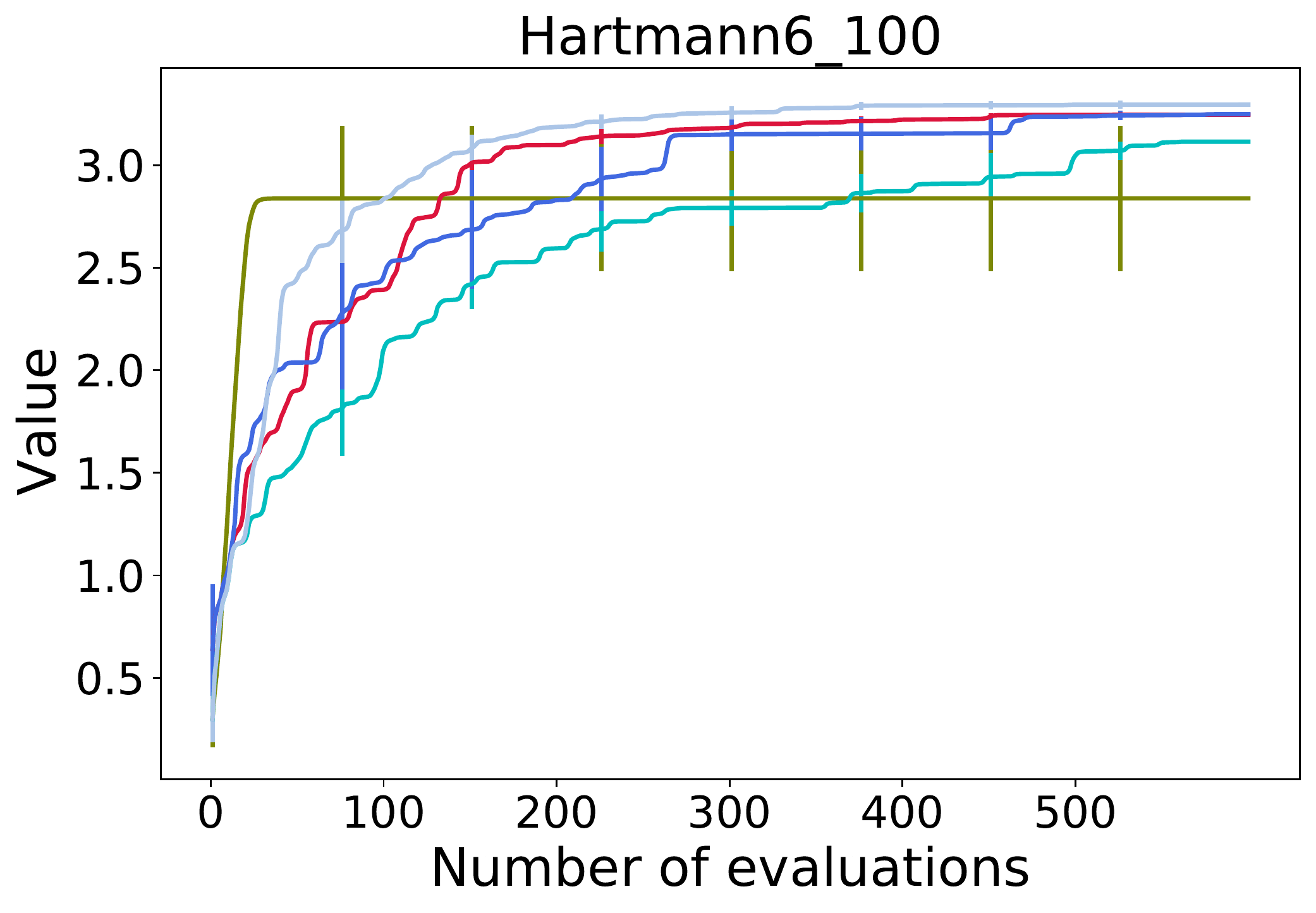}
    \includegraphics[width=0.49\textwidth]{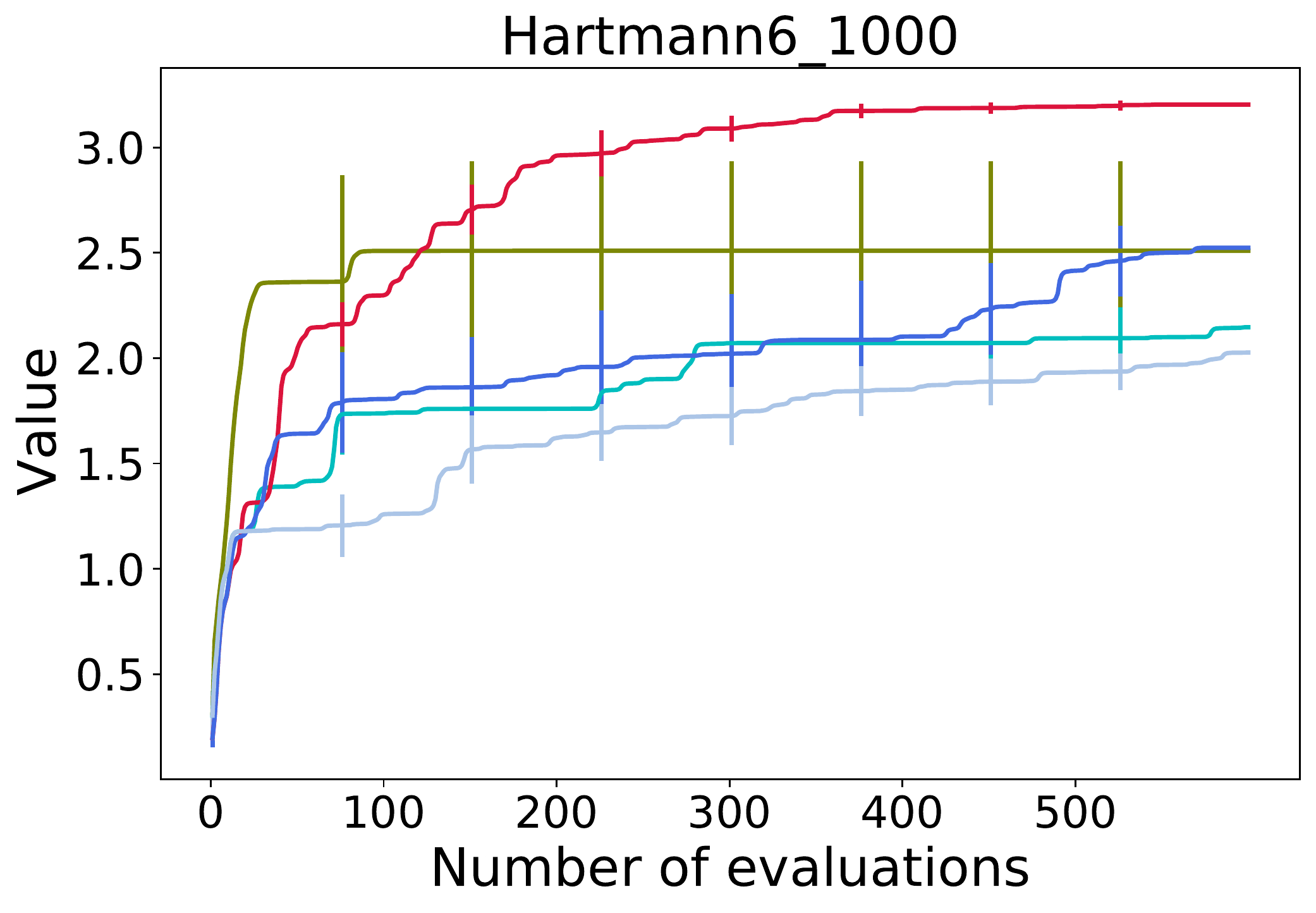}
    \vspace{-0.8em}
    \caption{Performance comparison on extremely low and high dimensional problems.}
    \label{fig:additionalexp}
    \end{minipage}
    \begin{minipage}[t]{0.33\textwidth}
    \includegraphics[width=0.95\textwidth]{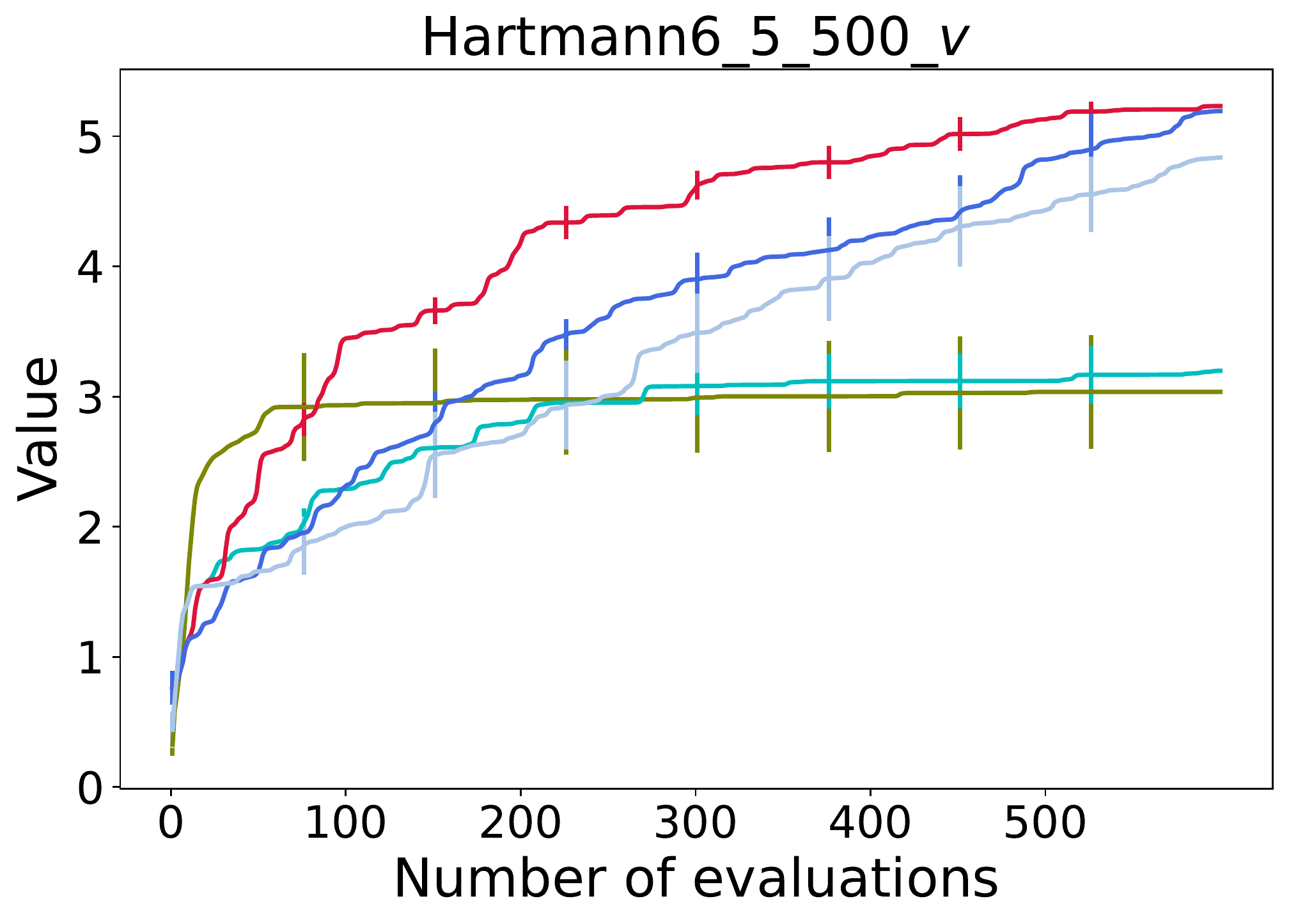}
    \vspace{-0.8em}
    \caption{Performance comparison on synthetic functions depending on a subset of variables to various extent.}
    \label{fig:various_extent}
    \end{minipage}
\end{figure}

\rbt{\textbf{Experiments with increasing ratio of valid variables.} We also examine the performance of MCTS-VS when the ratio of valid variables increases. We use the synthetic function Hartmann$6$\_$500$, and generate the variants with more valid variables by mixing multiple Hartmann$6$ functions as in Appendix~\ref{appendix:ablation}. For example, Hartmann$6$\_$5$\_$500$ is generated by mixing five Hartmann$6$ functions as Hartmann$6(\bm x_{1:6}) +$ Hartmann$6(\bm x_{7:12})+\cdots+$ Hartmann$6(\bm x_{25:30})$, and appending $470$ unrelated dimensions. We have compared MCTS-VS-TuRBO with LA-MCTS-TuRBO and TuRBO on Hartmann$6$\_$500$, Hartmann$6$\_$5$\_$500$, Hartmann$6$\_$10$\_$500$, $\dots$, Hartmann$6$\_$30$\_$500$, and Hartmann$6$\_$83$\_$500$, which has the largest number (i.e., $6\times 83=498$) of valid variables. The results are shown in Figure~\ref{fig:increasing_ratio}. It can be observed that LA-MCTS-TuRBO performs the worst. As expected, when the percentage of valid variables is low (e.g., in Hartmann$6$\_$500$, Hartmann$6$\_$5$\_$500$ and Hartmann$6$\_$10$\_$500$), MCTS-VS-TuRBO can be better than TuRBO; but as the percentage of valid variables increases, TuRBO becomes better, because a leaf node of MCTS can contain only a small fraction of valid variables. }

\begin{figure*}[htbp!]
    \centering
    \subfigure{\includegraphics[width=0.45\textwidth]{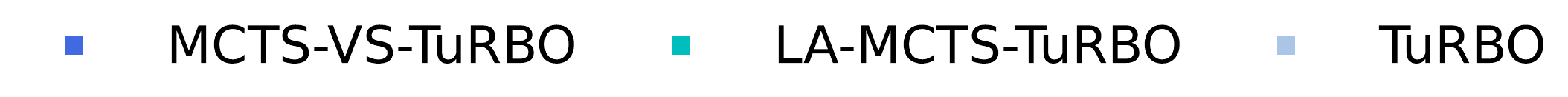}}\vspace{-1em}
    
    \centering
    \subfigure{\includegraphics[width=0.24\textwidth]{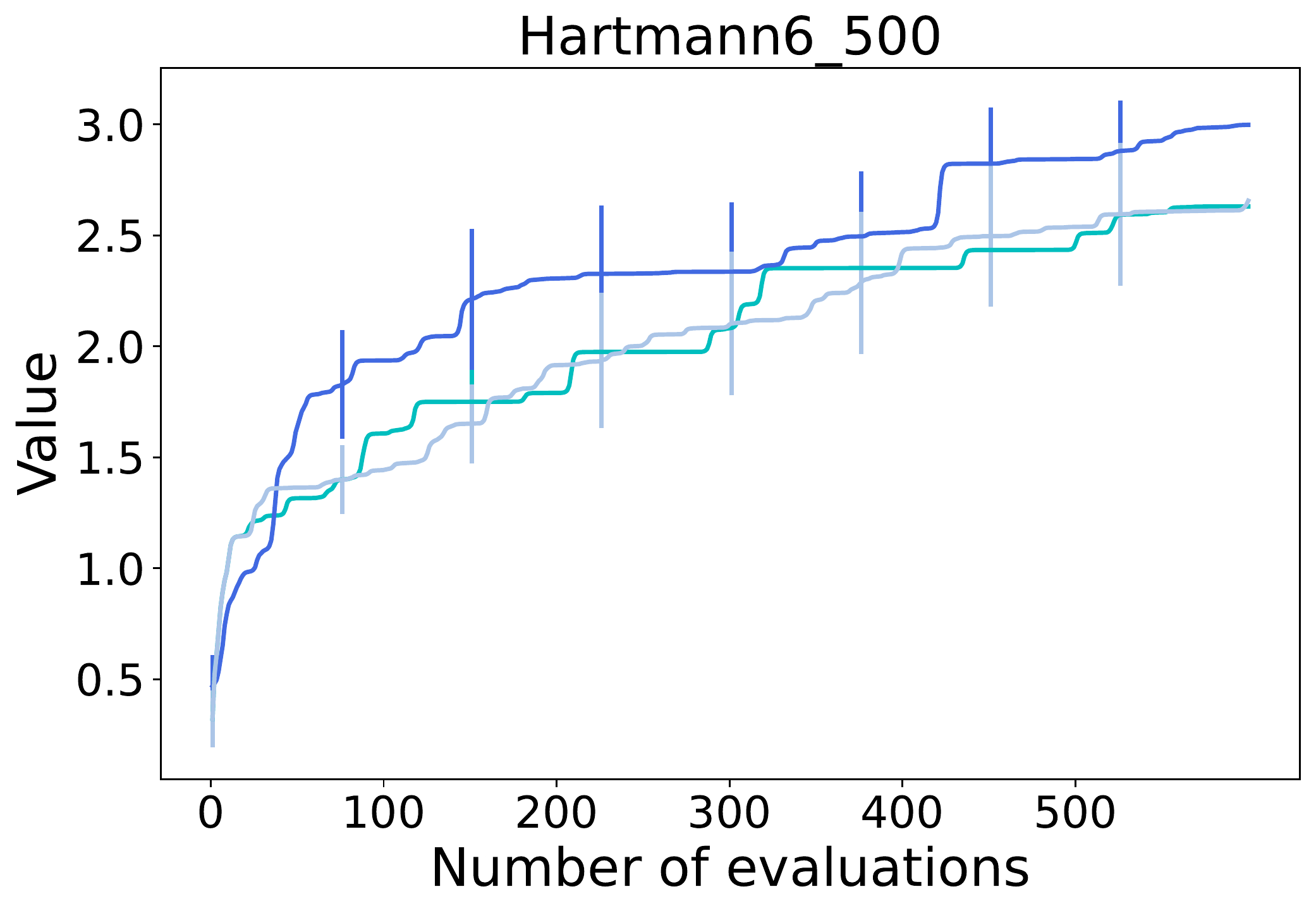}}
    \subfigure{\includegraphics[width=0.24\textwidth]{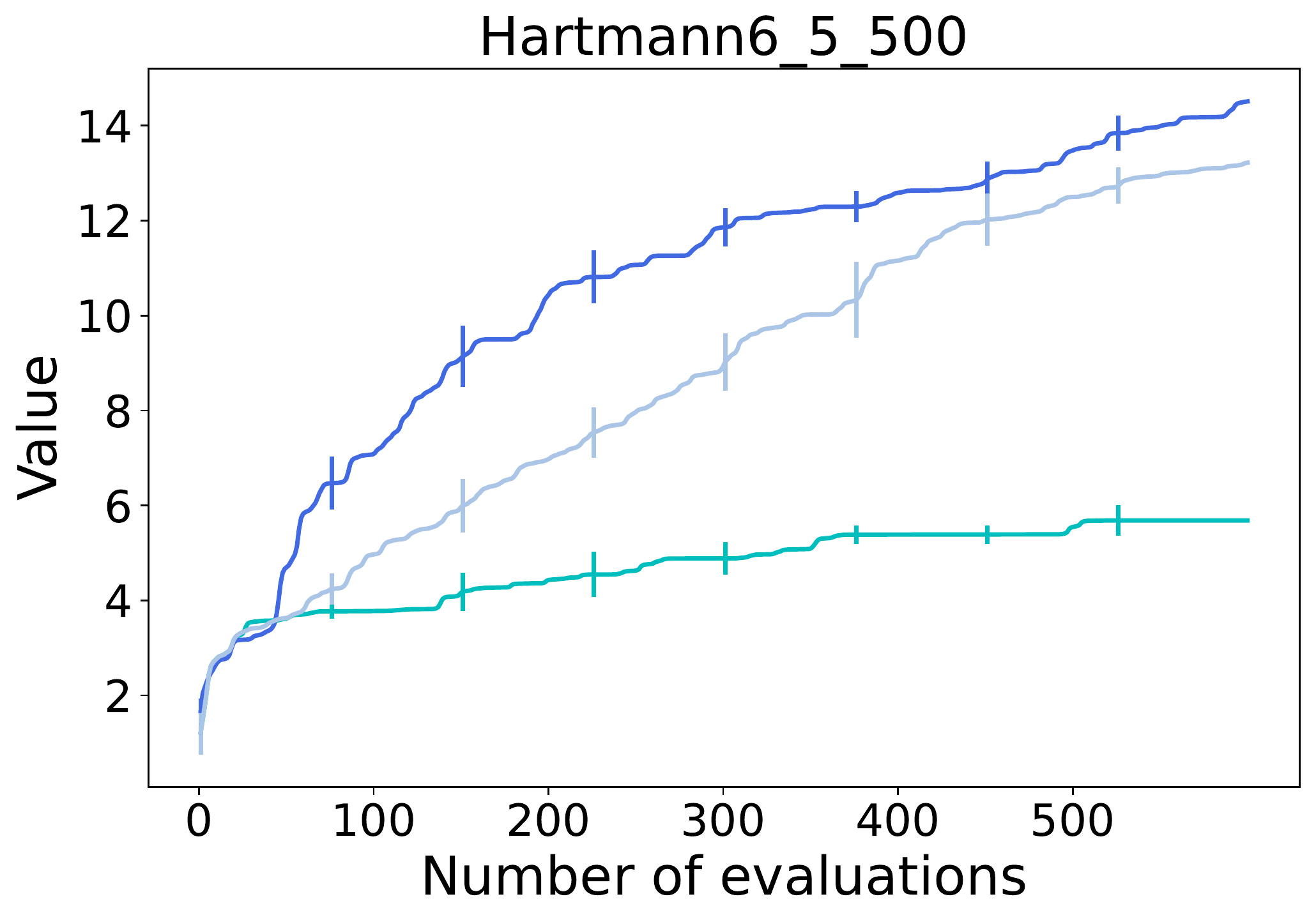}}
    \subfigure{\includegraphics[width=0.24\textwidth]{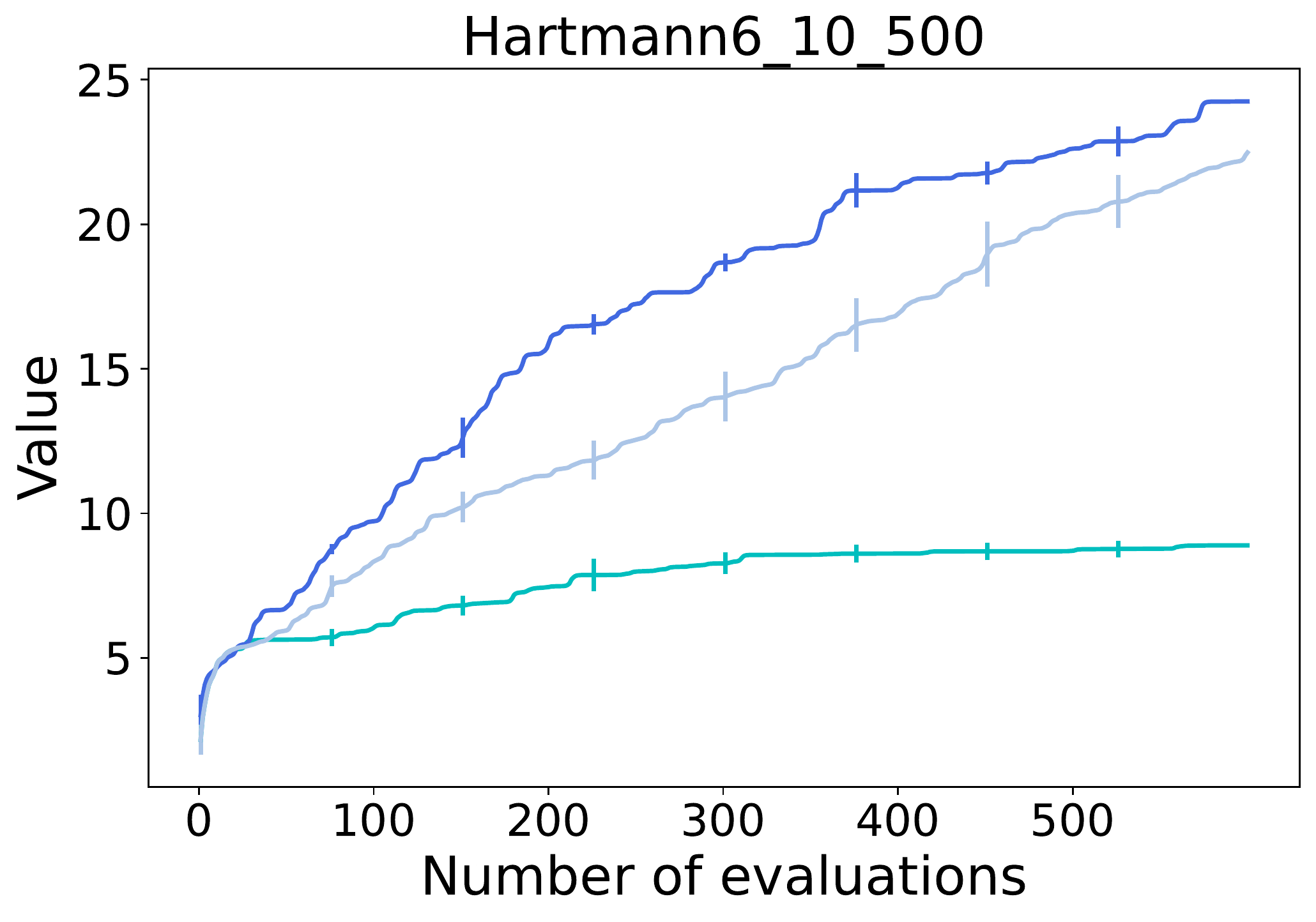}}
    \subfigure{\includegraphics[width=0.24\textwidth]{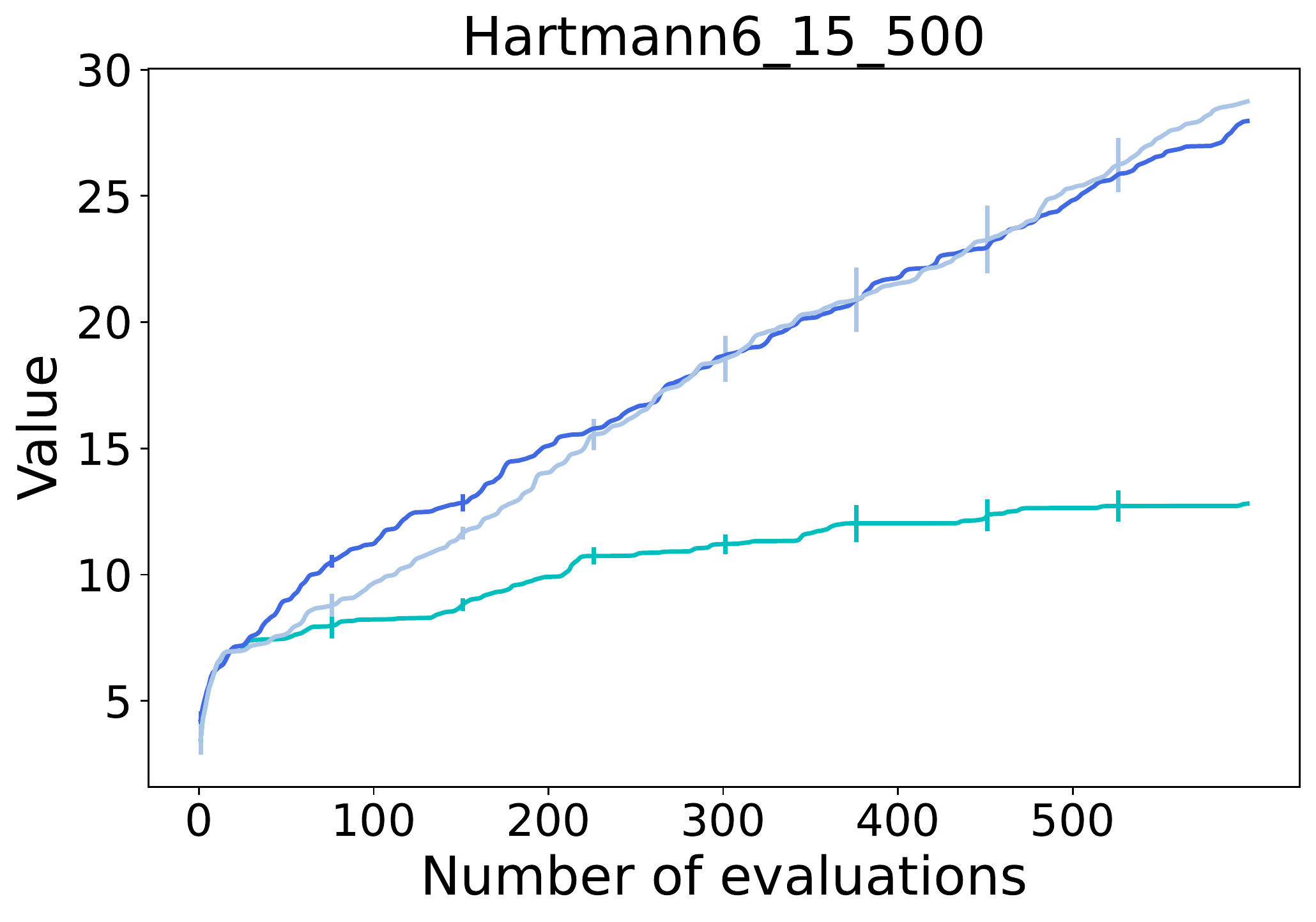}}
    \subfigure{\includegraphics[width=0.24\textwidth]{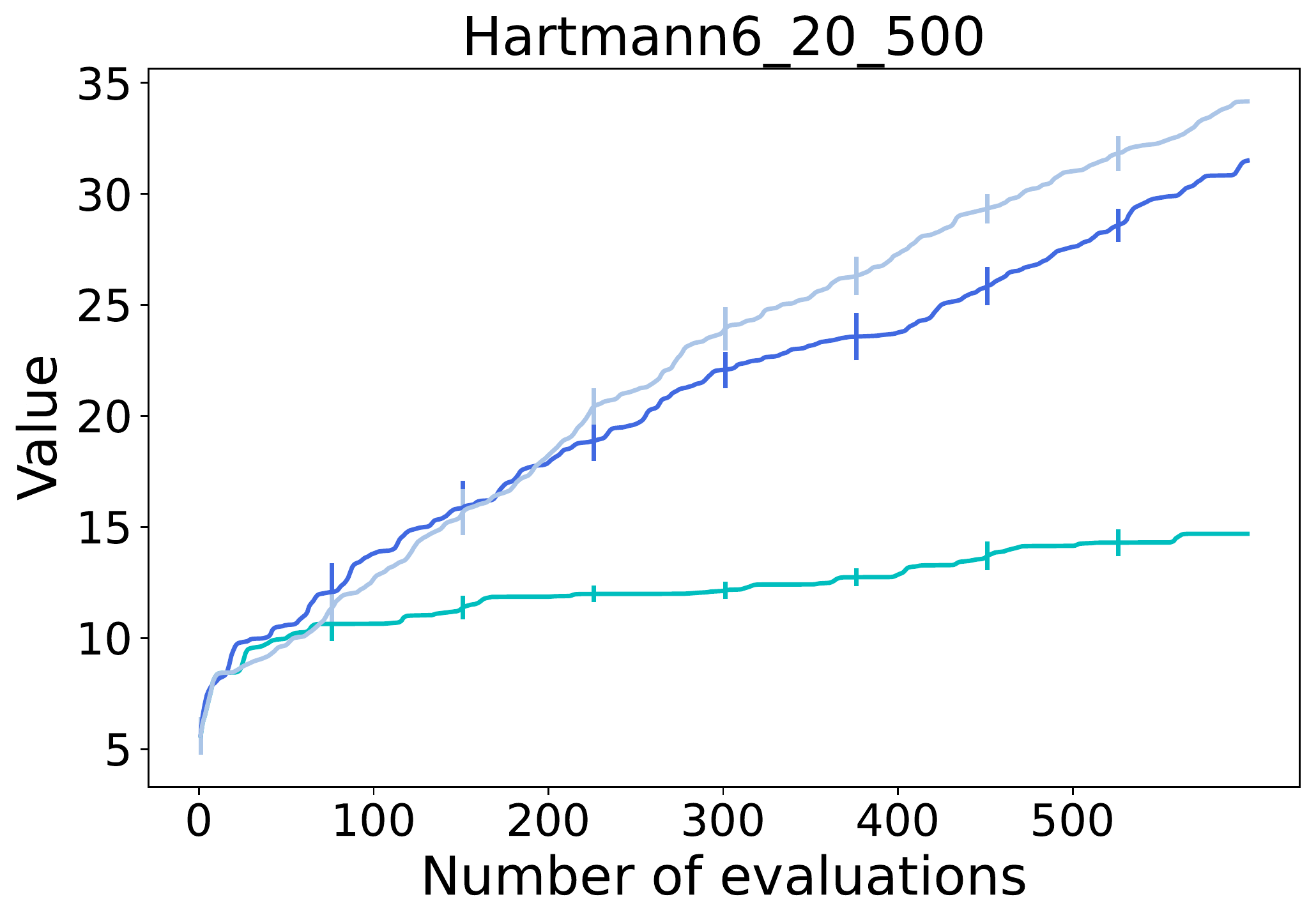}}
    \subfigure{\includegraphics[width=0.24\textwidth]{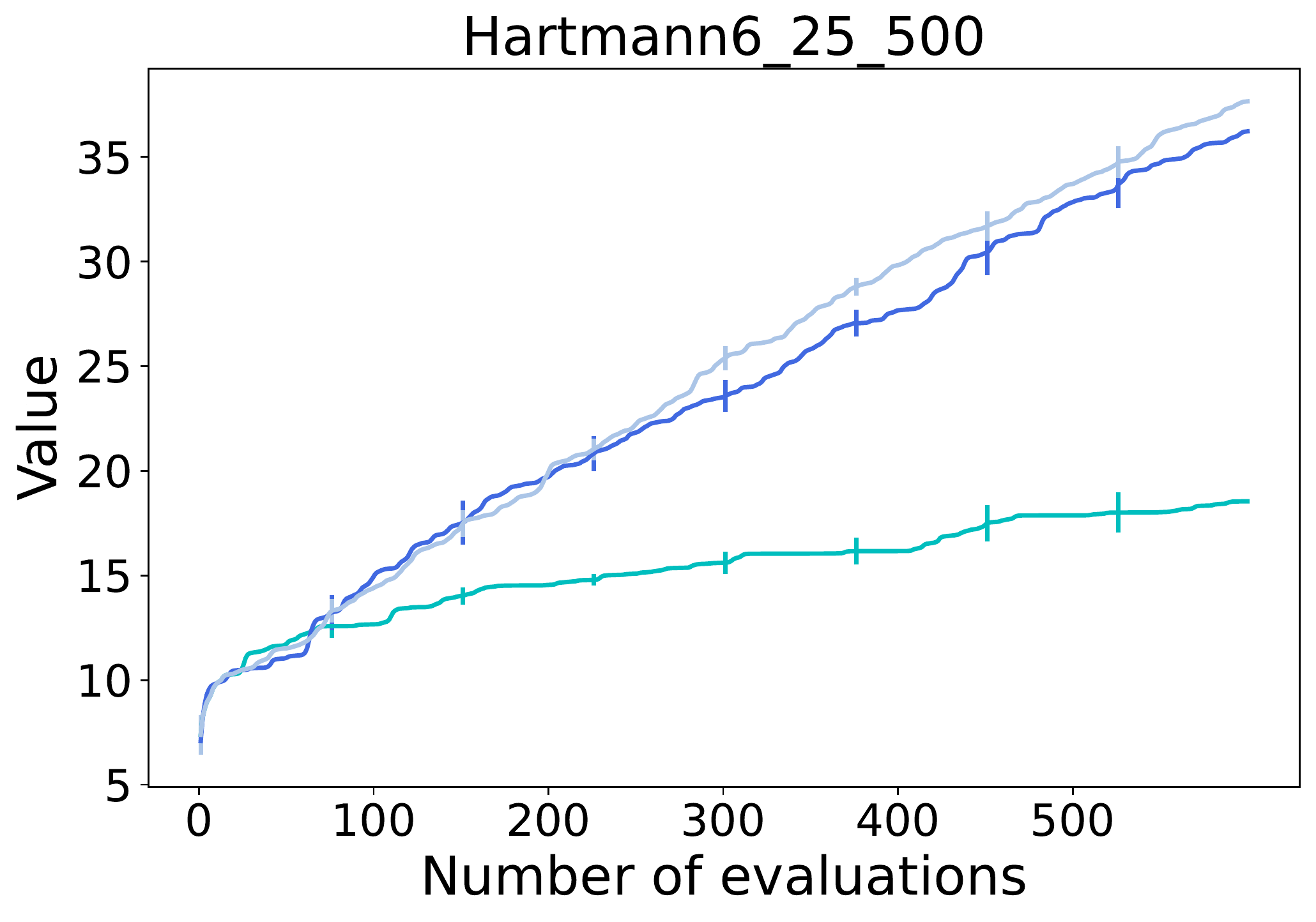}}
    \subfigure{\includegraphics[width=0.24\textwidth]{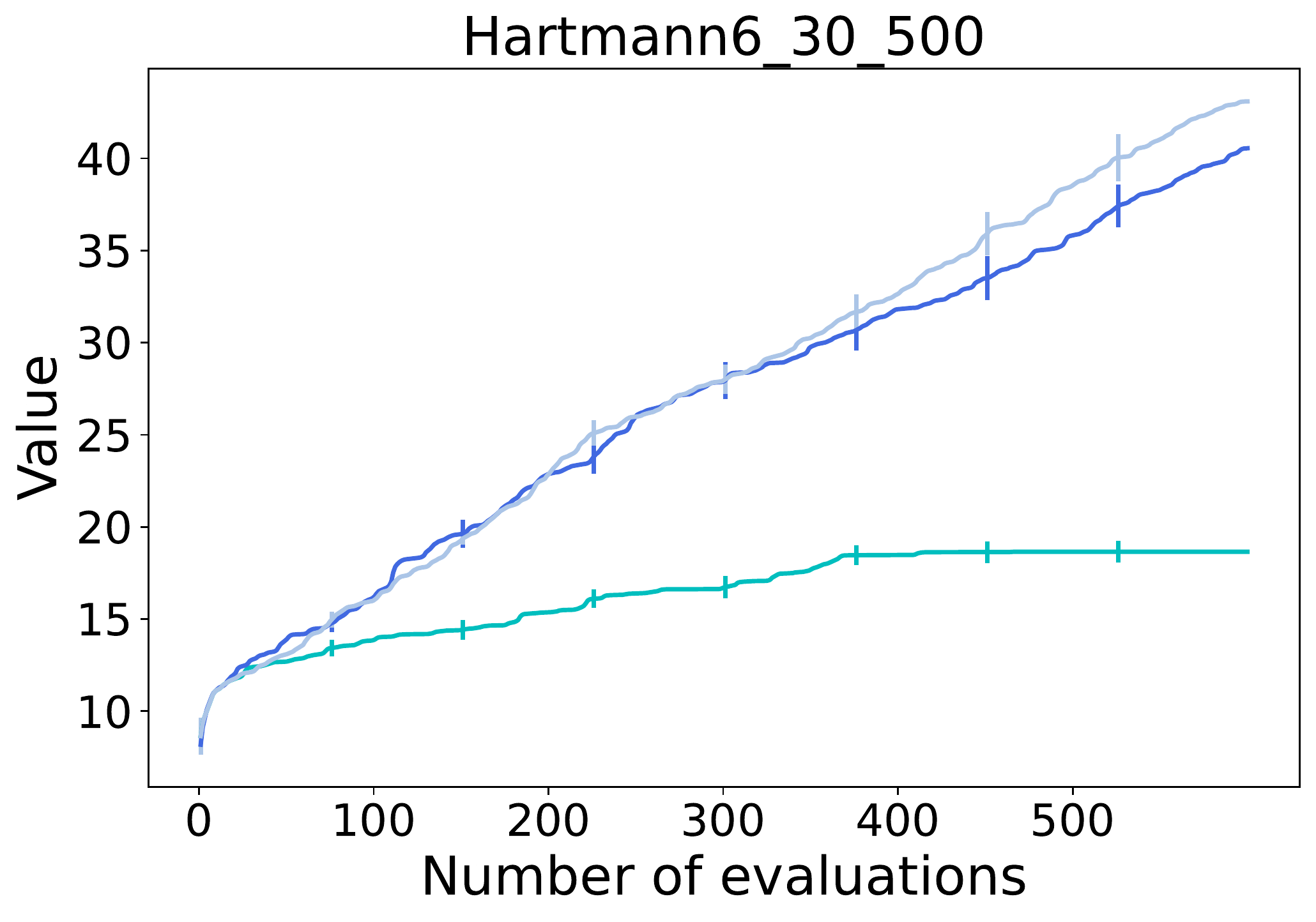}}
    \subfigure{\includegraphics[width=0.24\textwidth]{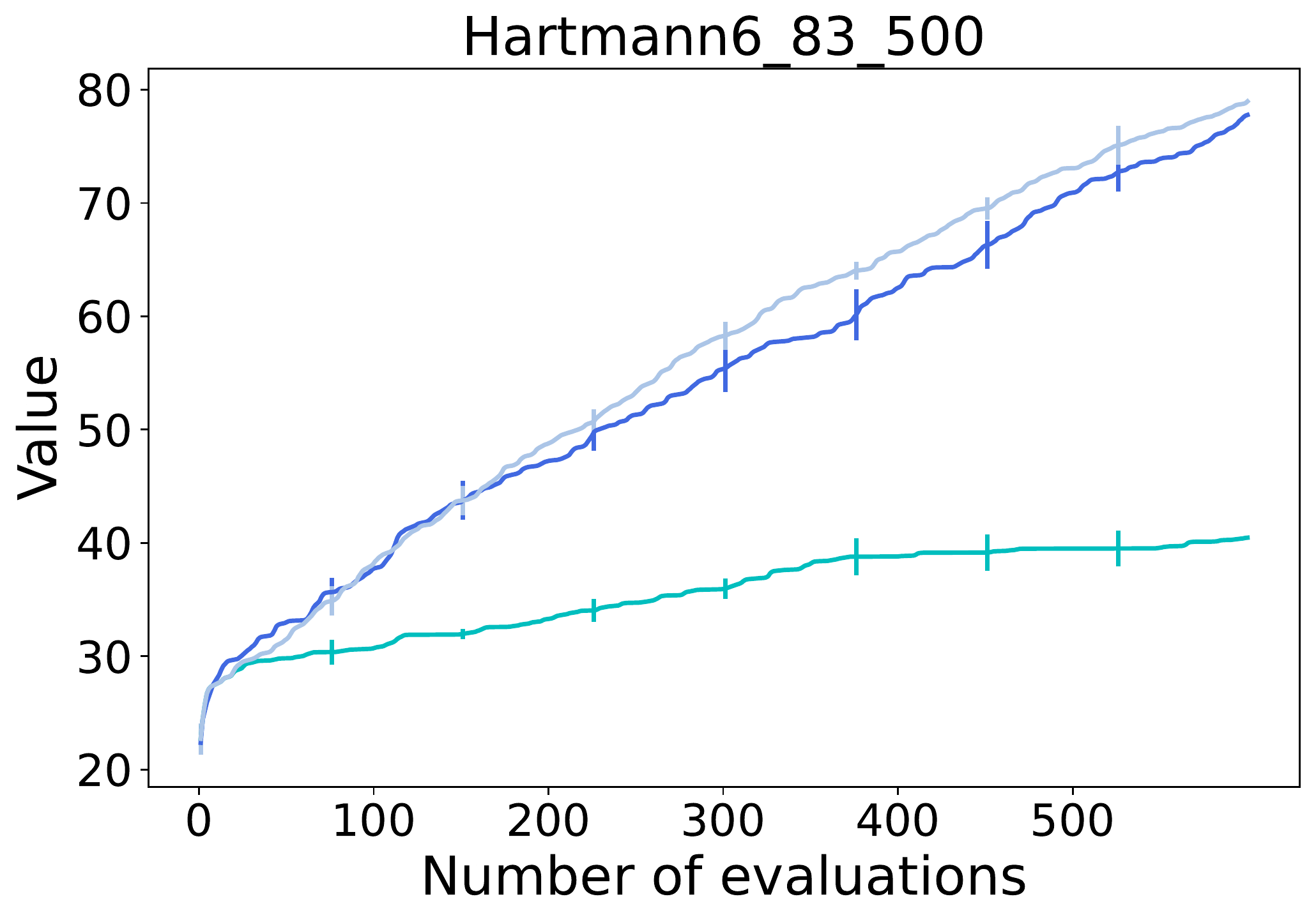}}
    \caption{\rbt{Performance comparison with increasing ratio of valid variables.}}
    \label{fig:increasing_ratio}
\end{figure*}

\songl{\textbf{Hierarchical variable selection for extremely high-dimensional problems.} We attempt to combine MCTS-VS and SAASBO (i.e., MCTS-VS-SAASBO) to handle extremely high-dimensional problems. MCTS-VS-SAASBO can be viewed as a hierarchical variable selection method, i.e., MCTS-VS first performs an efficient but rough variable selection to select some variables, and then SAASBO performs a time-consuming but precise variable selection under the relative low-dimensional space, to further select the important variables. We run MCTS-VS-SAASBO and SAASBO on Hartmann$6$\_$500$. The results are shown in Figure~\ref{fig:hierarchy-vs}. The performance of MCTS-VS-SAASBO and SAASBO is similar. But when considering the runtime, the time of $200$ iterations of MCTS-VS-SAASBO is about $6000$s, while the time of SAASBO is about $45000$s. That is, MCTS-VS-SAASBO can achieve more than $7$ times acceleration. The curves of using the wall clock time as the $x$-axis in the right sub-figure of Figure~\ref{fig:hierarchy-vs} clearly show the advantage of MCTS-VS-SAASBO over SAASBO. MCTS-VS-SAASBO selects the variables containing important ones by MCTS and then uses SAASBO to optimize the selected variables, which reduces the dimension and thus costs much less time than using SAASBO directly. The combination of MCTS-VS and SAASBO may be a potential solution for BO to handle extremely high-dimensional optimization problems, where it is difficult to select important variables directly.}

\begin{figure*}[htbp!]
    \centering
    \subfigure{\includegraphics[width=0.4\textwidth]{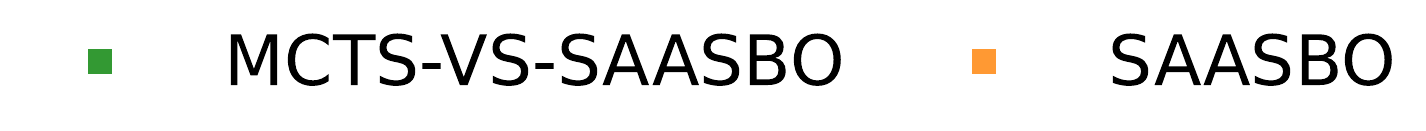}}\vspace{-1em} \\
    \centering
    \subfigure{\includegraphics[width=0.4\textwidth]{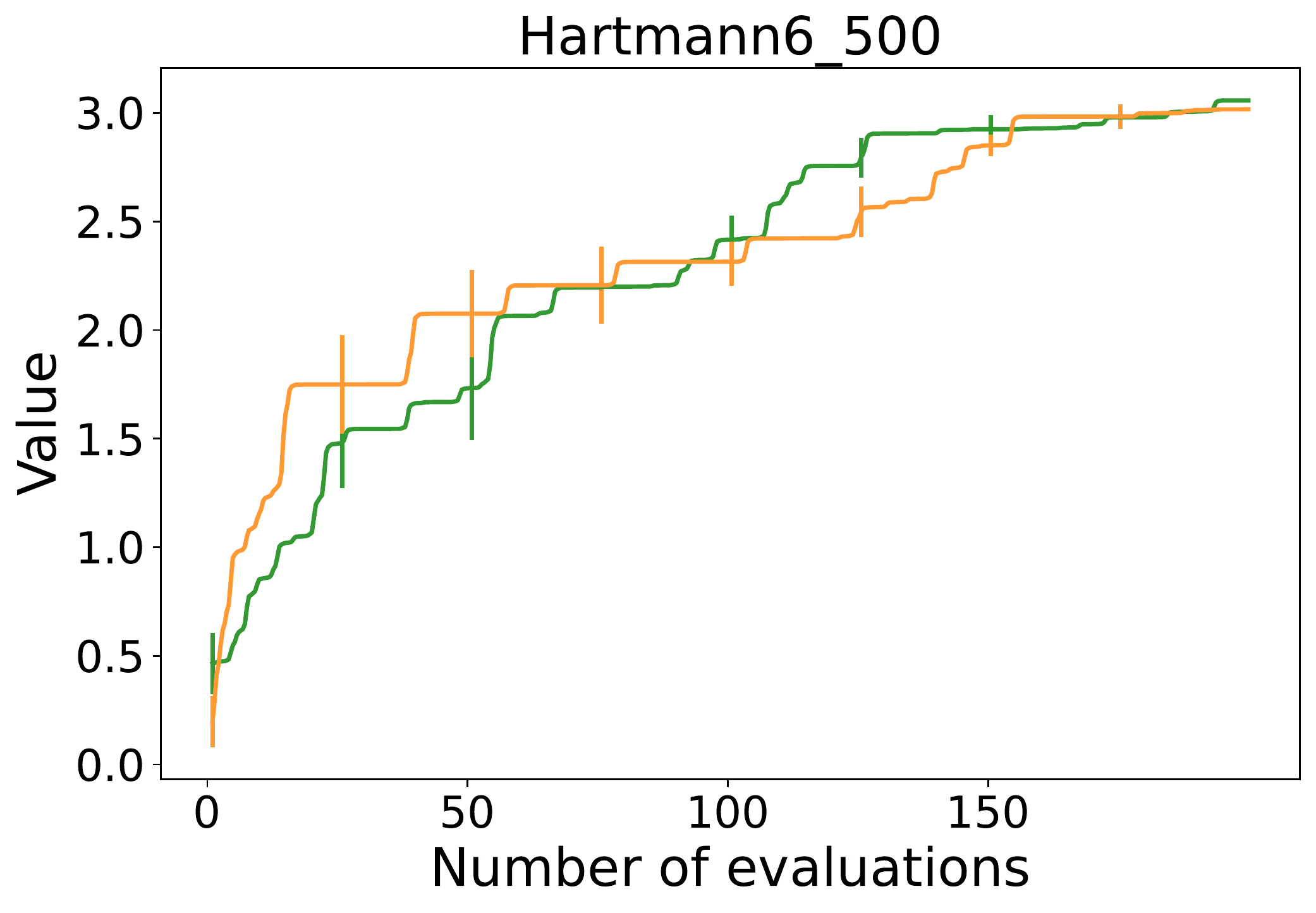}}
    \subfigure{\includegraphics[width=0.4\textwidth]{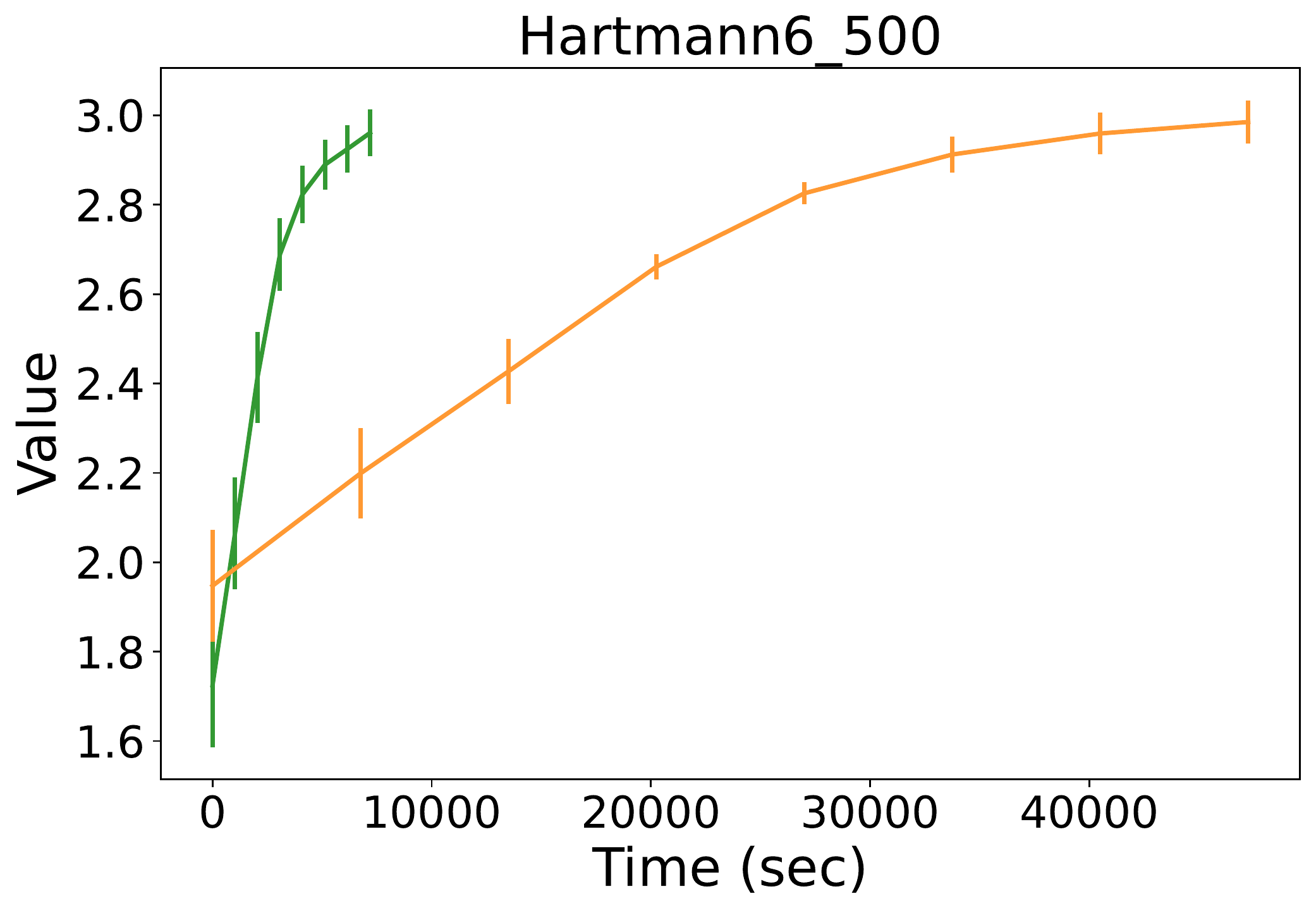}}
    \caption{Performance comparison among MCTS-VS-SAASBO and SAASBO on the synthetic function Hartmann$6$\_$500$.}
    \label{fig:hierarchy-vs}
\end{figure*}

\rbt{\textbf{Comparison with LASSO-VS.} There are other variable selection methods (e.g., LASSO), which are not designed for high dimensional BO but can be used directly. We have implemented the LASSO-based variable selection method, named LASSO-VS. We compare MCTS-VS, LASSO-VS and Dropout on the synthetic function Hartmann$6$\_$300$. When using LASSO-VS, the $d$ variables with the largest absolute values of the regression coefficients are selected at each iteration. The results are shown in Figure~\ref{fig:lasso}. When equipped with either BO or TuRBO, the proposed MCTS-VS always performs the best. We can also observe that when equipped with BO, LASSO-VS can even be worse than Dropout. This may be because many of existing variable selection methods (e.g., LASSO) usually require a large number of samples to fit the linear regression model well, while in BO scenarios, only a limited number of samples can be evaluated.}

\begin{figure*}[htbp!]
    \centering
    \subfigure{\includegraphics[width=0.8\textwidth]{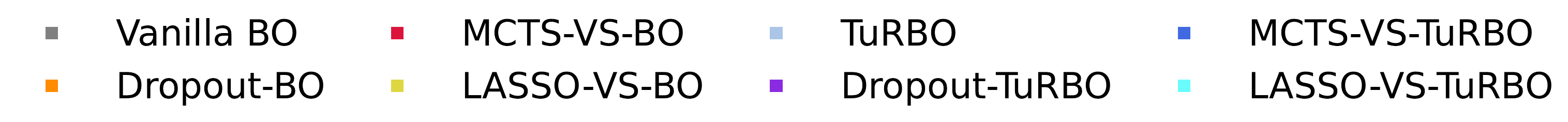}}\vspace{-1em}
    \centering
    \subfigure{\includegraphics[width=0.4\textwidth]{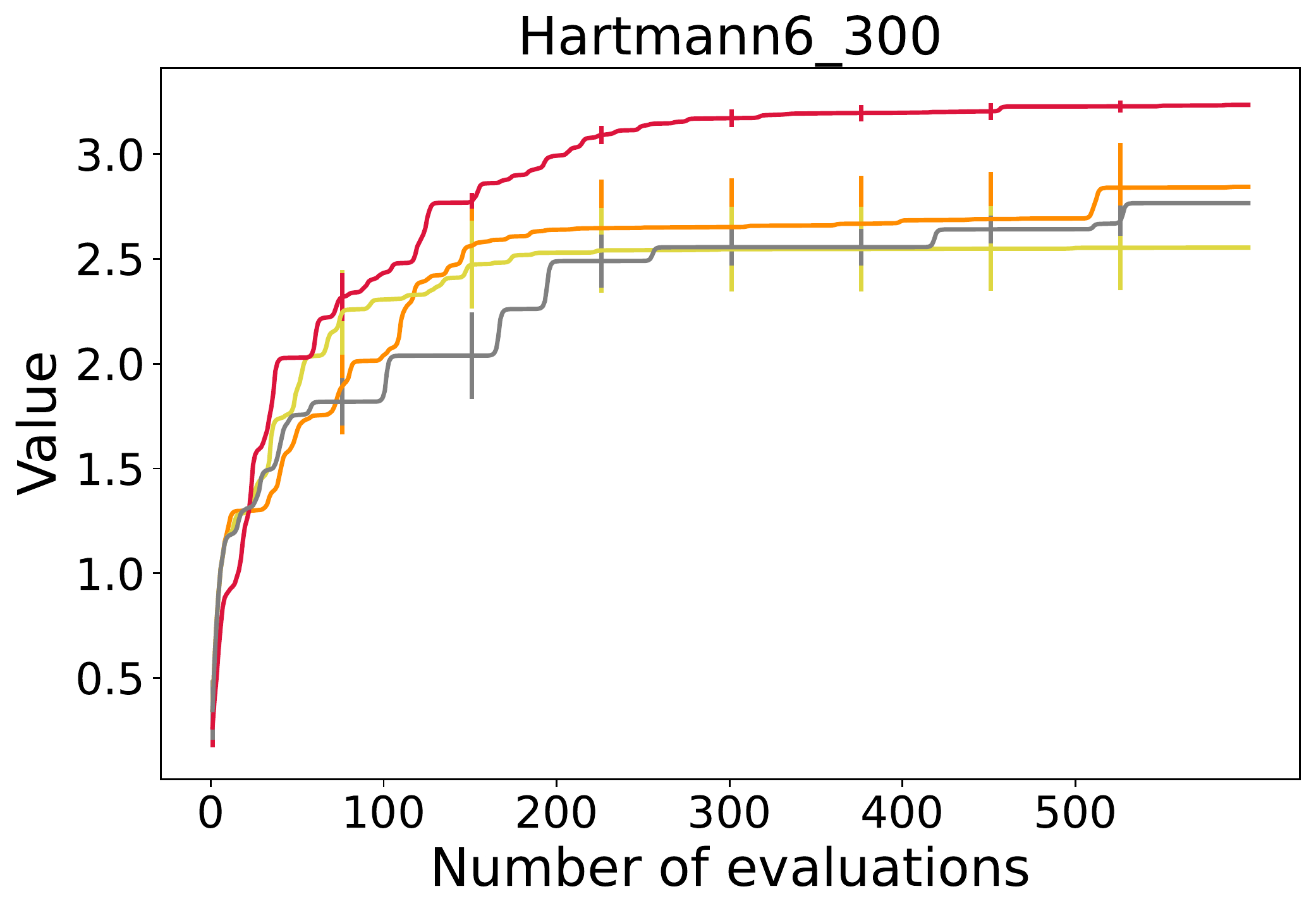}}
    \subfigure{\includegraphics[width=0.4\textwidth]{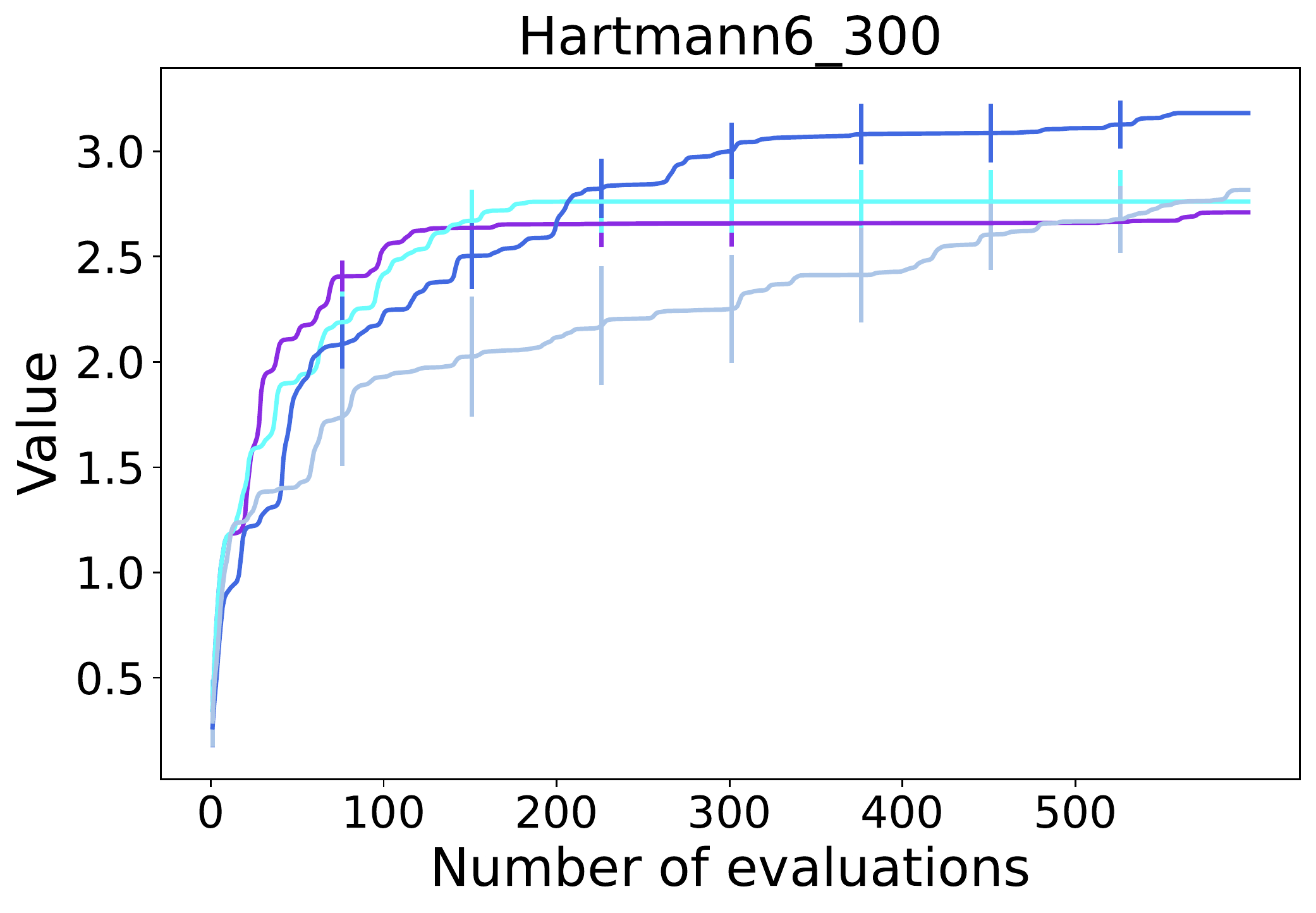}}
    \caption{\rbt{Performance comparison among MCTS-VS, LASSO-VS and Dropout on the synthetic function Hartmann$6$\_$300$.}}
    \label{fig:lasso}
\end{figure*}

\textbf{Statistical tests.} As most of the previous works, we have conducted experiments using 5 random seeds (2021--2025). Here, we also conduct statistical \rbt{tests} on Hartmann and Levy functions by running the methods for 50 times (random seeds 2021--2070), to make a more confident comparison. Considering the performance and runtime of the methods we have observed in Figure~\ref{fig:synthetic2}, we only compare MCTS-VS with LA-MCTS-TuRBO and TuRBO, which achieve good performance in acceptable time. The results are shown in Table~\ref{tab:significance_test}. MCTS-VS-BO achieves the best average objective value on all the synthetic functions except Levy$10$\_$100$ where the dimension is relatively low and TuRBO performs the best. By the Wilcoxon signed-rank test with confidence level 0.05, MCTS-VS-TuRBO is significantly better than LA-MCTS-TuRBO on all the synthetic functions, showing the advantage of MCTS-VS over LA-MCTS for variable selection. Compared with TuRBO, MCTS-VS-TuRBO is only significantly better on Hartmann functions, which may be because the ratio of valid variables of Hartmann$6$\_$300$ and Hartmann$6$\_$500$ is lower than that of Levy$10$\_$100$ and Levy$10$\_$300$, and thus the advantage of performing variable selection by MCTS-VS is more clear. Note that the observations about the performance rank of the compared methods are consistent with that observed in Figure~\ref{fig:synthetic2}, which plot the results of the compared methods by running five times.  

\begin{table}[htbp!]\scriptsize
\centering
\caption{Objective values obtained by MCTS-VS-BO, MCTS-VS-TuRBO, LA-MCTS-TuRBO and TuRBO on synthetic functions. Each result consists of the mean and standard deviation of 50 runs. The best mean value on each problem is bolded. The symbols `$+$', `$-$' and `$\approx$' indicate that MCTS-VS-TuRBO is significantly superior to, inferior to, and almost equivalent to the corresponding method, respectively, according to the Wilcoxon signed-rank test with confidence level 0.05.}\vspace{0.5em}
\resizebox{\linewidth}{!}{
\begin{tabular}{c|c c c c}
\toprule
Problem  & MCTS-VS-BO  & MCTS-VS-TuRBO & LA-MCTS-TuRBO & TuRBO \\\midrule
Levy10\_100 & -2.620(1.757) $+$ & -1.102(1.711) & -2.444(1.708) $+$ & \textbf{-0.662}(1.049) $-$ \\
Levy10\_300 & \textbf{-1.506}(0.854) $\approx$ & -1.765(1.811) & -6.218(3.389) $+$ & -1.855(2.038) $\approx$ \\
Hartmann6\_300 & \textbf{3.223}(0.074) $\approx$ & 3.153(0.264) & 2.892(1.147) $+$ & 2.857(0.475) $+$ \\
Hartmann6\_500 & \textbf{3.200}(0.091) $-$ & 3.012(0.434) & 2.619(0.672) $+$ & 2.629(0.672) $+$ \\
\midrule $+$/$-$/$\approx$ & 1/1/2 & /  & 4/0/0  & 2/1/1 \\
\bottomrule
\end{tabular}}
\label{tab:significance_test}
\end{table}

\section{\rbt{Enlargement of Some Figures in the Main Paper}}

\rbt{Due to space limitation, Figures~\ref{fig:synthetic1} and~\ref{fig:synthetic2} in the main paper are a little small. Here, we also provide their enlarged versions, i.e., Figures~\ref{fig:enlargement_synthetic1} and~\ref{fig:enlargement_synthetic2}. }

\begin{figure*}[htbp!]
    \centering
    \hspace{-1em}\subfigure{\includegraphics[width=0.7\textwidth]{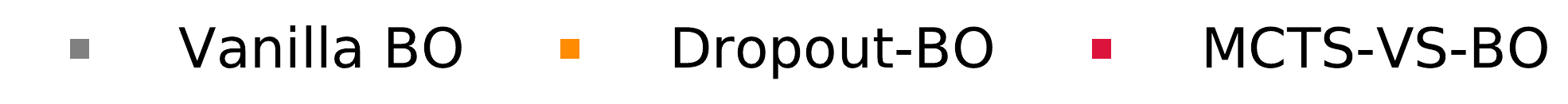}}\\\vspace{-1em}
    \centering
    \subfigure{\includegraphics[width=0.4\textwidth]{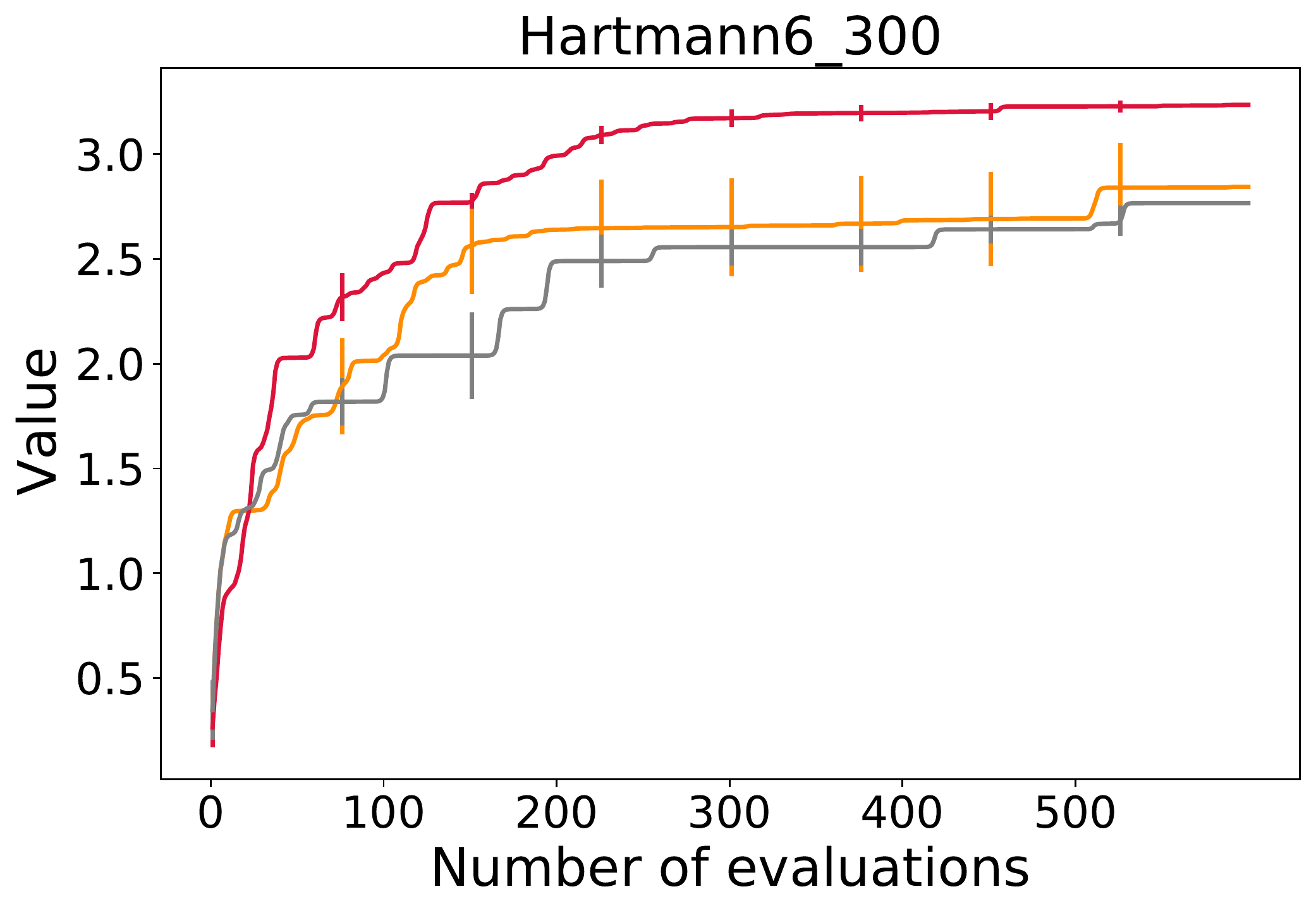}}
    \subfigure{\includegraphics[width=0.4\textwidth]{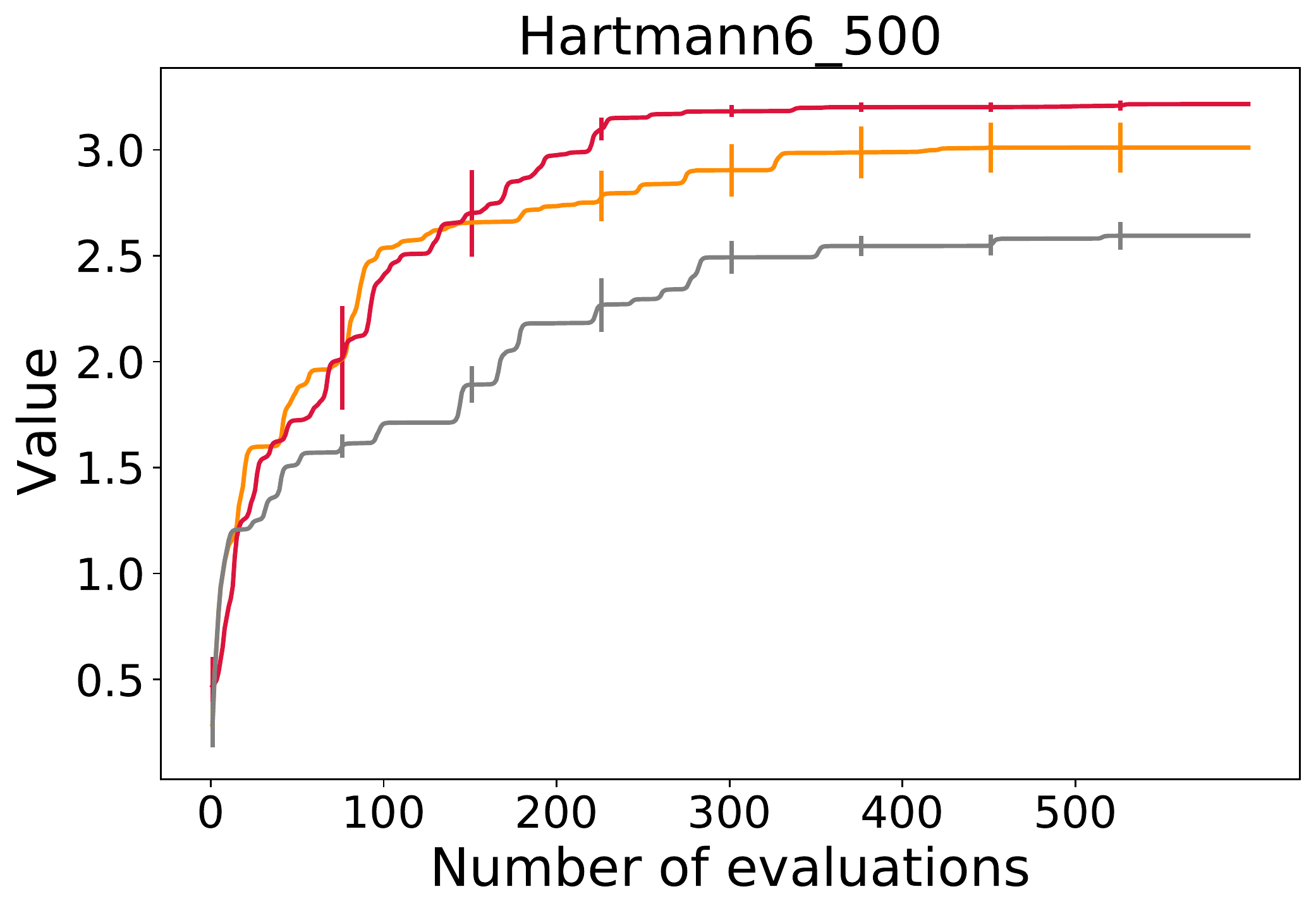}}
    \centering
    \subfigure{\includegraphics[width=0.7\textwidth]{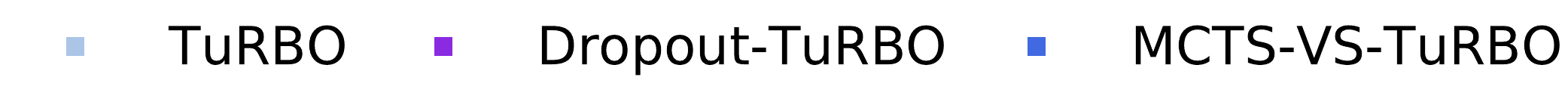}}\\\vspace{-1em}
    \subfigure{\includegraphics[width=0.4\textwidth]{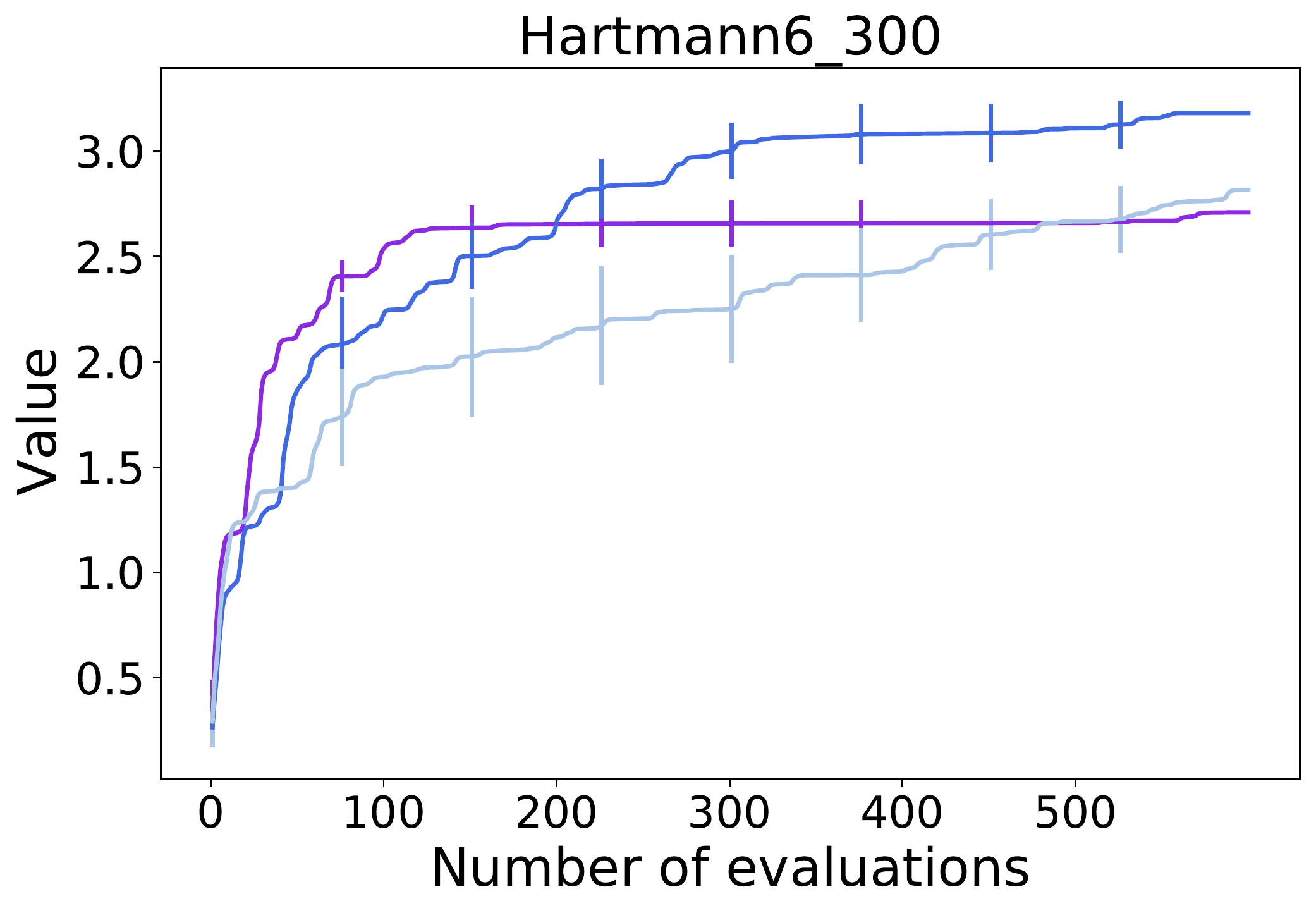}}
    \subfigure{\includegraphics[width=0.4\textwidth]{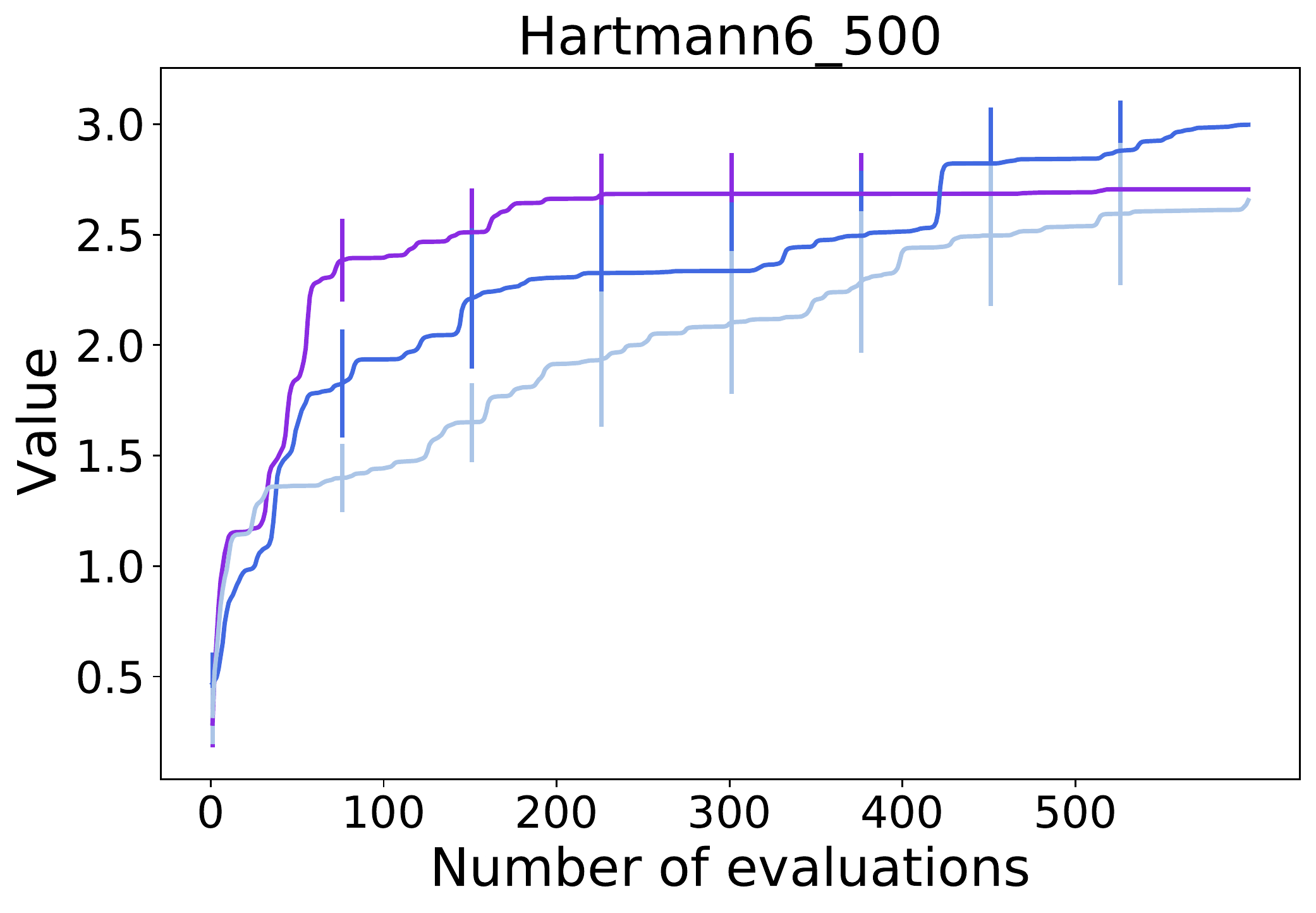}}\vspace{-0.8em}
    \caption{\rbt{Performance comparison among the two variable selection methods (i.e., MCTS-VS and Dropout) and the BO methods (i.e., Vanilla BO and TuRBO) on two synthetic functions.}}
    \label{fig:enlargement_synthetic1}\vspace{-0.5em}
\end{figure*}

\begin{figure*}[h!]
    \centering
    \subfigure{\includegraphics[width=0.9\textwidth]{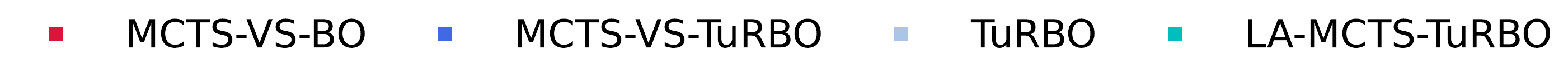}}\\
    \vspace{-1.3em}
    \subfigure{\includegraphics[width=0.8\textwidth]{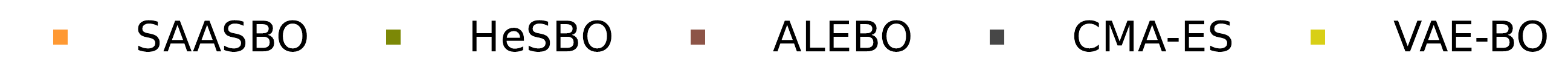}}\\
    \vspace{-1em}
    \centering
    \subfigure{\includegraphics[width=0.4\textwidth]{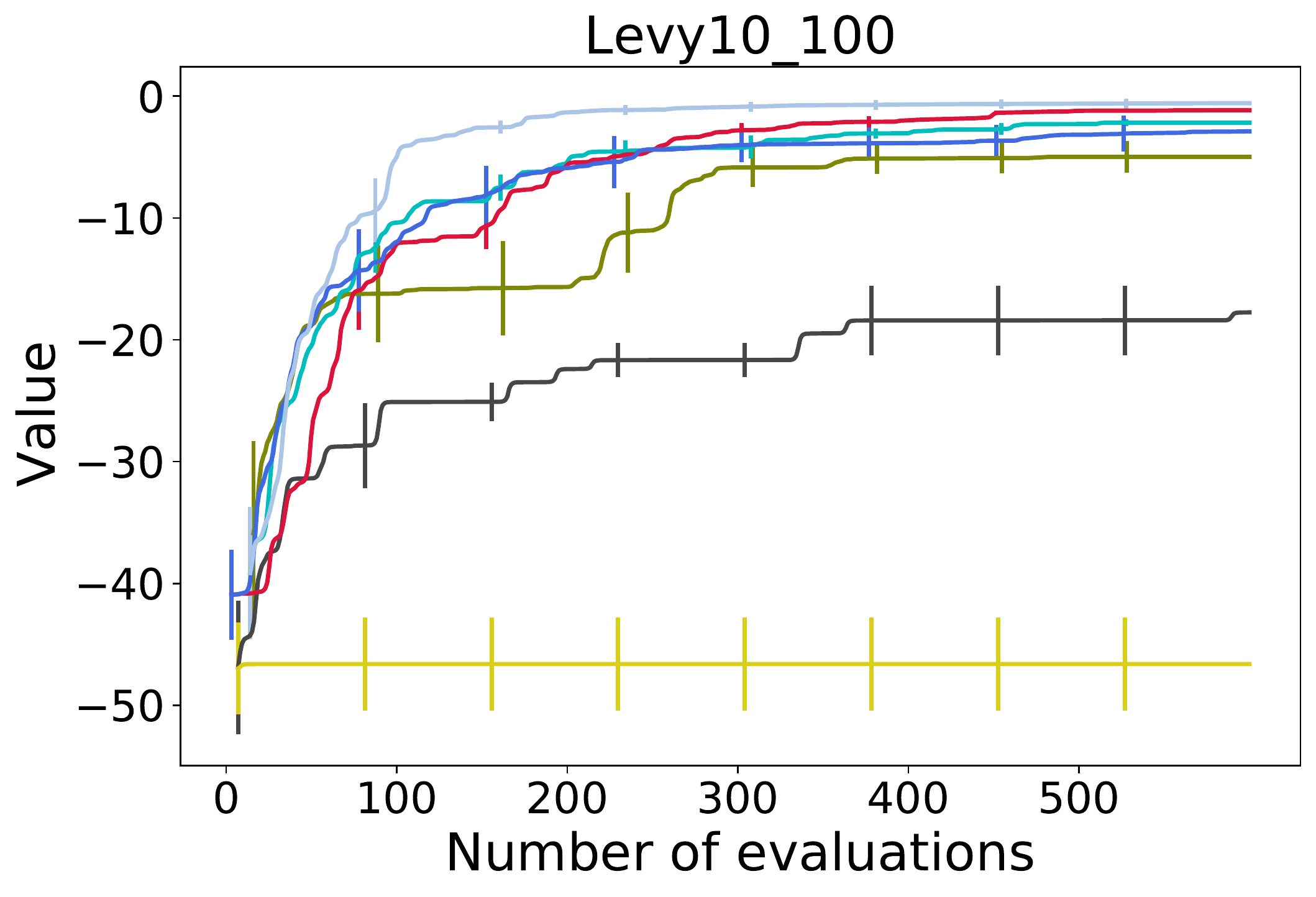}}
    \subfigure{\includegraphics[width=0.4\textwidth]{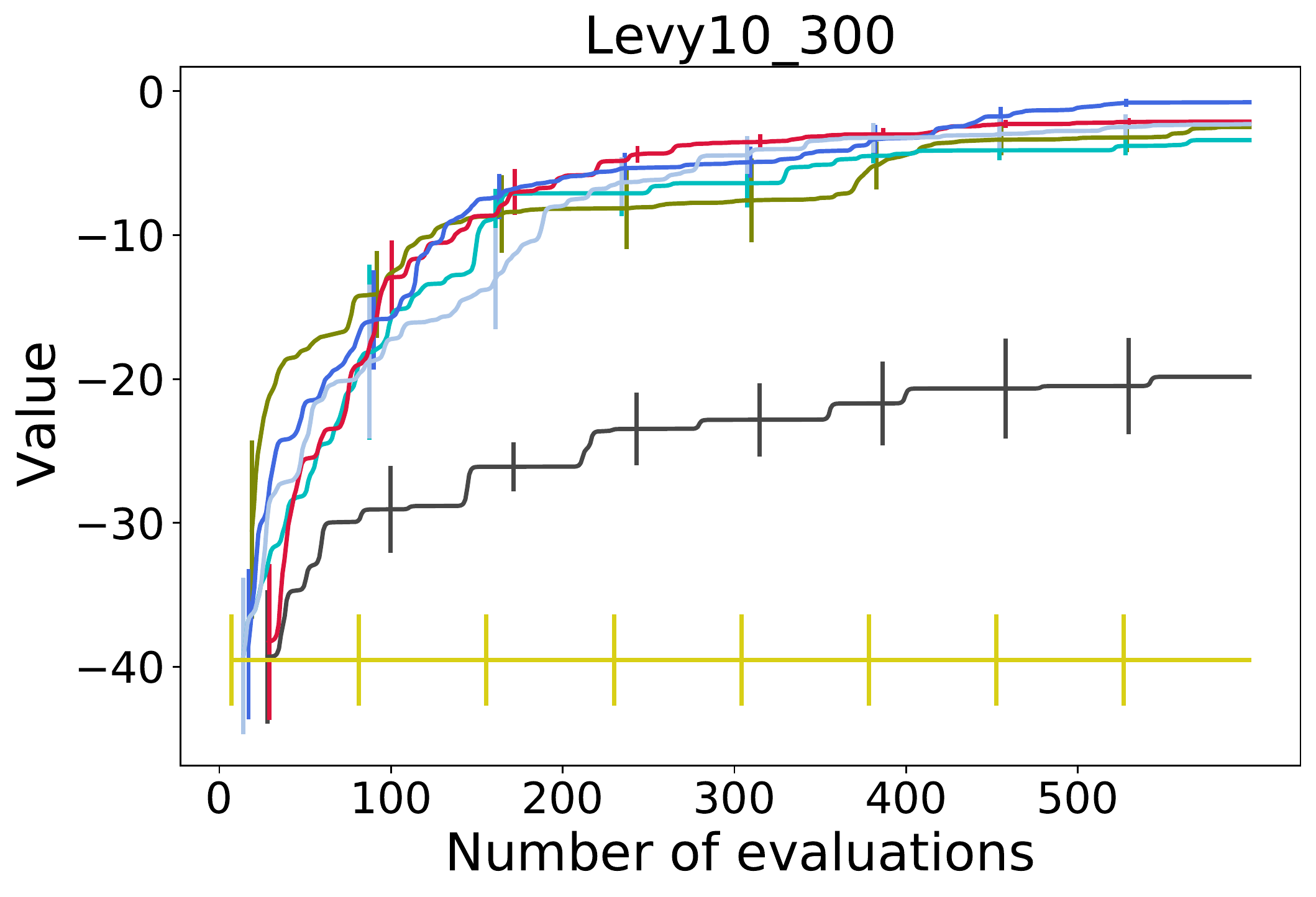}}
    \subfigure{\includegraphics[width=0.4\textwidth]{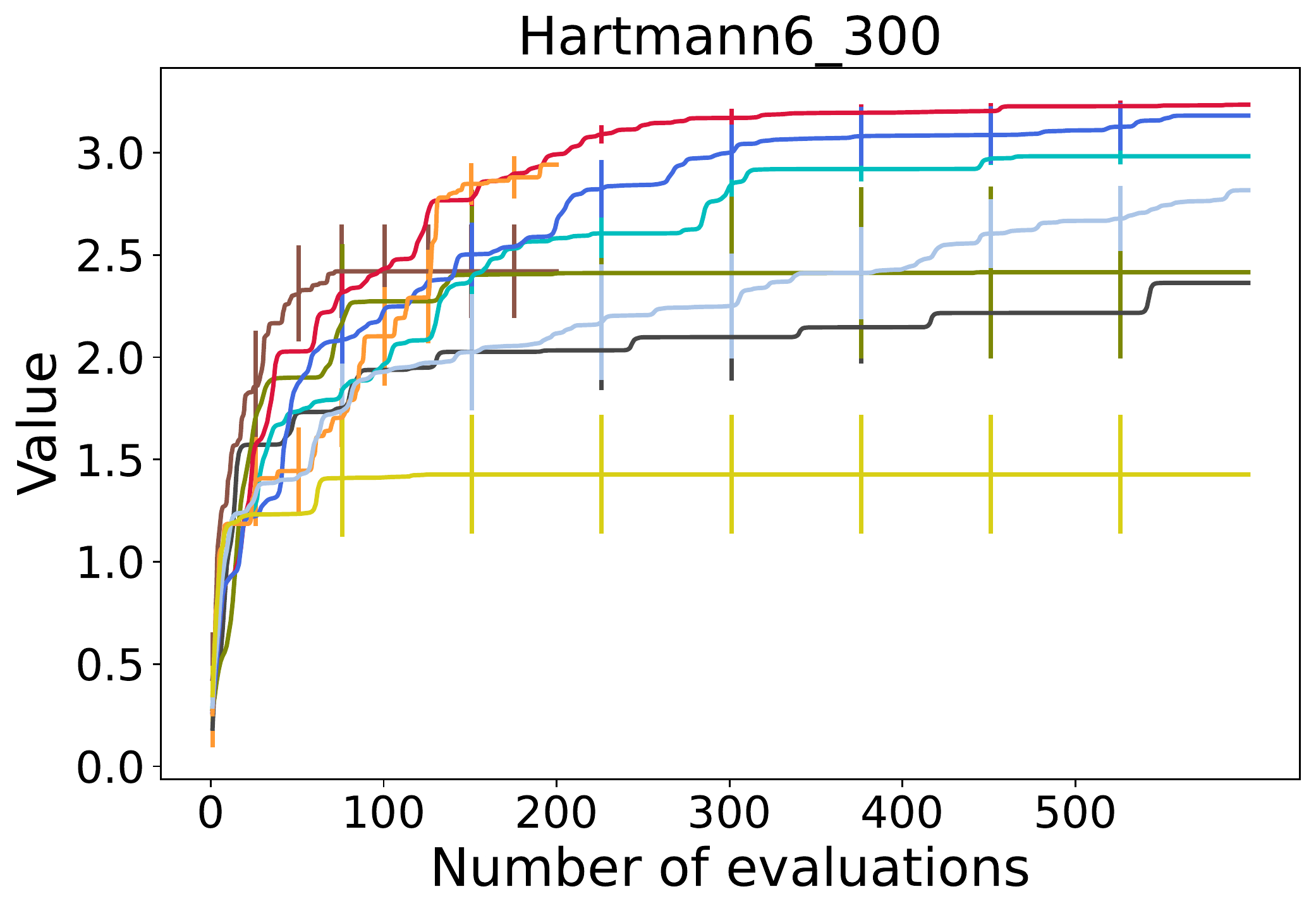}}
    \subfigure{\includegraphics[width=0.4\textwidth]{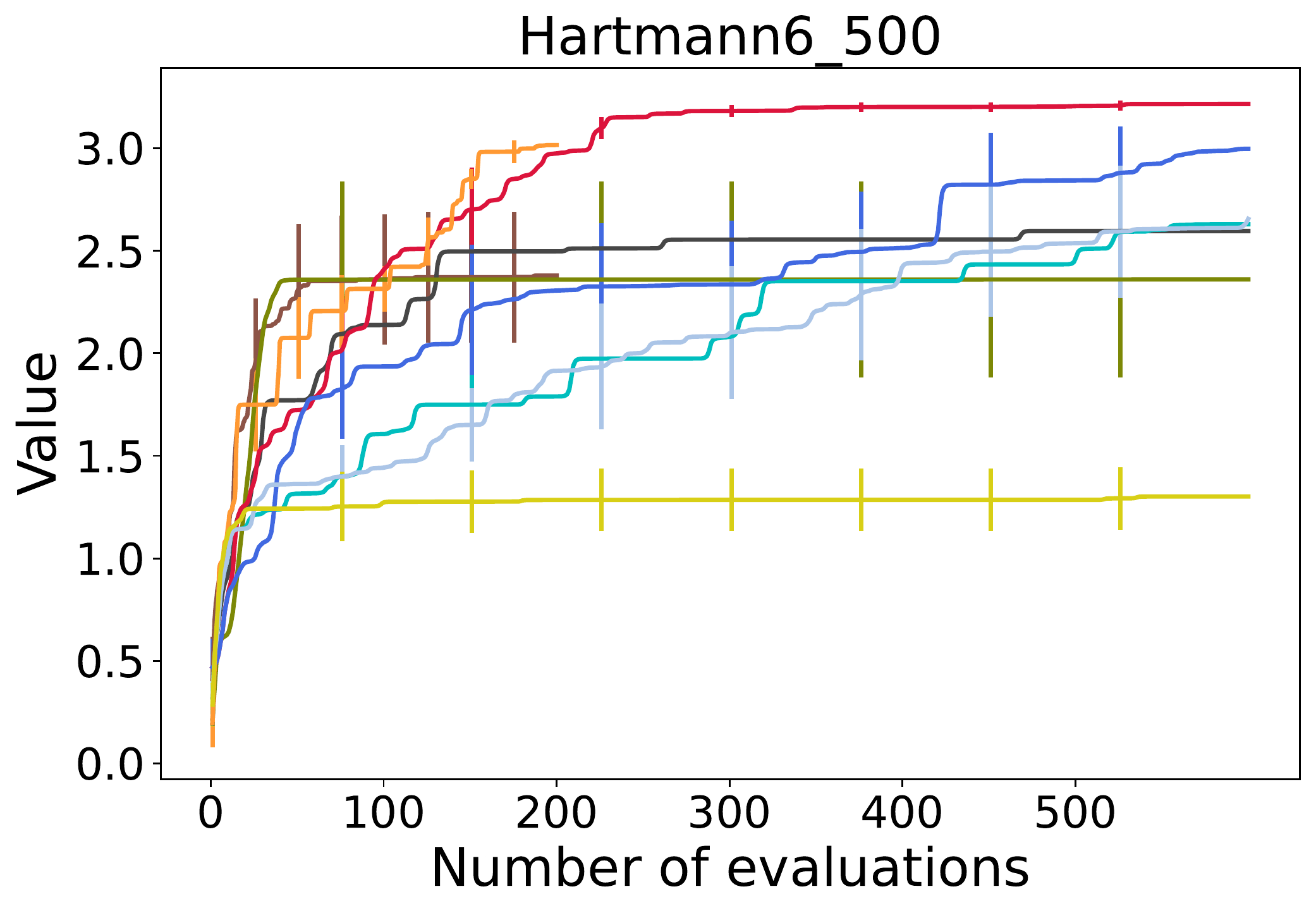}}\vspace{-0.8em}
    \caption{\rbt{Comparison among MCTS-VS and state-of-the-art methods on synthetic functions.}}
    \label{fig:enlargement_synthetic2}\vspace{-0.5em}
\end{figure*}



\end{document}